\documentclass[hidelinks,11pt]{article} 

\usepackage{amsfonts}
\usepackage{array}
\usepackage{textcomp}
\usepackage{stfloats}
\usepackage{url}
\usepackage{verbatim}
\usepackage{graphicx}
\usepackage{cite}

\usepackage{times}
\usepackage{soul}
\usepackage[hidelinks]{hyperref}
\usepackage[utf8]{inputenc}
\usepackage{graphicx}
\usepackage{amsmath,amssymb,amsthm}
\usepackage{booktabs}
\usepackage[ruled,vlined,linesnumbered]{algorithm2e}
\usepackage[switch]{lineno}
\usepackage{fullpage}
\usepackage{microtype}
\usepackage{booktabs} 
\usepackage{hyperref}
\usepackage{bm}
\usepackage{xcolor}
\usepackage{subcaption}
\usepackage{caption}
\usepackage{wrapfig}
\hypersetup{
    colorlinks=true,
    linkcolor=red,
    citecolor=blue,
    filecolor=magenta,      
    urlcolor=cyan,
   }     

\usepackage{tabularx}
\usepackage{multirow}
\usepackage{nicefrac}       
\newtheorem{theorem}{Theorem}
\newtheorem{lemma}{Lemma}
\newtheorem{proposition}{Proposition}

\newtheorem{assumption}{Assumption}
\newtheorem{corollary}{Corollary}
\newtheorem{definition}{Definition}
\newtheorem{remark}{Remark}

\newcommand{\norm}[1]{\ensuremath{\left\| #1\right\|}}

\usepackage{dsfont}
\newcommand{\ind}[1]{\ensuremath{\mathds{1}\left\{ #1 \right\}}}

\newcommand{\A}{\ensuremath{\mathcal{A}}}

\newcommand{\R}{\ensuremath{\mathbb{R}}}

\newcommand{\X}{\ensuremath{\mathcal{X}}}

\renewcommand{\O}{\ensuremath{\mathcal{O}}}

\newcommand{\CO}{\ensuremath{\mathcal{O}}}
\newcommand{\x}{\ensuremath{\mathrm{\mathbf{x}}}}
\newcommand{\supp}{\ensuremath{\mathrm{supp}}}

\newcommand{\n}{\ensuremath{\mathrm{\mathbf{n}}}}
\renewcommand{\v}{\ensuremath{\mathrm{\mathbf{v}}}}
\newcommand{\z}{\ensuremath{\mathrm{\mathbf{z}}}}
\newcommand{\y}{\ensuremath{\mathrm{\mathbf{y}}}}
\renewcommand{\c}{\ensuremath{\mathrm{\mathbf{c}}}}

\newcolumntype{C}[1]{>{\centering\arraybackslash}m{#1}}
\newcolumntype{Y}{>{\centering\arraybackslash}X}

\DeclareMathOperator{\Tr}{\mathrm{tr}}

\DeclareMathOperator*{\argmax}{arg\,max}

\title{Learning to Sparsify Stochastic Linear Bandits\\ with Nearly Minimax-Optimal Regret}

\author{Zhengmiao Wang, 
Ming Chi, 
Zhi-Wei Liu, 
Lintao Ye, 
Carla Fabiana Chiasserini
\thanks{Z. Wang, M. Chi, ZW. Liu and L. Ye are with the School of Artificial Intelligence and Automation, Huazhong University
of Science and Technology, Wuhan 430074, China, email: \{wangzhengmiao,chiming,zwliu,yelintao93\}@hust.edu.cn. C. F. Chiasserini is with the Department of Electronics and Telecommunications, Politecnico di Torino, 10129 Torino, Italy, email:
carla.chiasserini@polito.it.}
}

\begin{document}
\maketitle

\begin{abstract}
This paper addresses the problem of learning to sparsify stochastic linear bandits, where a decision-maker sequentially selects actions from a high-dimensional space subject to a sparsity constraint on the number of nonzero elements in the action vector. The key challenge lies in minimizing cumulative regret while tackling the potential NP-hardness of finding optimal sparse actions due to the inherent combinatorial structure of the problem. We propose an algorithmic framework that adaptively balances exploration and exploitation utilizing ordinary least squares for parameter learning and specialized subroutines for sparse action selection. When the action set is a Euclidean ball, optimal sparse actions can be efficiently computed, enabling us to establish a $\tilde{\mathcal{O}}(H\sqrt{T})$ regret, where $H$ is the sparsity constraint for actions and $T$ is the time horizon length. For smooth convex action sets, we employ a greedy support-selection procedure and establish a strictly positive submodularity ratio. Under an additional Lipschitz regularity condition on the restricted optimal-action mapping, the resulting APSEE-G algorithm achieves a $\tilde{\mathcal{O}}(H\sqrt{T})$ $\alpha$-regret. Furthermore, for general compact action sets, we propose an Adaptive Greedy Feel-Good Thompson Sampling (AG-FGTS) algorithm that achieves a $\widetilde{\mathcal{O}}(H\sqrt{T})$ $\alpha$-regret bound. To complement these upper bounds, we establish a matching minimax $\Omega(H\sqrt{T})$ regret lower bound on a canonical sparse Euclidean-ball subclass. We validate the performance of our algorithms using extensive experiments, including an application to a recommendation system.
\end{abstract}

\section{Introduction}\label{sec:introduction}
The field of online decision-making under uncertainty, particularly within the multi-armed bandit (MAB) framework, has been a central area of research \cite{auer2002finite,lattimore2020bandit}. Stochastic linear bandits model the reward at each decision step $t$ as an inner product of an unknown parameter (i.e., context vector) $\theta_*$ and the action vector $\mathbf{x}_t$ plus stochastic noise \cite{abbasi2011improved,bastani2021mostly}, which provides a simple yet general approach to model sequential decision-making in applications such as recommendation systems and adaptive control \cite{li2010contextual,chu2011contextual,hoi2021online}. An important consideration in linear bandits for real-world settings is sparsity, i.e., the context vector is sparse and only a small subset of its features influences the rewards \cite{agrawal2013thompson,bastani2021mostly}. Exploiting this sparsity, especially in high-dimensional action space, can lead to more efficient algorithms and improved regret bounds \cite{abbasi2012online,russo2018tutorial}. However, considering sparsity introduces an additional layer of complexity into both the algorithm design and regret analysis \cite{oh2021sparsity,hao2020high,daivariance}.

{\bf Linear Bandits with Sparse Context Vector.} High-dimensional linear bandits, where rewards depend on a sparse subset of the unknown parameters, present a significant challenge \cite{wainwright2019high}. 
The earlier work \cite{abbasi2012online} established a $\tilde{\CO}(\sqrt{sdT})$ regret bound for sparse linear bandits when the sparsity level of the unknown context vector is known, where $s$ is the number of nonzero entries in $\theta_*\in\mathbb{R}^d$. Later, for sparse stochastic linear bandits with unknown sparsity level $s=\|\theta_*\|_0$, Jin et al.~\cite{jin2024sparsity} proposed a randomized model selection approach over nested confidence sets, achieving $\widetilde{\mathcal O}(s\sqrt{dT})$ expected regret even when the action sets are generated by an adaptive adversary. When $s$ is known, their method recovers $\widetilde{\mathcal O}(\sqrt{sdT})$ expected regret. More recently, \cite{jangsparse} studied sparse linear bandits under an explicitly fixed sparsity-support assumption in both adversarial and stochastic regimes, while also highlighting the role of action-set geometry in the stochastic regime. In a related direction, \cite{ichikawa2024new} studied a sparse contextual linear bandit setting and extended the algorithm in \cite{oh2021sparsity} to more general classes of the problem under milder assumptions. 
It should be pointed out that, despite the substantial progress on unknown sparsity levels and adversarial regimes, the existing sparse-linear-bandit literature typically assumes that the underlying sparsity support of the unknown context or parameter vector is fixed throughout the learning horizon and does not change over time.

{\bf Combinatorial Multi-Armed Bandits (CMAB).} In the standard CMAB setting, a learner selects a subset of base arms (i.e., a super-arm) 
at each step, and the reward is given by a function of the individual reward of each arm \cite{chen2013combinatorial,chen2016combinatorial}. For linear reward functions, an upper confidence bound method adapted from linear bandits has been proposed, which chooses a super-arm from the predefined set of base arms \cite{gai2012combinatorial}. For nonlinear reward functions, the recent work \cite{fourati2024combinatorial} proposes an online algorithm that achieves a $(1-1/e)$-regret bound of $\tilde{\mathcal{O}}(T^{2/3})$ for i.i.d. sampled monotone stochastic submodular reward functions, which leverages offline approximation algorithms as black-box oracles \cite{nie2023framework}. A multi-agent grouped CMAB problem has been studied in \cite{vannella2023statistical,vannella2023best}, which uses a mean field approximation approach and characterizes its statistical and computational trade-off.   

{\bf Online Combinatorial Optimization.} There is a broader area of work that considers the online optimization framework and studies cases when the decision-maker first chooses an action that is a subset of a ground set at each step, and then a set function of the chosen action is revealed to return the corresponding reward \cite{streeter2008online}. As in the standard online optimization framework, the set functions can vary across different steps and potentially be generated adversarially \cite{wan2023bandit}. Since the action set chosen by the decision-maker often needs to satisfy certain combinatorial constraints (e.g., cardinality constraint), the problem is also termed online combinatorial optimization \cite{golovin2014online}. A key difficulty of this problem emerges when the offline version of the problem (i.e., when all the set functions are known) is NP-hard, which necessitates the use of $\alpha$-regret that compares the cumulative reward from an online algorithm against an approximately optimal reward with ratio $\alpha$ in hindsight \cite{streeter2008online}. The ratio $\alpha\in(0,1)$ is set to be the best-known approximation ratio of an offline algorithm for the problem \cite{das2018approximate,bian2017guarantees}. Recent work, such as \cite{ye2025online}, continues to explore efficient online algorithms for these complex problems, often relying on the properties of the underlying combinatorial structure and the approximation guarantees of offline solvers.

{\bf Thompson Sampling for Bandits.} Thompson Sampling (TS) has emerged as a powerful randomized approach for addressing the exploration-exploitation trade-off, where actions are selected according to their posterior probability of being optimal \cite{agrawal2013thompson,russo2018tutorial}. For linear bandits, TS has been proven to achieve frequentist regret bounds that are competitive with those of Upper Confidence Bound (UCB) methods, while often demonstrating superior empirical performance \cite{abeille2017linear}. 
Recent research has explored adapting TS to sparse environments, for instance, by employing spike-and-slab priors or integrating variable selection steps into the sampling process to efficiently recover the support of the underlying parameters \cite{wang2018thompson,kim2019doubly}. Furthermore, to overcome the suboptimal dimension dependencies often introduced by the variance inflation required in standard frequentist TS analysis, the Feel-Good Thompson Sampling (FGTS) framework was recently introduced by \cite{zhang2022feel}, which departs from standard Gaussian exploration by directly injecting a data-dependent optimism bonus into the posterior distribution, decoupling the parameter estimation error from the exploration variance.

{\bf Motivation.} 
While significant progress has been made in sparse linear bandits, existing work typically assumes the sparsity lies within the unknown context vector $\theta_*$. In such settings, the sparsity pattern is an environmental condition that the learner merely attempts to recover; the learner does not actively design the sparsity structure of the action itself. A critical and less explored intersection arises when the selected action vector $\mathbf{x}_t$ must adhere to explicit sparsity constraints. Unlike standard Combinatorial Multi-Armed Bandits (CMAB), where the action is a selection of discrete base arms, our setting imposes a sparsity constraint on a continuous vector $\mathbf{x}_t \in \mathbb{R}^d$, which confronts the learner with a dual challenge at every time step: making a discrete decision to specify the non-zero support of $\mathbf{x}_t$, while simultaneously determining the continuous values of these entries within a geometric constraint (e.g., a Euclidean ball). We now ask the following natural questions:

{\it What if the learner must simultaneously determine the sparsity pattern of the action vector $\mathbf{x}_t$ and the non-zero entries of $\mathbf{x}_t$? Can we design algorithms for the learner to achieve meaningful regret?} 

We refer to the resulting problem as {\it learning to sparsify stochastic linear bandits.} Such a problem arises in many applications: i) In an advertisement (ad) allocation or recommendation system, for each day or time period, the company needs to decide which TV channels or application platforms to deliver ads and how long or how many ads should be for each selected channel or platform. ii) In stock marketing, at each decision point, the buyer must decide which companies to purchase stock from and how many shares to purchase for each selected company. Among the above real-world applications, the linear bandits model is rich enough to capture the uncertainty in the online decision-making process while being tractable to analyze. Moreover, it is also clear from the above examples that the non-zero entries of $\mathbf{x}_t$ often need to satisfy certain combinatorial constraints; typically, the number of non-zero entries in $\mathbf{x}_t$ should be no greater than some positive integer $H$, which corresponds to a cardinality constraint.

\begin{table*}[ht]
\centering
\caption[Comparison of regret bounds for sparse linear bandits]{
Comparison of regret guarantees for sparse linear bandits. 
$s=\|\theta_*\|_0$: sparsity of the unknown reward parameter; $H$: prescribed action sparsity. ``Given'' means that the sparsity level of $\theta_*$  (hence, that of $\mathbf{x}_t$) is given. ``Designed'' means that the sparsity pattern of $\mathbf{x}_t$ is selected by the learner subject to
$\|\mathbf{x}_t\|_0\le H$. $C_{\min}$ is the exploration-design eigenvalue in \cite{hao2020high},
and $C_{\mathrm{gap}},C_{\mathrm{gap}}^G$ are instance-dependent support-recovery costs,
independent of $T$ for fixed positive identification gaps.
$\widetilde{\mathcal O}(\cdot)$ suppresses logarithmic factors.
Upper bounds implicitly refer to $1$-regret unless we explicitly write $\alpha$-regret (see Sec.~\ref{sec:Problem Formulation} for formal definitions). All lower bounds, if they exist in the corresponding reference, concern expected $1$-regret. 
}
\label{tab:regret_comparison}

\footnotesize
\setlength{\tabcolsep}{4pt}
\renewcommand{\arraystretch}{2.85}
\setlength{\extrarowheight}{0pt}

\begin{tabularx}{\textwidth}{
@{}C{0.20\textwidth}
|C{0.145\textwidth}
|C{0.19\textwidth}
|Y|Y@{}}

\hline
\parbox[c]{\linewidth}{\centering\textbf{References}}
& \parbox[c]{\linewidth}{\centering\textbf{Sparsity}}
& \parbox[c]{\linewidth}{\centering\textbf{Action set}}
& \parbox[c]{\linewidth}{\centering\textbf{Upper bound}}
& \parbox[c]{\linewidth}{\centering\textbf{Lower bound}} \\
\hline

\parbox[c]{\linewidth}{\centering
    Abbasi-Yadkori et al.~\cite{abbasi2012online}}
& \parbox[c]{\linewidth}{\centering
    Given\\(Known $s$)}
& \parbox[c]{\linewidth}{\centering
    Arbitrary bounded}
& \parbox[c]{\linewidth}{\centering
    $\widetilde{\mathcal O}(\sqrt{sdT})$}
& \parbox[c]{\linewidth}{\centering
    ---} \\
\hline

\parbox[c]{\linewidth}{\centering
    Hao et al.~\cite{hao2020high}}
& \parbox[c]{\linewidth}{\centering
    Given\\(Known $s$)}
& \parbox[c]{\linewidth}{\centering
    Fixed compact, full span}
& \parbox[c]{\linewidth}{\centering
    $\widetilde{\mathcal O}
    (C_{\min}^{-2/3}s^{2/3}T^{2/3})$}
& \parbox[c]{\linewidth}{\centering
    $\begin{gathered}
    \Omega\bigl(\min\{
    C_{\min}^{-1/3}s^{1/3}T^{2/3},\\[-2pt]
    \sqrt{dT}\}\bigr)
    \end{gathered}$} \\
\hline

\parbox[c]{\linewidth}{\centering
    Jin et al.~\cite{jin2024sparsity}}
& \parbox[c]{\linewidth}{\centering
    Given\\(Unknown $s$)}
& \parbox[c]{\linewidth}{\centering
    Arbitrary subsets of the Euclidean ball}
& \parbox[c]{\linewidth}{\centering
    $\widetilde{\mathcal O}(s\sqrt{dT})$}
& \parbox[c]{\linewidth}{\centering
    ---} \\
\hline

\parbox[c]{\linewidth}{\centering
    Ye et al.~\cite{ye2025online}}
& \parbox[c]{\linewidth}{\centering
    Designed}
& \parbox[c]{\linewidth}{\centering
    Convex, compact}
& \parbox[c]{\linewidth}{\centering
    $\widetilde{\mathcal O}(d\sqrt H\,T^{5/6})$\\
    $\alpha$-regret}
& \parbox[c]{\linewidth}{\centering
    ---} \\
\hline

\parbox[c]{\linewidth}{\centering
    \textbf{APSEE (ours)}\\
    Theorem~\ref{thm:regret_adaptive}}
& \parbox[c]{\linewidth}{\centering
    Designed}
& \parbox[c]{\linewidth}{\centering
    Euclidean ball}
& \parbox[c]{\linewidth}{\centering
    $\widetilde{\mathcal O}(H\sqrt T)
    +C_{\mathrm{gap}}$}
& \multirow[c]{3}{=}{
    \centering
    \shortstack[c]{
        Theorem~\ref{thm:lower_bound_randomized}\\[3pt]
        $\Omega(H\sqrt T)$}
    } \\
\cline{1-4}

\parbox[c]{\linewidth}{\centering
    \textbf{APSEE-G (ours)}\\
    Theorem~\ref{thm:regret_general}}
& \parbox[c]{\linewidth}{\centering
    Designed}
& \parbox[c]{\linewidth}{\centering
    Smooth convex}
& \parbox[c]{\linewidth}{\centering
    $\widetilde{\mathcal O}(H\sqrt T)
    +C_{\mathrm{gap}}^G$\\
    $\alpha$-regret}
& \\
\cline{1-4}

\parbox[c]{\linewidth}{\centering
    \textbf{AG-FGTS (ours)}\\
    Theorem~\ref{thm:ag_ts_regret}}
& \parbox[c]{\linewidth}{\centering
    Designed}
& \parbox[c]{\linewidth}{\centering
    General compact}
& \parbox[c]{\linewidth}{\centering
    $\widetilde{\mathcal O}(H\sqrt T)
    +C_{\mathrm{gap}}^G$\\
    $\alpha_{\mathrm{CH}}$-regret}
& \\

\hline
\end{tabularx}

\end{table*}

{\bf Contributions.} Unlike the linear bandits (resp., online combinatorial optimization) problem that selects a single continuous action vector (resp., discrete action set), in the learning to sparsify linear bandits problem, the learner needs to select both continuous and discrete components that constitute a sparse action. Moreover, finding an optimal sparse action even in the offline setting can be an NP-hard problem, which leads to the following major challenge when sparsifying linear bandits: 
{\it How can a learner efficiently minimize cumulative regret when the action selection at each step constitutes a potentially NP-hard problem?} This paper confronts this challenge by introducing a novel algorithmic framework for learning to sparsify stochastic linear bandits. We propose novel strategies that effectively {\it decouple} parameter learning from combinatorial action selection, employing specialized subroutines for the efficient selection of sparse actions. Our key contributions are:\\
$\bullet$ We first identify that an optimal sparse action is tractable under a Euclidean ball action set $\X$. We introduce a novel adaptive warm-up strategy with a data-driven stopping condition that does not rely on prior knowledge of the unknown parameters. Leveraging this adaptive strategy to recover the support of the optimal sparse action, we establish a $\tilde{\mathcal{O}}(H\sqrt{T})$ regret, where $H$ is the sparse constraint for actions.\\
$\bullet$ When finding optimal sparse solutions becomes intractable for general action sets, we utilize a combinatorial greedy algorithm as a subroutine to select a sparse action efficiently. Based on the regularity of the action set, we distinguish two cases: i) For \textbf{smooth convex} action sets with continuously differentiable boundaries, we establish a strictly positive submodularity ratio and propose the APSEE-G algorithm, which adaptively recovers the support of a greedy sparse action via the notion of greedy gap. Under an additional Lipschitz regularity condition on the restricted optimal-action mapping, we obtain a $\tilde{\mathcal{O}}(H\sqrt{T})$ $\alpha$-regret. ii) For \textbf{general compact} action sets, where the required regularity of the optimal-action mapping may not hold, we propose a modified Feel-Good Thompson Sampling algorithm. By leveraging the greedy support recovery mechanism to restrict the sampling space, we establish a $\tilde{\mathcal{O}}(H\sqrt{T})$ regret bound relative to the greedy-support benchmark, which further yields an $\alpha_{\mathrm{CH}}$-regret guarantee under the Smooth Convex Inner-Set Regularity condition.
\\
$\bullet$ We further complement our upper bound results by a minimax $\Omega(H\sqrt{T})$ regret lower bound. The lower bound identifies $H\sqrt{T}$ as an information-theoretic benchmark for the dependence on the sparse action constraint $H$ and the horizon $T$, and matches the dominant $H$- and $T$-dependence of our upper bounds.\\
$\bullet$ Finally, we validate our theoretical results using comprehensive numerical experiments across action sets with different geometries, including a real-world recommendation system instance.
\\This work introduces a new setup for the classic linear bandits model, offering both theoretical guarantees and algorithmic insights for a challenging class of high-dimensional decision-making problems.

{\bf Comparison with related work.} 
Table~\ref{tab:regret_comparison} summarizes representative theoretical results and highlights the distinctions between our work and existing studies in terms of sparsity constraints, feedback models, and regret guarantees. In particular, unlike conventional sparse linear bandits that impose sparsity on the unknown reward parameter, our formulation requires the learner to jointly select a sparse support and optimize the corresponding continuous action values. A preliminary version of this work appeared in our conference paper \cite{ijcai2026p560}, which we substantially extend in several directions. First, we introduce a restricted exploration mechanism following adaptive support identification, reducing the leading horizon-dependent regret from $\tilde{\O}(d\sqrt{T})$ to $\tilde{\O}(H\sqrt{T})$ for both Euclidean-ball and smooth convex action sets satisfying the required restricted-optimizer regularity condition, up to an instance-dependent warm-up cost. Second, for general compact action sets, we replace the deterministic phased exploitation strategy of the conference version with a novel Adaptive Greedy Feel-Good Thompson Sampling (AG-FGTS) algorithm, improving the leading horizon-dependent term from $\tilde{\O}(dT^{2/3})$ to $\tilde{\O}(H\sqrt{T})$. This guarantee holds relative to the greedy-support benchmark and further yields an $\alpha_{\mathrm{CH}}$-regret bound under the Smooth Convex Inner-Set Regularity assumption. Third, we establish an $\Omega(H\sqrt{T})$ minimax lower bound for the $1$-regret, identifying the fundamental dependence on the action sparsity level $H$ and the time horizon $T$. Finally, we substantially expand the numerical evaluation, including experiments on the Open Bandit Dataset from ZOZOTOWN \cite{saito2020open}.

{\bf Notations.}
The sets of integers and real numbers are denoted as $\mathbb{Z}$ and $\mathbb{R}$, respectively. For a set $\mathcal{X}$, $\mathcal{X}^n$ denotes the corresponding Cartesian product $\X\times\cdots\times\X$. 
For a real number $a$, let $\lfloor a \rfloor$ be the largest integer that is smaller than or equal to $a$. 
For a matrix $P\in\R^{d\times d}$, let $P^{\top}$, $\Tr(P)$, $(P)_{ij}$ be its transpose, trace, and element in the $i$-th row and $j$-th column, respectively. 
Let $\norm{x}_A=\sqrt{x^\top A x}$ for a vector $x\in \R^d$ and for a positive definite matrix $A\in\R^{d\times d}$. 
Let $\lambda_\mathrm{min} (A)$, $\lambda_\mathrm{max} (A)$ be the minimum and maximum eigenvalues of a positive definite matrix $A$, respectively.
Let $I_d$ be the $d$-dimensional identity matrix. Given any integer $n\ge1$, $[n]\triangleq\{1,\dots,n\}$. Let $\ind{\cdot}$ denote an indicator function. 
For any $\x\in\R^d$ and any $S\subseteq[d]$, let $\x(S)\in\R^d$ be such that $\supp(\x)=S$ and $(\x(S))_i=(\x)_i$ for all $i\in[d]$, where $(\x)_i$ denotes the $i$-th element of $\x$ and $\supp(\x)\triangleq\{i\in[d]:(\x)_i\ne 0\}$. 
Denote by $a_{i:j}$ the sub-sequence $a_i,a_{i+1},\dots,a_j$.
For a real number $y\in\R$, let $|y|$ be its absolute value. For a finite set $S$, let $|S|$ be its cardinality. For a vector $\x\in\R^d$, let $\|\mathbf{x}\|_p = \big( \sum_{i=1}^d |(\mathbf{x})_i|^p \big)^{\frac{1}{p}}$ be its $\ell_p$-norm for $p\in(0,\infty)$; let $\|\mathbf{x}\|_0 = |\{ i \in [d] \mid (\mathbf{x})_i \neq 0 \}|$ and $\norm{\x}_\infty=\max_{i\in[d]}|(\x)_i|$ be its $\ell_0$-norm and $\ell_\infty$-norm, respectively.

\section{Problem Formulation }\label{sec:Problem Formulation}
Stochastic linear bandits model a scenario of learning under uncertainty \cite{dani2007price,abbasi2011improved,lattimore2020bandit}. Over $T$ time steps, a decision-maker sequentially picks an action $\x_t$ from a feasible action space $\X_t \subset \R^d$ and receives a reward $Y_t\in\R$, which is a random variable with mean determined by the inner product of the action and an unknown parameter $\theta_*\in\R^d$, i.e., $\mathbb{E}[Y_t|\x_t] = \theta_*^\top\x_t$. The goal is to devise a policy that minimizes the cumulative regret $R_T$, which quantifies the sum of performance gaps between the optimal action in hindsight and the chosen action at each step: 
\begin{align}
    R_T=\sum_{t=1}^{T}(\max_{\x\in\mathcal{X}}\theta_*^\top \x-\theta_*^\top\x_t).\label{eqn:def of reg}
\end{align}

This basic model assumes that the reward $Y_t$ depends on the current action $\x_t$; however, in many cases, the external environment has requirements for the number of non-zero components of the action vector $\x_t$. For instance, i) the number of non-zero components of $\x_t$ must be smaller than or equal to $H$ for a deterministic $H\in[d]$ ($H$ is known to the decision maker); ii) some components of the action vector $\x_t$ are not available (they need to stay $0$ for time $t$). Specifically, letting $S_t=\supp(\x_t)$, at each time step $t$, the decision maker needs to select 
\begin{align}
    \x_t(S_t)\in\X, ~s.t.~ S_t\subseteq[d] ~\text{and}~ |S_t|\le H.
\end{align}
We make the following standard assumption on the set $\X$ \cite{dani2008stochastic,rusmevichientong2010linearly}.
\begin{assumption}
\label{ass:feasible set bound}
The feasible set $\X\subset\R^d$ does not change over time, and it is closed and bounded.
\end{assumption}
\begin{remark}\label{remark:fix action set}
    We can also consider $\X_t$ that changes over time as~\cite{abbasi2011improved,lattimore2020bandit}, and the description in Assumption~\ref{ass:feasible set bound} will become $\forall t\in[T],\forall \x\in \X_t, \norm{\x}_2\le L$. More discussion on this will be provided in Appendix~\ref{app:time_varying_action_set}.
\end{remark}
\noindent Based on the above discussions, the reward $Y_t$ at any time $t$ is given by
\begin{equation}\label{eqn:reward function}
Y_t=\theta_*^\top\x_t(S_t)  +\eta_t,
\end{equation}
where $\eta_t$ is a random noise satisfying a certain tail condition specified below. 
\begin{assumption}
\label{ass:noise}
For any $t\in[T]$, $\eta_t$ is a conditionally $\sigma$-sub-Gaussian distribution \cite{buldygin1980sub} with known variance proxy $\sigma$, i.e., $\mathbb{E}[\eta_t\mid \x_{1:t}(S_{1:t}),\eta_{1:t-1}]=0 $, and $\mathbb{E}[\exp(\lambda \eta_t)\mid \x_{1:t}(S_{1:t}),\eta_{1:t-1}]\le \exp(\frac{\lambda ^2 \sigma^2}{2} )$ for any $\lambda\in\R$. 
\end{assumption}

Hence, an optimal solution in hindsight (i.e., when $\theta_*$ is known) to the sparsifying linear bandits problem described above is given by solving
\begin{align}
    \x^\star(S^\star)\in\argmax_{\x(S)\in\X, S\subseteq[d],|S|\le H} \theta_*^\top\x(S).\label{eqn:def of exact solution}
\end{align}
Finding an optimal action $\x^\star(S^\star)$ involves solving an optimal subset selection problem, which is generally NP-hard \cite{das2018approximate,ye2025online}. This computational barrier typically makes it intractable for any polynomial-time algorithm to compete with the true optimum, necessitating a relaxed performance benchmark known as $\alpha$-regret \cite{kalai2005efficient,cesa2006prediction,streeter2008online}, given by
\begin{align}
\alpha\text{-}R_T\triangleq \alpha\left (\sum_{t=1}^{T}\theta_*^\top\x^\star(S^\star)\right )-\sum_{t=1}^{T}\theta_*^\top\x_t(S_t).\label{eqn:def of a-regret}
\end{align}
When $\alpha=1$, $\alpha\text{-}R_T$ in \eqref{eqn:def of a-regret} reduces to the ordinary regret, e.g., $R_T$ in \eqref{eqn:def of reg}. The $\alpha$-regret defined in \eqref{eqn:def of a-regret} compares the algorithm's actions against an $\alpha$-approximation of the optimal action, where $\alpha\in(0, 1]$ is the best-known approximation ratio for the offline version of problem \eqref{eqn:def of exact solution}. Nonetheless, we will establish special conditions under which the underlying best subset selection problem can be solved efficiently and exactly. This key distinction circumvents the usual NP-hardness, allowing us to set the approximation factor $\alpha=1$ and evaluate the algorithm's performance directly against the true optimal action without any approximation loss.

{\bf A Relevant Problem Formulation.} As we mentioned in the Introduction, \cite{ye2025online} studied a sparse online linear optimization problem. At each step $t$, the learner selects $\mathbf{x}_t\in\R^d$ and $S_t\subseteq[d]$ with $|S_t|\le H$ and receives a reward given by $f_t(\mathbf{x}_t,S_t)=\theta_t^{\top}\mathbf{x}_t(S_t)$, where $\theta_t$ is potentially generated by an adversary. However, there is no regret guarantee for the bandit setting (i.e., the learner only observes $\theta_t^{\top}\mathbf{x}_t(S_t)$) and they achieved an $\alpha$-regret of $\tilde{\CO}(\sqrt{H}dT^{5/6})$  when two quires of $f_t(\cdot)$ is available to the learner at each step $t$. The online linear optimization framework for a single continuous action (without sparsity consideration) has also been studied in \cite{bubeck2012towards}, which achieves a $\tilde{\CO}(\sqrt{dT})$-regret. In contrast with the setting of online linear optimization, we study the stochastic linear bandits setting whose reward is a random variable given by \eqref{eqn:reward function}.

{\bf Roadmap.} Our algorithm design and regret analysis in the subsequent sections hinge on the geometry of the action set $\X$. We mainly organize our approach into four parts:
(i) \textit{Euclidean Ball Action Sets} (Section~\ref{sec:euclidean_case}), where the optimal sparse action can be computed exactly and efficiently;
(ii) \textit{Smooth Convex Action Sets} (Section~\ref{sec:smooth_convex_case}), where we utilize a greedy approximation algorithm to select a sparse action and establish an $\alpha$-regret guarantee under the required regularity conditions;
(iii) \textit{General Compact Action Sets} (Section~\ref{sec:ag_ts}), where we propose a modified Thompson Sampling algorithm to address scenarios lacking smoothness or Lipschitz continuity of its restricted optimal-action mapping;
and (iv) \textit{Lower Bound} (Section~\ref{sec:Lower Bound}), where we establish a minimax $\Omega(H\sqrt{T})$ lower bound on the $1$-regret.

A fundamental property required for our analysis is the variation of the optimal value function $h(S;\theta)$ with respect to the parameter $\theta$. For general compact action sets, the value function is Lipschitz continuous, which provides a useful tool for analyzing the stability of the greedy marginal gains. However, this property alone does not guarantee the continuity of the corresponding optimal-action mapping. For the deterministic exploitation strategy in APSEE-G, we therefore impose an additional Lipschitz regularity condition on the restricted optimal-action mapping, as specified in Assumption~\ref{ass:smooth_restricted_optimizer}. Strong convexity of the action set, under the ball-intersection definition in \cite{polovinkin1996strongly}, provides one sufficient condition for this additional regularity, but is not required for establishing the positive submodularity ratio in our smooth convex setting.
\begin{proposition}[Lipschitz Continuity]
\label{prop:lipschitz_h}
Consider a compact set $\X$, and let $L_{\X} = \sup_{\x \in \X} \norm{\x}_2$. For any fixed support set $S \subseteq [d]$, the value function $h(S; \theta)$ defined as $h(S; \theta) = \sup_{\x \in \X, \supp(\x) \subseteq S} \theta^\top \x$ is convex and Lipschitz continuous with respect to $\theta$, i.e., for any $\theta_1, \theta_2 \in \R^d$,
\begin{equation}\label{eqn:lipschitz_valuefunction}
    |h(S; \theta_1) - h(S; \theta_2)| \le L_{\X} \norm{\theta_1 - \theta_2}_2.
\end{equation}
\end{proposition}

\section{Euclidean Ball Action Sets}
\label{sec:euclidean_case}
We first consider the scenario when $\X$ is an $\ell_2$-ball (Euclidean ball) with radius $L_\text{max}$, which possesses a structure that permits the efficient and exact computation of the optimal sparse action. One can show that $L_\mathrm{grad}=L_\text{max}$ in Proposition~\ref{prop:lipschitz_h}. The following proposition formalizes the closed-form solution for the optimal sparse action for $\ell_2$-ball action set; the proof can be found in Appendix~\ref{app:Omitted Proofs in sec:euclidean_case}.
\begin{proposition}[Optimal Sparse Action on $\ell_2$-Ball]
\label{prop:l2_exact_solution}
Let $\X = \{\x \in \R^d : \norm{\x}_2 \le L_\text{max}\}$. For any parameter vector $\theta \in \R^d$, the optimal $H$-sparse action $\x^\star$ for $\x^\star \in \argmax_{\x \in \X, \supp(\x)\subseteq[d],|\supp(\x)| \le H} \theta^\top \x$ is given by $\x^\star = L_\text{max} \cdot \frac{\theta(S_H)}{\|\theta(S_H)\|_2}$, where $S_H$ is the set of indices corresponding to the $H$ largest absolute values of $\theta$. 
\end{proposition}
Hence, finding the best $H$-sparse action $\x^\star(S^\star)$ for a given $\theta$ can be done in polynomial time by selecting the top $H$ largest-magnitude components of $\theta$. However, a critical challenge arises when the support set $S_t$ changes dynamically, and the mapping from a parameter vector $\theta$ to its optimal sparse support $S^\star_\theta$ is discontinuous. 

{\bf Key ideas.} To establish the desired $\tilde{\mathcal{O}}(H\sqrt{T})$ regret, it is essential to ensure that the proposed algorithm operates on the true optimal support $S^\star_{\theta_*}$ (which is simply denoted as $S^\star$ subsequently). To this end, we propose an \textit{Adaptive Warm-up Strategy} that does not require any prior knowledge of the true signal gap (i.e., the difference between the $H$-th largest and the $(H+1)$-th largest in the absolute value of the elements in $\theta_*$). Instead of fixing the exploration length a priori, our algorithm continuously monitors the \textit{empirical gap} $\Delta_c^\prime$ of the estimated parameter $\hat{\theta}_c$ (i.e., the aforementioned difference computed from $\hat{\theta}_c$). 
Furthermore, to eliminate the suboptimal dependence on $d$ during the stable phase, we introduce a \textit{Restricted Exploration} mechanism: once the true support is identified with high probability, the algorithm dynamically changes its base actions to span only the $H$-dimensional subspace. We will show that the above approach ensures that Algorithm~\ref{alg:adaptive_PEE} safely restricts its dimension and transitions to the exploitation phase only after identifying the true optimal support $S^\star$ with high probability. For later regret analysis, we define the restricted signal strength
on the optimal sparse support as
\begin{equation} \label{eqn:zeta_star}
    \zeta_\star \triangleq \norm{\theta_*^\prime(S^\star)}_2 = \norm{\theta_*(S^\star)}_2.
\end{equation}
Here, $\theta_*^\prime(S^\star)\in\mathbb{R}^H$ denotes the $H$-dimensional sub-vector of $\theta_*$ corresponding to the optimal sparse support $S^\star$. We refer to $\zeta_\star$ as the restricted signal strength on the optimal support.

\subsection{APSEE Algorithm}
\label{subsec:adaptive_algo_euclidean}
We now present the Adaptively Phased Sparse Exploration and Exploitation (\textbf{APSEE}) algorithm (Algorithm~\ref{alg:adaptive_PEE}). The algorithm operates in cycles indexed by $c$. The cycle index $c$ initializes as $1$ and increments by $1$ after each cycle, allowing the algorithm to dynamically adapt time steps during each cycle until the total horizon $T$ is reached. 
Each cycle begins with an \textit{Exploration Phase}. If the support has not been locked, the algorithm plays a full spanning set of $d$ exploration arms to update the $d$-dimensional OLS estimate $\hat{\theta}_c$ using accumulated data (line 9). To identify the true support $S^\star$, the algorithm performs a \textit{Support Verification} step (lines 10-16): it computes an anytime error bound $\varepsilon_c$, sorts the absolute values based on the estimated $\hat{\theta}_c$ of the components $\{i_1^\prime, \dots, i_d^\prime\}$, and calculates the empirical gap $\Delta_c^\prime$. 
If the stopping condition $\Delta_c^\prime > 2\varepsilon_c$ is satisfied, the support $S_{est}$ is locked (i.e., fixed). At this pivotal moment, the algorithm transitions its exploration strategy. In all subsequent cycles, it performs \textit{Restricted Exploration} by playing only $H$ base actions spanning the active subspace $S_{est}$ (lines 18-23), thereby cultivating a specialized $H$-dimensional OLS estimator $\bar{\theta}_m^{(H)}$. To ensure that actions maintain the ambient dimension $d$, line 24 zero-pads $\bar{\theta}_m^{(H)}$ outside the support $S_{est}$ to construct the $d$-dimensional estimator $\bar{\theta}_m^{(d)}$. Finally, the algorithm enters the \textit{Exploitation Phase} (lines 25-30), executing the best action on $S_{est}$ for $c$ steps. A key advantage of APSEE is that it operates without requiring prior knowledge of the total time horizon $T$. To ensure that the OLS estimator utilized in the Exploration Phase is well-defined and possesses the necessary concentration properties for our theoretical analysis, we impose the following regularity condition on the exploration actions.
\begin{assumption}[Exploration Regularity]
\label{ass:arms during exploration}
The exploration actions satisfy the following two regularity conditions.
\\ \noindent\textbf{(i) Full-dimensional exploration.} 
There exists a set of $d$ feasible sparse actions $\{b_1(S_1),\dots,b_d(S_d)\}\subseteq\X, \ |S_k|\le H,\ k\in[d]$, such that the Gram matrix $B\triangleq\sum_{k=1}^d b_k(S_k)b_k(S_k)^\top\in\mathbb R^{d\times d}$ is positive definite. We denote $\lambda_0\triangleq\lambda_{\min}(B)>0$. 
\\ \noindent\textbf{(ii) Restricted exploration.}
For every candidate support $S\subseteq[d]$ with $|S|=H$, there exist $H$ feasible restricted exploration actions $\{e_1(S),\dots,e_H(S)\}\subseteq\X(S)$, where $\X(S)=\{\x\in\X:\supp(\x)\subseteq S\}$, such that the restricted Gram matrix $B_H(S)\triangleq\sum_{k=1}^H e_k^\prime(S)e_k^\prime(S)^\top \in\mathbb R^{H\times H}$ is positive definite, where $e_k^\prime(S)$ denote the sub-vector of $e_k(S)$ indexed by $S$. We denote $\lambda_H\triangleq \min_{S\subseteq[d],|S|=H}\lambda_\mathrm{min} (B_H(S))>0$.
\end{assumption}
Assumption~\ref{ass:arms during exploration} requires sufficient identifiability both before and after support recovery. The first condition guarantees that the full-dimensional parameter can be estimated during the warm-up phase, whereas the second condition ensures that, once a support is identified, the parameter restricted to that support can be consistently estimated using only $H$-dimensional exploration. Although the original sparse action constraint allows $|\supp(\x)|\le H$, it is sufficient to formulate the restricted exploration condition for support sets of cardinality exactly $H$. Indeed, for any $S_1\subseteq S_2\subseteq[d]$, we have $\X(S_1)\subseteq\X(S_2)$, and hence $h(S_1;\theta)\le h(S_2;\theta)$. Therefore, $\max_{\substack{S\subseteq[d], |S|\le H}}h(S;\theta)=\max_{\substack{S\subseteq[d], |S|=H}} h(S;\theta)$. Thus, without loss of generality, the support-selection stage may be viewed as selecting an $H$-element coordinate set, even though the resulting optimal action itself may have fewer than $H$ nonzero coordinates. To this end, Condition~(ii) in Assumption~\ref{ass:arms during exploration} is imposed on every candidate $H$-element support because the support identified by the learner is data-dependent and unknown in advance. This condition guarantees that, whichever support is locked by the algorithm, the subsequent restricted OLS estimator is well defined.

These conditions are satisfied by many commonly used action sets. For example, suppose that there exists a constant $r>0$ such that $r\mathbb B_2^d\subseteq\X$, where $\mathbb B_2^d$ denotes the unit Euclidean ball in $\mathbb R^d$. Let $\mathbf u_j^{(d)}\in\mathbb R^d$ denote the $j$-th standard basis vector of $\mathbb R^d$, i.e., its $j$-th coordinate is 1
and all other coordinates are 0. Then the $d$ scaled coordinate actions $b_j = r\mathbf u_j^{(d)},\ j=1,\dots,d$, are feasible actions in $\X$. Their Gram matrix is $B=\sum_{j=1}^d b_jb_j^\top=r^2\sum_{j=1}^d\mathbf u_j^{(d)}\mathbf u_j^{(d)\top}=r^2 I_d$. Therefore, $\lambda_{\min}(B)=r^2$, and Assumption~\ref{ass:arms during exploration}(i) holds with $\lambda_0=r^2$. Next, consider any candidate support $ S=\{i_1,\dots,i_H\}\subseteq[d]$, where the elements of $S$ are listed according to a fixed ordering. For this support, choose the $H$ restricted exploration actions as $e_k(S)=r\mathbf u_{i_k}^{(d)},\ k=1,\dots,H$. Since $i_k\in S$, each $e_k(S)$ is supported entirely on $S$ and hence belongs to $\X(S)$. Recall that $e_k^\prime(S)\in\mathbb R^H$ in Assumption~\ref{ass:arms during exploration}(ii) denotes the restriction of the $d$-dimensional vector $e_k(S)$ to the coordinates indexed by $S$, ordered as $(i_1,\dots,i_H)$. Let $\mathbf u_k^{(H)}\in\mathbb R^H$ denote the $k$-th standard basis
vector of $\mathbb R^H$. Under this coordinate restriction, $e_k^\prime(S)=r\mathbf u_k^{(H)},\ k=1,\dots,H$. Consequently, the restricted Gram matrix satisfies $B_H(S)=\sum_{k=1}^H e_k^\prime(S)e_k^\prime(S)^\top =r^2\sum_{k=1}^H \mathbf u_k^{(H)}\mathbf u_k^{(H)\top}=r^2 I_H$. Hence, $\lambda_{\min}\bigl(B_H(S)\bigr)=r^2$ for every support $S\subseteq[d]$ with $|S|=H$. Therefore, Assumption~\ref{ass:arms during exploration}(ii) holds uniformly over all candidate supports with $\lambda_H=r^2$. Thus, both parts of Assumption~\ref{ass:arms during exploration} are satisfied with $\lambda_0=\lambda_H=r^2$.

More generally, the assumption does not require the exploration actions to be coordinate vectors or orthogonal. It only requires the corresponding Gram matrices to remain non-singular, with the restricted Gram matrices uniformly well-conditioned over the admissible supports. This type of identifiability condition is standard in phased exploration methods for linear bandits; see, for example, \cite[Assumption~1(b)]{rusmevichientong2010linearly}. In Section~\ref{sec:numerical experiments}, we further investigate the sensitivity of the proposed algorithms to different feasible choices of exploration actions and observe comparable sub-linear regret behavior whenever the required Gram-matrix regularity is satisfied.

\begin{algorithm}[!ht]
\small
\caption{APSEE}\label{alg:adaptive_PEE}
\SetKwComment{tcp}{}{} 
\SetKwInOut{Input}{Input}
\SetKwInOut{Requires}{Requires}
\Input{Dimension $d$, sparsity $H$, confidence parameter $\delta$.}

\BlankLine

\textbf{Initialize}: $B = \sum\limits_{k=1}^d b_k(S_k) b_k(S_k)^\top$ by Assumption~\ref{ass:arms during exploration}(i), $c=1$, $t=1$.\\
\textbf{Flag}: $\texttt{SupportFound} = \text{False}$, $S_{est} = \emptyset$, $m=0$, $C_0=\infty$.\\

\For{$c=1,2,\dots$}{
    \If{\texttt{SupportFound} is \text{False}}{
        \tcp*[h]{\textcolor{blue}{\textit{--- Full Exploration Phase (d steps) ---}}}\vspace{2pt}\\
        \For{$k=1, \dots, d$}{
            Play action $b_k$ and observe reward $Y_t$.\\
            Set $Y_k(c)=Y_t$.\\
            $t \leftarrow t+1$.\\
        }
        
        Compute OLS estimate $\hat{\theta}_c$ using all past data:\quad $\displaystyle  \hat{\theta}_c = (cB)^{-1} \sum_{s=1}^c \sum_{k=1}^d b_k Y_k(s)$ \hfill \refstepcounter{equation}(\theequation)\label{eqn:sparse-PEE-ols}.\\
        
        Compute error bound $\varepsilon_c = \sqrt{\frac{h_1 \ln(2dc^2/\delta)}{c}}$.\\
        Sort indices $\{i_1^\prime, \dots, i_d^\prime\}$ such that $|(\hat{\theta}_c)_{i_1^\prime}| \ge \dots \ge |(\hat{\theta}_c)_{i_d^\prime}|$.\\
        Compute $\Delta_c^\prime = |(\hat{\theta}_c)_{i_H^\prime}|- |(\hat{\theta}_c)_{i_{H+1}^\prime}|$.\\
        
        \If{$\Delta_c^\prime > 2\varepsilon_c$}{
            Set $S_{est} = \{i_1^\prime, \dots, i_H^\prime\}$.\\
            $\texttt{SupportFound} \leftarrow \text{True}$, $C_0 \leftarrow c$.\\
            Compute $B_H = \sum_{k=1}^H e_k^\prime e_k^{\prime\top}$ by Assumption~\ref{ass:arms during exploration}(ii).\\
        }
    }
    \Else{
        \tcp*[h]{\textcolor{blue}{\textit{--- Restricted Exploration Phase (H steps) ---}}}\vspace{2pt}\\
        $m \leftarrow m+1$.\\
        \For{$k=1, \dots, H$}{
            Zero-pad $e_k^\prime\in\R^H$ outside the coordinates in $S_{est}$ to obtain $e_k(S_{est})\in\R^d$, play $e_k(S_{est})$, and observe reward $\tilde{Y}_t$.\\
            Set $\tilde{Y}_k(m)=\tilde{Y}_t$.\\
            $t \leftarrow t+1$.\\
        }
        Compute restricted OLS estimate $\bar{\theta}_m^{(H)}\in\R^H$ on $S_{est}$:\quad $\displaystyle  \bar{\theta}_m^{(H)}= (m B_H)^{-1} \sum_{s=1}^m \sum_{k=1}^H e_k^\prime \tilde{Y}_k(s)$ \hfill \refstepcounter{equation}(\theequation)\label{eqn:sparse-PEE-olseqn:restricted-ols}.\\
        Zero-pad $\bar{\theta}_m^{(H)}$ outside the support $S_{est}$ to construct the $d$-dimensional estimator $\bar{\theta}_m^{(d)} \in \R^d$.\\
    }

    \tcp*[h]{\textcolor{blue}{\textit{--- Exploitation Phase (c steps) ---}}}\vspace{2pt}\\
    \If{$\texttt{SupportFound}$ is \text{True}}{ 
        Set working estimate $\theta_{work} = \ind{m=0}\hat{\theta}_{C_0}+\ind{m\neq0}\bar{\theta}_m^{(d)}$.\\
        Compute the best action on support $S_{est}$: $a_c = \argmax_{\x \in \X, \supp(\x) \subseteq S_{est}} \theta_{work}^\top \x$.\\
        \For{$j=1, \dots, c$}{
            Play action $a_c$ and observe reward $Y_t$.\\
            $t \leftarrow t+1$.\\
        }
    }
    $c \leftarrow c+1$.\\
}
\end{algorithm}

\subsection{Adaptive Support Recovery}
\label{subsec:Guaranteeing Support Consistency}
To start with, we establish the theoretical correctness of our adaptive support detection mechanism (lines 10-16 of Algorithm~\ref{alg:adaptive_PEE}). Since this mechanism relies on the OLS estimator $\hat{\theta}_c$ computed from exploration actions, we provide the following high-probability {\it anytime} upper bound on the estimation error.
\begin{lemma}[Anytime High-Probability Bound on OLS Error]
\label{lemma:ols-bound-hp}
Consider the OLS estimator $\hat{\theta}_c$ (expressed as \eqref{eqn:sparse-PEE-ols} in Algorithm~\ref{alg:adaptive_PEE}). Let $\delta \in (0, 1)$. With probability at least $1-\delta$, for \textbf{all} cycles $c \ge 1$ in Algorithm~\ref{alg:adaptive_PEE} simultaneously, the estimation error satisfies:
\begin{equation}\label{eqn:ols-bound-hp}
    \norm{\hat{\theta}_c - \theta_*}_2^2 \le \frac{h_1 \ln(2d c^2/\delta)}{c} \triangleq \varepsilon_c^2,
\end{equation}
where $h_1 = \frac{2 \sigma^2 \mathrm{Tr}(B)}{\lambda_0^2}$.
\end{lemma}
The proof of Lemma~\ref{lemma:ols-bound-hp} and details for the OLS estimate during the Restricted Exploration Phase (line 23 of Algorithm~\ref{alg:adaptive_PEE}) can be found in Appendix~\ref{app:Omitted Proofs in sec:euclidean_case} (Corollary~\ref{coro:restricted_ols}). Building on this error bound of $\hat{\theta}_c$, we demonstrate that monitoring the empirical gap $\Delta_c^\prime = |(\hat{\theta}_c)_{i_H^\prime}| - |(\hat{\theta}_c)_{i_{H+1}^\prime}|$ is sufficient to guarantee that the estimated support $S_c^\star$ (with $S_c^\star=S_{est}$ in line 14) matches the true support $S^\star$.

\begin{lemma}[Support Consistency for Euclidean Ball Action Sets]
\label{lemma:stopping_condition_l2ball}
Consider any cycle $c \ge 1$. Suppose $\|\hat{\theta}_c - \theta_*\|_2 \le \varepsilon_c$ holds.
Let the sorted indices of $\theta_*$ be $\{i_1^\star, \dots, i_d^\star\}$ such that $|(\theta_*)_{i_1^\star}| \ge \dots \ge |(\theta_*)_{i_d^\star}|$, i.e., the true support is $S^\star = \{i_1^\star, \dots, i_H^\star\}$. Similarly, let the sorted indices of $\hat{\theta}_c$ be $\{i_1^\prime, \dots, i_d^\prime\}$ with estimated support $S_c^\star = \{i_1^\prime, \dots, i_H^\prime\}$.
If the empirical gap satisfies $\Delta_c^\prime = |(\hat{\theta}_c)_{i_H^\prime}| - |(\hat{\theta}_c)_{i_{H+1}^\prime}| > 2\varepsilon_c$, then $S_c^\star = S^\star$.
\end{lemma}
The proof of Lemma~\ref{lemma:stopping_condition_l2ball} can be found in Appendix~\ref{app:Omitted Proofs in sec:euclidean_case}. This result provides the theoretical foundation for the warm-up (pure exploration) phase of our algorithm. Once the condition $\Delta_c^\prime > 2\varepsilon_c$ is met, we safely transition to the restricted basis exploration and exploitation phases, ensuring low regret.

\subsection{Regret Analysis for Algorithm~\ref{alg:adaptive_PEE}}
Having established that Algorithm~\ref{alg:adaptive_PEE} correctly identifies the support $S^\star$ before entering the exploitation phase, we can now bound its cumulative regret defined in~\eqref{eqn:def of a-regret}. Note that we have $\alpha=1$ in \eqref{eqn:def of a-regret}, due to Proposition~\ref{prop:l2_exact_solution}, i.e., we compare against the true optimal solution in highsight.

\begin{theorem}[Regret Bounds for APSEE]
\label{thm:regret_adaptive}
Suppose Assumptions~\ref{ass:feasible set bound}-\ref{ass:arms during exploration} hold. Assume the unknown parameter $\theta_*$ has a non-zero signal gap $\Delta_{\min} = |(\theta_*)_{i_H^\star}| - |(\theta_*)_{i_{H+1}^\star}| > 0$, where $i_H^*$ and $i_{H+1}^*$ are described in Lemma~\ref{lemma:stopping_condition_l2ball}. 
Let $\delta\in(0,1)$. Then, with probability at least $1-\delta$:
\\\noindent $\bullet$ If $T \le d C_0$, the cumulative regret is bounded by the initial exploration cost:
    \begin{equation}
         R_T \le C_{gap} \triangleq \tilde{\mathcal{O}}\left( \frac{d L_{\max} \norm{\theta_*}_2 h_1}{\Delta_{\min}^2} \right).
    \end{equation}
\\\noindent $\bullet$ If  $T > d C_0$, the cumulative regret satisfies
    \begin{equation}
        R_T = \tilde{\CO}\left(\left(\norm{\theta_*}_2+\frac{1}{\zeta_\star}\right)H\sqrt{T}\right) + C_{gap}.
    \end{equation}
Here, $\tilde{\CO}(\cdot)$ hides $\log T$ factors, $C_0={\mathcal{O}}(h_1 \Delta_{\min}^{-2})$ is the duration of the warm-up phase in Algorithm~\ref{alg:adaptive_PEE}, $h_1$ is defined in Lemma~\ref{lemma:ols-bound-hp} and $\zeta_\star$ is defined by \eqref{eqn:zeta_star}.
\end{theorem}
\noindent{\it Proof Sketch.} The proof decomposes the regret into the warm-up exploration cost, the restricted exploration cost, and the exploitation cost after support recovery. By Lemmas~\ref{lemma:ols-bound-hp} and \ref{lemma:stopping_condition_l2ball}, the adaptive stopping rule identifies the optimal support $S^\star$ after a finite number of warm-up cycles
$C_0=\tilde{\mathcal O}(h_1\Delta_{\min}^{-2})$, which yields the instance-dependent term $C_{gap}$. After support recovery, exploration is restricted to only $H$ directions. Moreover, Corollary~\ref{coro:restricted_ols} controls the restricted OLS error, while the directional stability of the Euclidean-ball optimizer converts the parameter-estimation error into a quadratic exploitation-regret bound. Since the total number of cycles satisfies $K=\Theta(\sqrt T)$, summing the restricted exploration and exploitation costs gives the stated $\tilde{\mathcal O}(H\sqrt T (\|\theta_*\|_2+\zeta_\star^{-1}))$ horizon-dependent term. The detailed proof is provided in Appendix~\ref{app:Omitted Proofs in sec:euclidean_case}.

It is worth mentioning that we establish a high-probability regret bound, which distinguishes our result from the expected regret analysis found in \cite{rusmevichientong2010linearly} for stochastic linear bandits without sparsity constraint. The work in \cite{rusmevichientong2010linearly} assumes the unknown parameter is a random variable, thereby deriving an $\O(d\sqrt{T})$ bound on the expected regret (Bayes risk). In the context of general stochastic linear bandits \cite{abbasi2011improved,lattimore2020bandit}, high-probability regret bounds are widely considered. Our analysis also reveals that by decoupling the environmental dimension $d$ from the stable exploration through the restricted basis mechanism, we achieve an asymptotic $\tilde{\CO}(H\sqrt{T})$ regret. The $d$-dependence is absorbed into the instance-dependent constant $C_{gap}$, which does not scale with $T$. Furthermore, our approach is based on the Explore-Then-Commit (ETC) framework \cite{rusmevichientong2010linearly} that does not require prior knowledge of $\norm{\theta_*}_2$, which offers a significant advantage over standard Upper Confidence Bound (UCB) algorithms \cite{abbasi2011improved,lattimore2020bandit} that typically necessitate an explicit upper bound on $\norm{\theta_*}_2$ to construct valid confidence sets.

\section{Smooth Convex Action Sets}
\label{sec:smooth_convex_case}

We now consider a broader class of action sets $\X$ with smooth boundaries. Throughout this section, we assume that $\X\subset\R^d$ is a compact, full-dimensional convex set
containing the origin in its interior, and that its boundary $\partial\X$ is continuously differentiable (i.e., smooth). We refer to such a set as a smooth convex action set. This class includes ellipsoids and $\ell_p$-balls for $p>1$. In particular, the smoothness assumption ensures that every boundary point has a unique outward normal vector, without requiring a uniform positive lower bound on the boundary curvature.

Unlike the $\ell_2$-ball case, an explicit closed-form solution for the sparse optimization problem is generally unavailable for a general smooth convex set \cite{natarajan1995sparse}. Consequently, obtaining the exact optimal sparse action typically requires enumerating all possible support sets, which incurs a prohibitive computational cost. To address this computational challenge, we adopt the Greedy algorithm
\cite{nemhauser1978analysis}, a widely used and efficient heuristic for combinatorial optimization. While the Greedy algorithm may not recover the global optimum, it guarantees an $\alpha$-approximate solution, where the approximation ratio $\alpha$ is determined by the \textit{submodularity ratio} of the objective function. Accordingly, our subsequent theoretical analysis focuses on bounding the cumulative $\alpha$-regret defined in \eqref{eqn:def of a-regret}. The compactness of $\X$ ensures that the associated support-value function is Lipschitz continuous with respect to the parameter. For the exploitation analysis, we will also explicitly state the regularity condition required for the restricted optimal-action mapping.

\subsection{Adaptive Support Recovery via Greedy Gap and APSEE-G Algorithm}
\label{subsec:Adaptive Support Recovery via Greedy Gap}

\begin{algorithm}[!t]
\small
\caption{APSEE-G}
\label{alg:adaptive_general_PEE}
\SetKwComment{tcp}{}{}
\SetKwInOut{Input}{Input}
\SetKwInOut{Requires}{Requires}
\Input{Dimension $d$, sparsity $H$, confidence parameter
$\delta$, action set $\X$, and Lipschitz constant $L_{\X}$.}

\BlankLine
\textbf{Initialize}: $B = \sum\limits_{k=1}^d
b_k(S_k)b_k(S_k)^\top$ by
Assumption~\ref{ass:arms during exploration}(i),
$c=1$, $t=1$.\\
\textbf{Flag}: $\texttt{SupportFound} = \text{False}$,
$S_{est} = \emptyset$, $m=0$, $C_0=\infty$.\\

\For{$c=1,2,\dots$}{
    \If{\texttt{SupportFound} is \text{False}}{
        \tcp*[h]{\textcolor{blue}{\textit{--- Full Exploration Phase ($d$ steps) ---}}}\vspace{2pt}\\
        Run lines 5--10 of Algorithm~\ref{alg:adaptive_PEE}
        (Full Exploration, OLS update $\hat{\theta}_c$,
        and $\varepsilon_c$).\\
        Let $G_0 = \emptyset$ and
        $\Delta_c^{Greedy} = \infty$.\\
        \For{$k = 1, \dots, H$}{
            Compute $\rho_k(v;\hat{\theta}_c)
            =h(G_{k-1}\cup\{v\};\hat{\theta}_c)
            -h(G_{k-1};\hat{\theta}_c)$,
            $\quad\forall v\in[d]\setminus G_{k-1}$.\\
            Find
            $g_k=\argmax_v\rho_k(v;\hat{\theta}_c)$.\\
            Find
            $v_{\mathrm{second}}
            =\argmax_{v\neq g_k}\rho_k(v;\hat{\theta}_c)$.\\
            Compute
            $\hat{\delta}_k
            =\rho_k(g_k;\hat{\theta}_c)
            -\rho_k(v_{\mathrm{second}};\hat{\theta}_c)$.\\
            Update minimum gap:
            $\Delta_c^{Greedy}
            =\min(\Delta_c^{Greedy},\hat{\delta}_k)$.\\
            Update set:
            $G_k=G_{k-1}\cup\{g_k\}$.\\
        }
        \If{$\Delta_c^{Greedy}>2L_{\X}\varepsilon_c$}{
            Set $S_{est}=G_H$.\\
            $\texttt{SupportFound}\leftarrow\text{True}$.\\
            $C_0\leftarrow c$.\\
            Compute
            $B_H=\sum_{k=1}^H e_k^\prime e_k^{\prime\top}$
            by Assumption~\ref{ass:arms during exploration}(ii).\\
        }
    }
    \Else{
        \tcp*[h]{\textcolor{blue}{\textit{--- Restricted Exploration Phase ($H$ steps) ---}}}\vspace{2pt}\\
        $m\leftarrow m+1$.\\
        \For{$k=1,\dots,H$}{
            Play the corresponding zero-padded action
            $e_k(S_{est})\in\R^d$ associated with $e_k^\prime$
            and observe reward $\tilde Y_t$.\\
            Set $\tilde Y_k(m)=\tilde Y_t$.\\
            $t\leftarrow t+1$.\\
        }
        Compute the restricted OLS estimate
        $\bar{\theta}_m^{(H)}\in\R^H$ on $S_{est}$: $\bar{\theta}_m^{(H)}
        =(mB_H)^{-1}
        \sum_{s=1}^m\sum_{k=1}^H
        e_k^\prime\tilde Y_k(s)$. \\
        Zero-pad $\bar{\theta}_m^{(H)}$ outside $S_{est}$
        to construct the $d$-dimensional estimator
        $\bar{\theta}_m^{(d)}\in\R^d$.\\
    }

    \tcp*[h]{\textcolor{blue}{\textit{--- Exploitation Phase ($c$ steps) ---}}}\vspace{2pt}\\
    \If{$\texttt{SupportFound}$ is \text{True}}{
        Set the working estimate
        $\theta_{work}
        =\ind{m=0}\hat{\theta}_{C_0}
        +\ind{m\neq0}\bar{\theta}_m^{(d)}$.\\
        Compute the optimal action on the identified support: $a_c=\argmax_{\x\in\X,\ \supp(\x)\subseteq S_{est}}
        \theta_{work}^\top\x$.
        \For{$j=1,\dots,c$}{
            Play action $a_c$ and observe reward $Y_t$.\\
            $t\leftarrow t+1$.\\
        }
    }
    $c\leftarrow c+1$.\\
}
\end{algorithm}

We introduce the Adaptively Phased Sparse Exploration and Exploitation with Greedy (\textbf{APSEE-G}) algorithm (Algorithm~\ref{alg:adaptive_general_PEE}). Similar to the
Euclidean case in Section~\ref{sec:euclidean_case}, to avoid the suboptimal ambient dimension dependency $d$ in the long term, APSEE-G incorporates a \textit{Restricted Exploration} phase within an $H$-dimensional subspace.

Define the maximum expected reward achievable on a fixed support set $S\subseteq[d]$ with parameter $\theta$ as
\begin{equation}
\label{eqn:def of h(S;)}
h(S;\theta)\triangleq\max_{\x\in\X,\ \supp(\x)\subseteq S}\theta^\top\x,
\qquad S\subseteq[d].
\end{equation}
For any fixed support set $S$, the function $h(S;\theta)$ is convex and Lipschitz continuous with respect to $\theta$. In particular, letting $L_{\X}=\max_{\x\in\X}\|\x\|_2$, we have
\begin{equation}
\label{eqn:support_value_lipschitz}
|h(S;\theta)-h(S;\theta')|
\leq L_{\X}\|\theta-\theta'\|_2.
\end{equation}
To proceed, let $S^G$ and $S_c^G$ be the support sets obtained by running the Greedy algorithm on the true parameter $\theta_*$ and the estimate $\hat{\theta}_c$ (obtained as in \eqref{eqn:sparse-PEE-ols} of Algorithm~\ref{alg:adaptive_PEE}), respectively.
Specifically, given any parameter $\theta$, the Greedy algorithm constructs the support set iteratively. Starting with $G_0\leftarrow\emptyset$, at each step $k=1,\dots,H$, it adds an element $v\in[d]\setminus G_{k-1}$ that maximizes the marginal gain
\[
\rho_k(v;\theta)
=h(G_{k-1}\cup\{v\};\theta)-h(G_{k-1};\theta).
\]
For subsequent restricted estimation, we denote by $\theta_*^G
\triangleq\theta_*^\prime(S^G)\in\R^H$ the sub-vector of the true parameter corresponding to the greedy support $S^G$. The notation $\theta_*^G\in\R^H$ only denotes the restriction of $\theta_*$ to the coordinates selected by the Greedy algorithm. We further define the restricted signal strength on the greedy support as
\begin{equation}
\label{eqn:zeta_G}
\zeta_G\triangleq\|\theta_*^G\|_2=\|\theta_*^\prime(S^G)\|_2=\|\theta_*(S^G)\|_2.
\end{equation}
The quantity $\zeta_G$ measures the signal strength of the true parameter restricted to the support selected by the Greedy algorithm.

\textbf{Greedy Gap and Lipschitz Continuity.}
We analyze the consistency of the Greedy algorithm by defining the gap at each step $k\in\{1,\dots,H\}$. Let $G_{k-1}^*$ be the set of size $k-1$ selected by the Greedy algorithm given $\theta_*$. The marginal gain of an element $v$ is
\[
\rho_k(v;\theta_*)=h(G_{k-1}^*\cup\{v\};\theta_*)-h(G_{k-1}^*;\theta_*),\ \text{for all}\ v\in[d]\setminus G_{k-1}^*.
\]
Let $g_k^*$ be the best element and $v_k^{\mathrm{second}}$ be the second-best element at step $k$, i.e., $g_k^*=\argmax_{v\in[d]\setminus G_{k-1}^*}\rho_k(v;\theta_*)$, 
and $v_k^{\mathrm{second}}=\argmax_{\substack{v'\in[d]\setminus G_{k-1}^*,v'\neq g_k^*}}\rho_k(v';\theta_*)$. We define the step-wise Greedy gap as 
\[
\delta_k
=\rho_k(g_k^*;\theta_*)
-\rho_k(v_k^{\mathrm{second}};\theta_*),
\]
and the minimum Greedy gap as
\begin{equation}
\label{eqn:delta_min}
\Delta_{\min}^G
=\min_{k=1,\dots,H}\delta_k.
\end{equation}
We assume $\Delta_{\min}^G>0$, which is an
instance-dependent identifiability condition ensuring that
the greedy maximizer is unique at every step. To relate the estimation error to the true support recovery, we use the Lipschitz property
\eqref{eqn:support_value_lipschitz}. Given that
$\|\hat{\theta}_c-\theta_*\|_2\leq\varepsilon_c$, we have
\[
|h(S;\hat{\theta}_c)-h(S;\theta_*)|
\leq L_{\X}\varepsilon_c.
\]
Consequently, the error in estimating a marginal gain is
bounded by
\begin{align}
\left|\rho_k(v;\hat{\theta}_c)-\rho_k(v;\theta_*)\right|
&\leq\left|h(G\cup\{v\};\hat{\theta}_c)-h(G\cup\{v\};\theta_*)\right|
+\left|h(G;\hat{\theta}_c)-h(G;\theta_*)
\right|\nonumber\\
&\leq 2L_{\X}\varepsilon_c.
\label{eqn:greedy marginal gain}
\end{align}

\textbf{Stopping Condition and Consistency Lemma.}
To guarantee that $S_c^G=S^G$ for some $c\geq1$, the empirical gap $\Delta_c^{Greedy}$, calculated during the Greedy procedure using the estimate $\hat{\theta}_c$, must be sufficiently large to overcome the estimation error of $\hat{\theta}_c$. Specifically, we need to distinguish the best element $g_k^*$ from the second-best element $v_k^{\mathrm{second}}$ despite the estimation error. Similar to the discussion in Section~\ref{subsec:Guaranteeing Support Consistency}, we have the following lemma, whose proof is provided in Appendix~\ref{app:proofs_smooth_convex_case}.

\begin{lemma}[Greedy Support Consistency]
\label{lemma:general_consistency}
Suppose that $\|\hat{\theta}_c-\theta_*\|_2\leq\varepsilon_c$ holds. Let $\Delta_c^{Greedy}$ be the minimum empirical marginal gain gap observed during the Greedy algorithm execution on $\hat{\theta}_c$. If the stopping condition $\Delta_c^{Greedy}>2L_{\X}\varepsilon_c$ holds, then the support set $S_c^G$ obtained by running the Greedy algorithm on $\hat{\theta}_c$ is identical to the support set $S^G$ obtained by running the Greedy algorithm on the true parameter $\theta_*$, i.e., $S_c^G=S^G$.
\end{lemma}

Note that the greedy support $S^G$ need not coincide with the globally optimal sparse support $S^\star$. Lemma~\ref{lemma:general_consistency} only guarantees that the empirically recovered support $S_c^G$ coincides with the greedy support $S^G$ obtained under the true parameter $\theta_*$. The approximation guarantee of the Greedy algorithm \cite[Theorem 6]{das2018approximate}, $h(S^G;\theta_*)\geq\alpha h(S^\star;\theta_*)$, 
then connects the greedy support to the true optimal benchmark. Moreover, after support identification, only the true parameter's restriction $\theta_*^G=\theta_*^\prime(S^G)\in\R^H$ needs to be estimated.

\textbf{Regularity of the Restricted Optimal-Action Mapping.} The smoothness of $\partial\X$ is sufficient for the submodularity-ratio result established below. However,
smoothness alone does not necessarily imply Lipschitz continuity of the optimal-action mapping. The exploitation analysis requires the following additional regularity condition.

\begin{assumption}[Lipschitz Continuity of Restricted Optimal Actions]
\label{ass:smooth_restricted_optimizer}
For every support set $S\subseteq[d]$ with $|S|=H$, let $\X(S)=\{\x\in\X:\supp(\x)\subseteq S\}$. For each nonzero $u\in\R^H$, suppose that the maximizer $\x_S^\star(u)\in\argmax_{\x\in\X(S)}u^\top\x(S)$ is unique. Moreover, there exists a constant $L_{\mathrm{grad}}>0$ such that, for all unit vectors
$u,v\in\R^H$,
\begin{equation}
\label{eqn:restricted_optimizer_lipschitz}
\|\x_S^\star(u)-\x_S^\star(v)\|_2\leq L_{\mathrm{grad}}\|u-v\|_2.
\end{equation}
If an estimated parameter is zero, the algorithm uses a fixed unit direction for tie-breaking.
\end{assumption}

\begin{remark}[A Sufficient Condition for Restricted Optimizer Regularity]
Strong convexity of $\X$ is a sufficient condition for Assumption~\ref{ass:smooth_restricted_optimizer}. In particular, under the strong-convexity notion adopted in \cite{ijcai2026p560}, the strong convexity of $\X$ is inherited by its coordinate-restricted action sets, which ensures the uniqueness and Lipschitz continuity of their optimal-action mappings. This follows from \cite[Proposition 1]{ijcai2026p560}. Therefore, smooth strongly convex action sets constitute an important subclass satisfying the restricted-optimizer regularity condition required by APSEE-G.
\end{remark}

\subsection{Regret Analysis for Algorithm~\ref{alg:adaptive_general_PEE}}

We introduce the submodularity ratio for set functions,
which plays a crucial role in proving the approximation
ratio of Greedy algorithms \cite{das2018approximate}.

\begin{definition}
\label{def:submodularity ratio}
The submodularity ratio of a set function
$g:2^{[n]}\to\R_{\geq0}$ is the largest
$\gamma\in\R$ such that
\begin{align*}
\sum_{\omega\in\Omega\setminus S}
\left(g(S\cup\{\omega\})-g(S)\right)
\geq
\gamma\left(g(S\cup\Omega)-g(S)\right),
\quad
\text{for all }S,\Omega\subseteq[n].
\end{align*}
\end{definition}

\begin{remark}
\label{remark:submodularity ratio 1}
For a nonnegative and nondecreasing set function
$g(\cdot)$ with submodularity ratio $\gamma$, we have
$\gamma\in[0,1]$, and $g(\cdot)$ is submodular if and only
if $\gamma=1$ \cite{bian2017guarantees}. Besides, we
follow the description in \cite[Definition 2]{das2018approximate}
and let $0/0:=1$.
\end{remark}

The following lemma describes the approximation ratio
of the Greedy algorithm.

\begin{lemma}\cite[Theorem 6]{das2018approximate}
\label{lemma:approximation ratio expression}
Let $g:2^{[d]}\to\R_{\geq0}$ be a nonnegative,
monotone nondecreasing set function with $g(\emptyset)=0$,
and let $OPT$ be the maximum value of $g(\cdot)$ over any
set of size at most $H$. Then, the set $S^G$ selected by
the Greedy algorithm has the following approximation
guarantee:
\begin{align}
g(S^G)
\geq
\left(1-e^{-\gamma}\right)OPT,
\label{eqn:a-approximate optimal ratio greedy}
\end{align}
where $\gamma$ is the submodularity ratio of $g(\cdot)$.
\end{lemma}

Recalling the $\alpha$-regret defined in
\eqref{eqn:def of a-regret}, we let
$\alpha=1-e^{-\gamma}$ be the approximation ratio
\cite{fourati2024combinatorial,ye2025online}, and derive
the corresponding $\alpha$-regret upper bound, where
$\gamma$ is the submodularity ratio of the set function
$h(S;\theta_*)$ defined in \eqref{eqn:def of h(S;)}.
We now establish that $\gamma>0$, which implies that
$\alpha>0$ and hence the $\alpha$-regret is non-vacuous.

\begin{proposition}[Submodularity Ratio for Smooth Convex Sets]
\label{prop:submodularity_general}
Let $\X\subset\R^d$ be a compact, full-dimensional convex set containing the origin in its interior, and suppose that $\partial\X$ is continuously differentiable. For a fixed nonzero parameter $\theta\in\R^d$, consider the set function $h(S;\theta)=\max_{\x\in\X,\ \supp(\x)\subseteq S}\theta^\top\x,\ S\subseteq[d]$. Define the minimum positive marginal gain as
\begin{equation}
\label{eqn:minimum_positive_marginal_gain}
\delta_{\mathrm{marg}}\triangleq\min_{\substack{S\subseteq[d],\ i\in[d]\setminus S,
h(S\cup\{i\};\theta)>h(S;\theta)}}\left[h(S\cup\{i\};\theta)-h(S;\theta)\right].
\end{equation}
Then, the submodularity ratio $\gamma$ of $h(\cdot;\theta)$ admits the lower bound
\begin{equation}
\label{eqn:submodularity_ratio_smooth_lower_bound}
\gamma\geq\frac{\delta_{\mathrm{marg}}}{h([d];\theta)}>0.
\end{equation}
\end{proposition}

\begin{remark}[]
\label{remark:conference_version_correction}
The above result does not hold under the strong-convexity notion as in \cite{ijcai2026p560}. In particular, an intersection of equal-radius Euclidean balls can have nonsmooth boundary points at which multiple outward normal vectors coexist. At such points, individual coordinate extensions may have zero marginal gain even though a joint extension produces a strictly positive gain. Therefore, the previous argument does not establish a positive submodularity ratio for the entire strongly convex class. 
In contrast, the smoothness condition imposed here ensures that every relevant boundary point has a unique outward normal vector. This property rules out the above zero-singleton-gain/positive-joint-gain phenomenon and allows us to establish the positive lower bound in \eqref{eqn:submodularity_ratio_smooth_lower_bound}. Finally, under the strong-convexity condition adopted in \cite{ijcai2026p560}, one can establish only the nonnegativity of the submodularity ratio, i.e., $\gamma\geq 0$, but not its strict positivity.
\end{remark}

For an ellipsoid $\X=\{\x\in\R^d:\x^\top A\x\leq1\},\quad A\succ0$, the above lower bound can be further refined to
\begin{equation}
\label{eqn:ellipsoid_gamma_bound}
\gamma\geq\frac{\lambda_{\min}(A)}{\max_{i\in[d]}A_{ii}}\geq\frac{\lambda_{\min}(A)}{\lambda_{\max}(A)}>0.
\end{equation}
In particular, if $\lambda_{\max}(A)\leq1$, then $\gamma\geq\lambda_{\min}(A)$. The proof of Proposition~\ref{prop:submodularity_general} and the derivation of the ellipsoidal bound are provided in Appendix~\ref{app:proofs_smooth_convex_case}. Under the above guarantees and the condition that the Greedy supports match (i.e., $S_c^G=S^G$), we derive the regret upper bound for Algorithm~\ref{alg:adaptive_general_PEE}.

\begin{theorem}[Regret Bounds for APSEE-G]
\label{thm:regret_general}
Suppose Assumptions~\ref{ass:feasible set bound}--
\ref{ass:arms during exploration} and
Assumption~\ref{ass:smooth_restricted_optimizer} hold.
Assume further that
$\Delta_{\min}^G>0$ and $\zeta_G>0$.
Let
$L_{\max}=\max_{\x\in\X}\|\x\|_2$.
For any $\delta\in(0,1)$, with probability at least
$1-\delta$, the following bounds hold:

\noindent$\bullet$ If $T\leq dC_0^G$, the cumulative
$\alpha$-regret satisfies
\begin{equation}
\alpha\text{-}R_T
\leq C_{\mathrm{gap}}^G.
\end{equation}

\noindent$\bullet$ If $T>dC_0^G$, the cumulative
$\alpha$-regret satisfies
\begin{equation}
\alpha\text{-}R_T
\leq
\widetilde{\mathcal O}
\left(
H\sqrt T
\left(
L_{\max}\|\theta_*\|_2
+\frac{L_{\mathrm{grad}}}{\zeta_G}
\right)
\right)
+C_{\mathrm{gap}}^G.
\end{equation}

Here, $C_0^G=
\widetilde{\mathcal O}
\left(
\frac{h_1}{(\Delta_{\min}^G)^2}
\right)$ is the duration of the warm-up phase in Algorithm~\ref{alg:adaptive_general_PEE}, and $h_1$ is defined in Lemma~\ref{lemma:ols-bound-hp}. The instance-dependent term $C_{\mathrm{gap}}^G$ collects the finite warm-up and lower-order contributions: $C_{\mathrm{gap}}^G=\widetilde{\mathcal O}(\frac{dL_{\max}\|\theta_*\|_2h_1}{(\Delta_{\min}^G)^2}+\frac{L_{\mathrm{grad}}h_1}{\zeta_G}+\frac{HL_{\mathrm{grad}}C_0^G}{\zeta_G})$. Moreover, $\alpha=1-e^{-\gamma}$, where $\gamma$ is the submodularity ratio of $h(\cdot;\theta_*)$ defined in
\eqref{eqn:def of h(S;)}, and $\zeta_G$ is defined in \eqref{eqn:zeta_G}.
\end{theorem}

\noindent
\textit{Proof Sketch.}
The proof follows the same phased decomposition as for APSEE, but the exact support $S^\star$ is replaced by the greedy support $S^G$. Lemma~\ref{lemma:general_consistency},
together with the anytime OLS concentration bound, guarantees that the empirical Greedy procedure identifies $S^G$ after a finite warm-up period, producing the instance-dependent cost $C_{\mathrm{gap}}^G$. The positive submodularity ratio established in
Proposition~\ref{prop:submodularity_general} yields the Greedy approximation guarantee
$h(S^G;\theta_*)\geq\alpha h(S^\star;\theta_*)$. After support recovery, only $H$ restricted exploration actions are played per cycle. The restricted OLS
concentration bound, together with Assumption~\ref{ass:smooth_restricted_optimizer},
controls the exploitation error on $S^G$. Summing over $K=\Theta(\sqrt T)$ cycles yields the stated $\widetilde{\mathcal O} \big(H\sqrt T(L_{\max}\|\theta_*\|_2+ L_{\mathrm{grad}}/\zeta_G)\big)$ term. The detailed proof is provided in Appendix~\ref{app:proofs_smooth_convex_case}.

The computational cost of APSEE-G is dominated by the Greedy support-identification procedure and the subsequent restricted linear optimization. At each Greedy step $k$, at most $d-k+1$ candidate coordinates are evaluated. Hence, one complete Greedy execution requires $\sum_{k=1}^{H}(d-k+1)=\O(Hd)$ support-value oracle calls. Since the support
identification phase terminates after $C_0^G$ exploration cycles, the total number of such calls is $\O(HdC_0^G)$. After identifying the greedy support, the algorithm computes one restricted optimal action per exploitation cycle. Since the number of cycles is $\O(\sqrt T)$, this phase requires only $\O(\sqrt T)$ additional optimization-oracle calls. Consequently, the total number of restricted optimization-oracle calls is $\O(HdC_0^G+\sqrt T)$. 
In particular, for action sets admitting efficient closed-form or oracle-based linear optimization, APSEE-G avoids the repeated posterior sampling required by randomized methods such as AG-FGTS. We emphasize that this is an oracle-complexity bound; the arithmetic cost of each oracle call depends on the representation of the action set. For example, for $\ell_p$-balls with $p>1$, the support-value function admits the closed-form
expression $h(S;\theta)=\|\theta(S)\|_q$, where $q=p/(p-1)$, allowing each marginal gain to be evaluated efficiently.

\section{Feel-Good Thompson Sampling for General Compact Sets}
\label{sec:ag_ts}
The analysis in the previous section relies on two distinct regularity properties: the smoothness of the action-set boundary, which guarantees a positive submodularity ratio, and the Lipschitz continuity of the restricted optimal-action mapping, which controls the exploitation regret of the deterministic plug-in strategy. For general compact action sets, neither property is guaranteed. In particular, the optimal action may change discontinuously even under arbitrarily small perturbations of the estimated parameter, making the deterministic exploitation analysis in APSEE-G inapplicable. This motivates us to extend our framework to more general action sets by developing an alternative exploitation mechanism that does not rely on the regularity of the optimizer mapping. It is worth noting that the two approaches serve complementary purposes: when the geometric and optimizer-regularity conditions required by APSEE-G are satisfied, its deterministic optimization strategy is preferable due to its lower computational cost, as it avoids repeated posterior sampling. For more general action sets where these conditions are unavailable, a randomized approach provides greater flexibility, although it incurs additional sampling costs. In the preliminary conference version of this work~\cite{ijcai2026p560}, we handled this difficulty using a pre-specified phased exploration--exploitation schedule. Balancing the exploration cost against the accumulated exploitation error led to a weaker $\tilde{\mathcal O}(dT^{2/3})$ $\alpha$-regret bound. Therefore, to improve the dependence on both the horizon $T$ and the ambient dimension $d$, we seek an adaptive exploration--exploitation mechanism whose regret analysis does not require the optimizer $\x^\star(S;\theta)$ itself to vary smoothly with $\theta$.

This observation motivates us to replace the deterministic plug-in exploitation used in Algorithm~\ref{alg:adaptive_general_PEE} with a randomized posterior-sampling mechanism, such as Thompson Sampling (TS), whose regret can be controlled through the value structure of the problem rather than the continuity of the optimizer mapping. In the broader context of stochastic linear bandits, randomized algorithms like TS are favored for their empirical performance and their ability to handle exploration--exploitation trade-offs without requiring explicit upper confidence bounds~\cite{agrawal2013thompson,lattimore2020bandit}. 
While standard linear TS is widely used, it leads to suboptimal dependencies on the dimension \cite{abeille2017linear}. Moreover, \cite{hamidi2023frequentist} showed that standard LinTS without sufficient posterior variance inflation can incur linear regret under prior or noise misspecification, further highlighting the limitations of conventional posterior sampling. Standard TS also does not naturally accommodate the combinatorial sparsity constraint $|S_t|\le H$, nor can it efficiently solve the NP-hard subset selection problem at each time step $t$.

To address these challenges, we propose the Adaptive Greedy Feel-Good Thompson Sampling \textbf{(AG-FGTS)} algorithm. It leverages the \textit{Adaptive Support Recovery via Greedy Gap} mechanism introduced in Section~\ref{subsec:Adaptive Support Recovery via Greedy Gap} to identify the greedy support $S^G$, and then transitions to a restricted version of Feel-Good Thompson Sampling~\cite{zhang2022feel} on the identified subspace, establishing a $\tilde{\mathcal O}(H\sqrt{T})$ high-probability regret bound relative to the greedy-support benchmark. Furthermore, under the Smooth Convex Inner-Set Regularity condition introduced below, this guarantee can be converted into a $\tilde{\mathcal O}(H\sqrt{T})$ $\alpha_{\mathrm{SC}}$-regret bound relative to the globally optimal sparse action. Thus, AG-FGTS extends our framework to a broader class of action sets without requiring Lipschitz continuity of the restricted optimal-action mapping, while retaining the improved dependence on $H$ and $T$.

Consider the general setting where the action set $\X \subset \R^d$ is a compact set. We define the sparse search space as $\X'=\bigcup_{S\subseteq[d],\,|S|\le H}\X(S)$, where $\X(S)=\{x\in\X:\supp(x)\subseteq S\}$ is the feasible action set restricted to a specific support $S$.
The global optimal value function is defined as:
\begin{equation}
    J(\theta)=\sup_{x\in\X'}x^\top\theta. \label{eqn:global optimal value function}
\end{equation}
Recall the definition of $\alpha$-regret in \eqref{eqn:def of a-regret}:        
\begin{align}
\alpha\text{-}R_T&=\alpha\left(\sum_{t=1}^{T}\theta_*^\top\x^\star(S^\star)\right)-\left(
\sum_{t=1}^{T}\theta_*^\top\x_t(S_t)\right)\nonumber\\
&=\alpha\sum_{t=1}^{T}J(\theta_*)-\left(\sum_{t=1}^{T}\theta_*^\top\x_t(S_t)\right).
\label{eqn:def of a-regret2}
\end{align}
In addition, we define the restricted value function for a fixed support $S$ as:
\begin{equation}
    J(\theta;S)=\sup_{x\in\X(S)}x^\top\theta. \label{eqn:def of restricted J}
\end{equation}

\begin{assumption}[Smooth Convex Inner-Set Regularity]
\label{ass:smooth_convex_inner_set}
\begin{sloppypar}
For every support $S\subseteq[d]$ with $|S|\le H$, let $\x^\star(S)\in\argmax_{\x\in\X,\ \supp(\x)\subseteq S}\theta_*^\top\x$, and define $\X^\star\triangleq
\left\{\x^\star(S):S\subseteq[d],\ |S|\le H\right\}$. We assume that there exists a compact, full-dimensional convex set $\mathcal C\subseteq\X$ such that $\operatorname{conv}(\X^\star)\subseteq\mathcal C,\ \mathbf 0\in\operatorname{int}(\mathcal C)$, and the boundary $\partial\mathcal C$ is continuously differentiable.
\end{sloppypar}
\end{assumption}
\textbf{Role of Smooth Convex Inner-Set Regularity.}
Assumption~\ref{ass:smooth_convex_inner_set} introduces a smooth convex subset $\mathcal C$ of the original action set $\X$, which contains the convex hull of all
support-wise optimal actions under the true parameter $\theta_*$. For any support $S\subseteq[d]$ with $|S|\le H$, define $J_{\mathcal C}(\theta_*;S)\triangleq\max_{\x\in\mathcal C,\ \supp(\x)\subseteq S}\theta_*^\top\x$. Since $\mathcal C\subseteq\X$, we have $J_{\mathcal C}(\theta_*;S)\le J(\theta_*;S)$.
On the other hand, by Assumption~\ref{ass:smooth_convex_inner_set}, $\x^\star(S)\in\mathcal C$, and hence $J_{\mathcal C}(\theta_*;S)
\ge \theta_*^\top\x^\star(S)=J(\theta_*;S)$. Therefore, for every support $S$ with $|S|\le H$,
\begin{align}
J_{\mathcal C}(\theta_*;S)=J(\theta_*;S). \label{eqn:smooth_inner_set_value_equivalence}
\end{align}
Consequently, the support-wise optimal value function induced by the original action set $\X$ is exactly preserved when restricting the optimization to $\mathcal C$. Moreover, $\mathcal C$ satisfies all the geometric conditions required by Proposition~\ref{prop:submodularity_general}. Thus, for any nonzero $\theta_*$, the submodularity ratio $\gamma_{\mathrm{CH}}$ of the set function $S\mapsto J(\theta_*;S)$ is strictly positive: $\gamma_{\mathrm{CH}}>0$. In particular, the Greedy algorithm achieves the approximation ratio $\alpha_{\mathrm{CH}}\triangleq 1-e^{-\gamma_{\mathrm{CH}}}>0$.

Again, let $S^G$ and $S_c^G$ be the support sets obtained by running the Greedy algorithm on the true parameter $\theta_*$ and the estimate $\hat{\theta}_c$, respectively. For any support set $S\subseteq[d]$, let $P_S:\mathbb R^d\to\mathbb R^d$ denote the coordinate projection onto $S$, i.e.,
\[
(P_S\theta)_i= \begin{cases} \theta_i, & i\in S,\\ 0, & i\notin S.
\end{cases}
\]
We further define $\Omega(S)\triangleq\{\vartheta\in\mathbb R^d:\supp(\vartheta)\subseteq S\}$.
Conditioned on identifying the greedy support $S^G$, define the effective restricted true parameter $\vartheta_*\triangleq P_{S^G}\theta_* \in \Omega(S^G)$. 
Note that this definition reflects the fact that every action played during the Restricted FGTS phase is supported on $S^G$. Consequently, for every $\x\in\X(S^G)$, $\x^\top\theta_*=\x^\top\vartheta_*$, and therefore $J(\theta_*;S^G)=J(\vartheta_*;S^G)$.

To proceed, we define the value associated with the greedy support as $J_G\triangleq J(\theta_*;S^G)$. Correspondingly, the cumulative regret relative to the greedy support benchmark is defined by
\begin{equation}
R_T^G \triangleq T J(\theta_*;S^G)-\sum_{t=1}^{T}\theta_*^\top\x_t. \label{eqn:def_greedy_benchmark_regret}
\end{equation}
If Assumption~\ref{ass:smooth_convex_inner_set}
additionally holds, the set function $S\mapsto J(\theta_*;S)$ has submodularity ratio $\gamma_{\mathrm{CH}}\in(0,1]$ as we argued before. Therefore, applying Lemma~\ref{lemma:approximation ratio expression} gives $J(\theta_*;S^G)\ge\left(1-e^{-\gamma_{\mathrm{CH}}}\right) J(\theta_*)$. We define
\begin{equation}
\alpha_{\mathrm{CH}} \triangleq 1-e^{-\gamma_{\mathrm{CH}}}. \label{eqn:alpha_CH}
\end{equation}
Using the effective restricted parameter defined above, the same approximation guarantee can also be written as
\begin{equation}
\label{eqn:J_vartheta_ge}
J(\vartheta_*;S^G) = J(\theta_*;S^G) \ge \alpha_{\mathrm{CH}} J(\theta_*).
\end{equation}
Consequently, one can show that
\begin{equation}
\alpha_{\mathrm{CH}}\text{-}R_T \le R_T^G. \label{eqn:alpha_regret_leq_greedy_regret}
\end{equation}

As is standard in the analysis of posterior-sampling algorithms for linear bandits \cite{zhang2022feel,abeille2017linear}, we impose a boundedness assumption on the unknown parameter $\theta_*$.
\begin{assumption}[Bound on unknown parameter]
\label{asm:param.set}
There exists a positive constant $S>0$ such that $\|\theta_*\|_2\le S$ and $S$ is known. Since $\vartheta_*=P_{S^G}\theta_*$ is an orthogonal coordinate projection of $\theta_*$, the former argument also implies $\norm{\vartheta_*}_2\le\norm{\theta_*}_2\le S$.
\end{assumption}

\subsection{AG-FGTS Algorithm}

The AG-FGTS algorithm, outlined in Algorithm~\ref{alg:ag_ts}, proceeds in two phases. Initially, it acts circularly as a pure explorer using the basis set from Assumption~\ref{ass:arms during exploration} to obtain the estimator $\hat{\theta}_c$ for $\theta_*$ and compute the \textit{Greedy Gap} (same as Algorithm~\ref{alg:adaptive_general_PEE}). Once the empirical gap $\Delta_c^{Greedy}$ is sufficiently large to guarantee support consistency (i.e., $\Delta_c^{Greedy}>2L_\X\varepsilon_c$), the algorithm locks the estimated support $S_{est}$ ($S_{est}=S^G$), and switches to the Restricted Feel-Good TS phase.

In the second phase, we employ a Feel-Good posterior supported on the identified subspace $\Omega(S_{est})$. Recall that $\Omega(S_{est})=\left\{\theta\in\mathbb R^d:\supp(\theta)\subseteq S_{est}\right\}$. Let $S_{est}=\{i_1,\dots,i_H\}$, where the indices are listed according to a fixed ordering. For any $\theta\in\Omega(S_{est})$, define its restricted-coordinate representation as $\theta^\prime(S_{est})=(\theta_{i_1},\dots,\theta_{i_H})^\top\in\mathbb R^H$. Conversely, any vector $u=(u_1,\dots,u_H)^\top\in\mathbb R^H$ can be embedded into $\mathbb R^d$ by assigning its $k$-th coordinate to index $i_k$ and setting all coordinates outside $S_{est}$ to zero. Therefore, $\Omega(S_{est})$ is linearly isomorphic to $\mathbb R^H$, which we denote by $\Omega(S_{est})\simeq\mathbb R^H$. Let $p_{0,S_{est}}$ denote a prior distribution on
$\Omega(S_{est})$. Under the above coordinate representation, we may equivalently specify an isotropic Gaussian prior $\mathcal N(0,\rho^{-1}I_H)$ on the $H$ active coordinates, where $\rho>0$ denotes the precision parameter of the Gaussian prior (i.e., its covariance matrix is $\rho^{-1}I_H$). The sampled $H$-dimensional vector is then zero-padded outside $S_{est}$ to obtain a parameter in $\Omega(S_{est})$. For a candidate parameter $\theta\in\Omega(S_{est})$ and an observation $(\x_s,Y_s)$, we define the Feel-Good loss as
\begin{equation}
\label{eqn:fgts_loss}
L(\theta,\x_s,Y_s)= \eta(\x_s^\top\theta-Y_s)^2-\lambda J(\theta;S_{est}),
\end{equation}
where $\eta>0$ controls the squared prediction-loss term and $\lambda>0$ controls the Feel-Good optimism bonus. Let $\mathcal F_t=\sigma(\x_\tau,Y_\tau,\dots,\x_t,Y_t)$ be the filtration generated by the history from time $\tau$ to time $t$. The Restricted Feel-Good posterior at time $t>\tau$ is given by
\begin{equation}
\label{eqn:fgts_posterior}
p_t(\theta)\triangleq p(\theta\mid\mathcal F_{t-1}) \propto p_{0,S_{est}}(\theta)\exp\big(-\sum_{s=\tau+1}^{t-1} L(\theta,\x_s,Y_s) \big),\quad \theta\in\Omega(S_{est}).
\end{equation}
The posterior in~\eqref{eqn:fgts_posterior} adopts the Feel-Good Thompson Sampling construction of \cite{zhang2022feel}, but is applied only after the greedy support has been identified. In contrast to the original FGTS formulation over the full model class, our posterior is restricted to the data-dependent subspace $\Omega(S_{est})\simeq\mathbb R^H$ and is initialized at the random stopping time $\tau$. Accordingly, both the squared-loss term and the Feel-Good bonus $J(\theta;S_{est})=\max_{\x\in\X(S_{est})}\x^\top\theta$ are evaluated only on the identified support. This restricted construction introduces two main challenges in the regret analysis. First, both $S_{est}$ and the switching time $\tau$ are
data-dependent, so we analyze Phase 2 conditioned on the support-consistency event $S_{est}=S^G$. Second, for general compact action sets, a small parameter error does not imply a small change in the maximizing action, so the deterministic optimizer-stability argument required by APSEE and APSEE-G may be unavailable. Instead, conditioned on correct support recovery, the restricted problem admits an $H$-dimensional linear representation, which enables the FGTS decoupling and log-partition analysis and yields a $\tilde{\mathcal O}(H\sqrt{T})$ horizon-dependent regret term.

{\bf Practical Sampling Schemes and Computational Cost.}
Exact sampling from the Feel-Good posterior $p_t(\theta)$ in~\eqref{eqn:fgts_posterior} is generally analytically intractable. Practical sampling implementations can therefore employ
approximate Markov Chain Monte Carlo (MCMC) methods. For example, Zhang~\cite{zhang2022feel} implemented Feel-Good Thompson Sampling using Stochastic Gradient Langevin Dynamics (SGLD) with $\mathcal O(T^2)$ computational cost.  
Huix et al.~\cite{huix2023tight} further established theoretical computational complexity bounds $\widetilde{\mathcal O}(dT^4)$ and $\widetilde{\mathcal O}(dT^2)$ for smoothed versions of Feel-Good Thompson Sampling using Langevin Monte Carlo (LMC) and the Metropolis-Adjusted Langevin Algorithm (MALA), respectively. 
In our setting, posterior sampling is performed in the identified $H$-dimensional subspace and the aforementioned practical sampling scheme may be applied accordingly. 
Therefore, AG-FGTS generally incurs higher computational costs than APSEE-G due to the additional posterior sampling required at each exploitation round. However, this additional cost allows AG-FGTS to handle a broader class of compact action sets without relying on the regularity of the restricted optimal-action mapping. Conversely, when the action set is smooth and convex and satisfies Assumption~\ref{ass:smooth_restricted_optimizer}, APSEE-G is preferable from a computational perspective, as its deterministic exploitation strategy avoids repeated posterior sampling while achieving the same $\widetilde{\mathcal O}(H\sqrt{T})$ leading regret dependence.

\begin{algorithm}[!ht]
\small
\caption{Adaptive Greedy Feel-Good TS (AG-FGTS)}
\label{alg:ag_ts}
\SetKwComment{tcp}{}{}
\SetKwInOut{Input}{Input}
\SetKwInOut{Requires}{Requires}
\Input{Dimension $d$, sparsity $H$, confidence parameter $\delta$.}
\Requires{ $\eta>0$, $\lambda>0$, $\rho>0$, an $H$-dimensional base prior $p_0^{(H)}$.
}
\BlankLine
\textbf{Initialize Constants}:
$B=\sum_{k=1}^d b_k(S_k)b_k(S_k)^\top$ by Assumption~\ref{ass:arms during exploration}(i), $\lambda_0=\lambda_{\min}(B)$, $h_1=\frac{2\sigma^2\mathrm{Tr}(B)}{\lambda_0^2}$.\\ 
\textbf{Initialize Variables}:
Cycle index $c=1$, Time step $t=1$. Stop Time $\tau=\infty$.\\
\textbf{Flag}:
$\texttt{SupportFound}=\text{False}$, $S_{est}=\emptyset$.\\

\For{$t=1,2,\dots$}{
    \If{$\texttt{SupportFound}$ is \text{False}}{
        \tcp*[h]{\textcolor{blue}{\textit{--- Phase 1: Adaptive Exploration ---}}}\vspace{2pt}\\
        Run lines 3-14 of Algorithm~\ref{alg:adaptive_general_PEE}.\\
        \If{$\Delta_c^{Greedy}>2L_{\X}\varepsilon_c$}{
            Set $S_{est}=G_H$, instantiate the restricted prior $p_{0,S_{est}}$ on $\Omega(S_{est})$
            from the base prior $p_0^{(H)}$.\\
            $\texttt{SupportFound}\leftarrow\text{True}$.\\
            Reset variables for \textbf{FGTS}:
            $\tau=t$, History $\mathcal H_\tau=\emptyset$.\\
        }
        \Else{
            $c \gets c+1$.\\
        }
    }

    \If{$\texttt{SupportFound}$ is \text{True}}{
        \tcp*[h]{\textcolor{blue}{\textit{--- Phase 2: Restricted Feel-Good Thompson Sampling ---}}}\vspace{2pt}\\
        Form the Restricted Feel-Good posterior $p_t(\theta)$ on
        $\Omega(S_{est})$ according to~\eqref{eqn:fgts_posterior}.\\
        Sample optimistic parameter $\tilde{\theta}_t\sim p_t(\theta)$.\\
        Select $\x_t=\argmax_{x\in\X(S_{est})} x^\top\tilde{\theta}_t$.\\
        Observe reward $Y_t$.\\
        Update history $\mathcal H_t =\mathcal H_{t-1}\cup\{(\x_t,Y_t)\}$.\\
        $t\leftarrow t+1$.\\
    }
}
\end{algorithm}

\subsection{Regret Analysis}

We first analyze the regret relative to the greedy support benchmark defined in~\eqref{eqn:def_greedy_benchmark_regret}. It can be decomposed into the warm-up (pure exploration) phase and the Restricted FGTS phase:
\begin{equation}
R_T^G=\underbrace{\sum_{t=1}^{\tau}\left(J(\theta_*;S^G)-\x_t^\top\theta_*\right)}_{R_{Explore}^G}+\underbrace{\sum_{t=\tau+1}^{T}\left(J(\theta_*;S^G)-\x_t^\top\theta_*\right)}_{R_{FGTS}^G}. \label{eqn:reg TS decompose}
\end{equation}
The analysis of Phase 1 relies on identifying the true greedy support. Let $\mathcal E_{consist}$ be the event that the stopping condition based on the Greedy Gap successfully identifies the true greedy support, i.e., $S_{est}=S^G$. Invoking the \textit{Support Consistency Lemma} (Lemma~\ref{lemma:general_consistency}) and the stopping threshold $C_0^G$ derived in the proof of Theorem~\ref{thm:regret_adaptive} (Eq.~\eqref{eqn:C_0}) in Appendix~\ref{app:Omitted Proofs in sec:euclidean_case}, we have the following guarantee.

\begin{corollary}
\label{coro:agts_consistency}
Assume the true parameter $\theta_*$ satisfies the Minimum Greedy Gap \eqref{eqn:delta_min} condition with $\Delta_{\min}^G>0$. Let $\delta\in(0,1)$. With probability at least $1-\delta$, the stopping time $\tau=\tilde{\CO}\left(d(\Delta_{\min}^G)^{-2}\right)$ in Algorithm~\ref{alg:ag_ts} is finite and the identified support is
consistent, i.e., $S_{est}=S^G$. 
\end{corollary}
Corollary~\ref{coro:agts_consistency} can be proved by noticing $\tau=dC_0^G$ and following the previous derivations (the proofs of Lemma~\ref{lemma:general_consistency} and Eq.~\eqref{eqn:C_0}). Since $\tau$ is bounded by a constant independent of the time horizon $T$ for every fixed instance satisfying $\Delta_{\min}^G>0$, the regret relative to the greedy benchmark accumulated during Phase 1 is given by:
\begin{equation}
R_{Explore}^G\le\tau\cdot 2\sup_{x\in\X}|x^\top\theta_*|=\tilde{\mathcal O}\left(\frac{dL_{\max}
\norm{\theta_*}_2h_1}{(\Delta_{\min}^G)^2}\right). \label{eqn:TS pure exploration a-reg}
\end{equation}
Conditioned on the support-consistency event $S_{est}=S^G$, every Phase-2 action belongs to $\X(S^G)$. Hence, for any $\x\in\X(S^G)$, $\x^\top\theta_*=\x^\top\vartheta_*,\ \vartheta_*\triangleq P_{S^G}\theta_*\in\Omega(S^G)$. Thus, only the restriction of $\theta_*$ to the identified support affects the rewards collected in Phase 2. After support identification, the original sparse linear bandit therefore reduces to a linear bandit with action set $\X(S^G)$ and effective parameter $\vartheta_*$. Using the coordinate representation $\Omega(S^G)\simeq\mathbb R^H$, this restricted problem is equivalent to an $H$-dimensional stochastic linear bandit. This reduction allows us to apply the Feel-Good Thompson Sampling analysis directly in the identified $H$-dimensional subspace.

\begin{theorem}[Regret for AG-FGTS]
\label{thm:ag_ts_regret}
Suppose Assumptions~\ref{ass:feasible set bound}, \ref{ass:noise}, \ref{ass:arms during exploration}, and \ref{asm:param.set} hold. Assume the Minimum Greedy Gap defined in \eqref{eqn:delta_min} satisfies $\Delta_{\min}^G>0$. Let $\delta\in(0,1)$. Set $\eta=0.25/\sigma^2,\ \lambda=\Theta(\sqrt{\frac{H\log(HT)+\rho S^2}{(1+H/\eta)}})$, 
where $\rho$ is the precision parameter of the restricted Gaussian prior $p_{0,S_{est}}$. Then, with probability at least $1-\delta$, the cumulative regret of AG-FGTS relative to the greedy support benchmark satisfies
\begin{equation}
R_T^G\le C_{gap}^G +\tilde{\mathcal O}\left(H\sqrt T\right), \label{eqn:AGFGTS_greedy_regret}
\end{equation}
where $C_{gap}^G=\tilde{\mathcal O}\left(\frac{dL_{\max}\norm{\theta_*}_2h_1}{(\Delta_{\min}^G)^2}\right)$ is the finite instance-dependent cost of adaptive greedy support
recovery in~\eqref{eqn:TS pure exploration a-reg}. If, in addition, Assumption~\ref{ass:smooth_convex_inner_set} holds, then, with $\alpha_{\mathrm{CH}}$ defined in~\eqref{eqn:alpha_CH}, the same event also guarantees
\begin{equation}
\alpha_{\mathrm{CH}}\text{-}R_T\le C_{gap}^G+\tilde{\mathcal O}\left(H\sqrt T \right).
\label{eqn:AGFGTS_alpha_regret}
\end{equation}
\end{theorem}

\noindent{\it Proof Sketch.}
The greedy-benchmark regret $R_T^G$ is decomposed according to \eqref{eqn:reg TS decompose}. (i) $R_{Explore}^G$ is bounded by the stopping time $\tau$ multiple by the maximum reward difference. By Corollary~\ref{coro:agts_consistency}, $\tau$ scales with $(\Delta_{\min}^G)^{-2}$, yielding the horizon-independent instance-dependent term $C_{gap}^G$ in \eqref{eqn:TS pure exploration a-reg}. (ii) For $R_{FGTS}^G$, conditioned on $S_{est}=S^G$, every Phase-2 action is supported on $S^G$. Therefore, $\x_t^\top\theta_*=\x_t^\top\vartheta_*$, where $\vartheta_*=P_{S^G}\theta_*$. Thus the Phase-2 reward model admits an $H$-dimensional linear representation. The decoupling coefficient is consequently bounded by $H$, and the restricted Feel-Good posterior yields a $\tilde{\mathcal O}(H\sqrt T)$ leading horizon-dependent term. A conditional expected bound is converted into a high-probability bound using the non-negativity of the greedy-benchmark Phase-2 regret and Markov's inequality. Combining the two phases gives \eqref{eqn:AGFGTS_greedy_regret}. If Assumption~\ref{ass:smooth_convex_inner_set} holds, then \eqref{eqn:alpha_regret_leq_greedy_regret} converts this greedy-benchmark guarantee into \eqref{eqn:AGFGTS_alpha_regret}. The detailed proof is provided in Appendix~\ref{app:Omitted Proofs in Section ag_ts}.

\section{Lower Bound}\label{sec:Lower Bound}
We derive the minimax lower bound on the $1$-regret of sparsifying stochastic linear bandits given by \eqref{eqn:def of a-regret}. 
\begin{theorem} \label{thm:lower_bound_randomized}
Let $\X_H=\left\{\x\in\mathbb R^d:\norm{\x}_2\le 1,\ \norm{\x}_0\le H\right\}$, where $H\in\{1,\dots,d\}$, and let $T\ge d$. Consider the sparsifying stochastic linear bandit problem on $\X_H$ with standard Gaussian reward noise $\eta_t\overset{\mathrm{i.i.d.}}{\sim}\mathcal N(0,1)$. Then the minimax regret satisfies
\begin{equation}
    \inf_{\pi} \sup_{\theta\in\mathbb R^d} \mathbb E_{\theta}^{\pi} \left[R_T\right]
    \ge\frac{H\sqrt{T}}{16\sqrt{e}} = \Omega(H\sqrt{T}).
\end{equation}
where $\pi$ ranges over all possibly randomized non-anticipating policies, and $\mathbb E_{\theta}^{\pi}$ denotes expectation under the joint distribution of the action--reward trajectory induced by $\pi$ and the Gaussian reward noise when the environment parameter is fixed at~$\theta$.
\end{theorem}
\begin{sloppypar}
\noindent{\it Proof Sketch.}
The proof applies Yao's minimax principle with a hard prior supported on $2^H$ parameters
$\theta^\sigma=\Delta_H(\sigma_1,\ldots,\sigma_H,0,\ldots,0)$, where $\sigma_i\in\{-1,+1\}$. For each instance, the optimal sparse action aligns with the corresponding sign vector. A geometric inequality converts instantaneous regret into the number of incorrectly inferred coordinate signs. We then pair environments that differ in one sign and apply the Bretagnolle--Huber inequality together with the Gaussian KL-divergence decomposition. Averaging the pairwise divergences over all coordinates yields an information bound of order $T\Delta_H^2/H$. Finally, choosing $\Delta_H=\sqrt{H/(4T)}$ balances distinguishability and regret and gives $\inf_{\pi}\sup_{\theta\in\mathbb R^d}\mathbb E_\theta^\pi[R_T]\ge\frac{H\sqrt T}{16\sqrt e}$. The detailed proof is provided in Appendix~\ref{app:Lower Bound}.
\end{sloppypar}

Theorem~\ref{thm:lower_bound_randomized} shows that an $H\sqrt{T}$ dependence is unavoidable in the worst case, which matches with the $H\sqrt{T}$-scaling in the regret upper bounds for APSEE, APSEE-G and AG-FGTS. While the regret upper bound of APSEE and its variants contain instance-dependent support-recovery terms $C_{\mathrm{gap}}>0$, these terms are independent of $T$ for every fixed instance. However, these terms can become arbitrarily large as the relevant gap tends to zero and hence are not uniformly controlled in the minimax sense. Obtaining a gap-independent or gap-adaptive bound that remains uniform over the full parameter class is an important open problem.

For the Euclidean-ball setting, the sparse optimization problem can be solved exactly, and APSEE is therefore evaluated under $1$-regret. When the instance-dependent quantities are treated as fixed, the leading horizon-dependent term of its regret upper bound has the same $H\sqrt{T}$ dependence as the lower bound in Theorem~\ref{thm:lower_bound_randomized}. For more general action sets, however, exact sparse optimization is generally computationally intractable, and our guarantees are consequently stated in terms of $\alpha$-regret, as
is common in online learning problems with combinatorial structure \cite{streeter2008online}. Since standard minimax lower bounds are formulated for $1$-regret \cite{lattimore2020bandit}, these results cannot yet be directly compared in a minimax sense. Closing this gap requires either improving the achievable approximation factor $\alpha$ or developing lower bounds directly for $\alpha$-regret, which remains largely unexplored apart from a few results in different structured bandit settings such as
\cite{glasgow2023tight}. Nevertheless, Theorem~\ref{thm:lower_bound_randomized} shows that even on the canonical sparse Euclidean-ball subclass, any learning algorithm must incur an $\Omega(H\sqrt{T})$ expected regret in the worst case. Hence, the $H\sqrt{T}$ dependence reflects an intrinsic statistical difficulty of the sparse-action bandit problem, rather than an artifact of a particular algorithmic design or regret analysis. In this sense,
Theorem~\ref{thm:lower_bound_randomized} provides a sharp benchmark for the dependence on $H$ and $T$, rather than a complete minimax characterization of all instance-dependent terms.

\section{Numerical Experiments}
\label{sec:numerical experiments}

In this section, we empirically validate the performance and core mechanisms of our proposed algorithms. We focus our primary evaluation on the ellipsoid action set, which belongs to the smooth convex action-set class discussed in Section~\ref{sec:smooth_convex_case} and satisfies the restricted-optimizer regularity condition required by APSEE-G. We refer the readers to Appendix~\ref{app:additional_experiments} for additional experiments on Euclidean balls, general compact sets, control experiments compared with the conference version, and a real-world recommendation system application.

\textbf{Regret Performance on Ellipsoid Action Sets.} We fix the action dimension to $d=20$ and the sparsity constraint to $H=5$. The true parameter $\theta_*$ is drawn uniformly from the hypercube $[-1, 1]^d$. We simulate the environment for 2,000 time steps, where the noise $\eta_t$ in the reward is i.i.d. Gaussian noise drawn from $\mathcal{N}(0, 0.5^2)$. The results are based on 20 independent trials. We construct an ellipsoid action set $\X = \{\x \in \R^d : \x^\top A \x \le 1\}$, where $A$ is a randomly generated positive definite matrix with $\lambda_\text{max}(A)\le 1$. We employ the \textbf{APSEE-G} algorithm (Algorithm~\ref{alg:adaptive_general_PEE}) and validate the results established in Theorem~\ref{thm:regret_general}. Figure~\ref{fig:case4} (Left) illustrates the cumulative $\alpha$-regret of Algorithm~\ref{alg:adaptive_general_PEE}. The scatter cloud highlights the variance across different runs, while the mean trajectory demonstrates a sub-linear growth, indicating that the algorithm successfully learns the optimal sparse action. To verify the regret bounds provided in Theorem~\ref{thm:regret_general}, Figure~\ref{fig:case4} (Right) plots the $\alpha$-regret normalized by $\sqrt{T}$. The curve of the mean normalized regret stabilizes and exhibits a slow, logarithmic growth trend, confirming the $\tilde{\mathcal{O}}(\sqrt{T})$ bound established in Theorem~\ref{thm:regret_general}.

\begin{figure}[!htbp]
    \captionsetup{justification=centering}
    \centering
    \includegraphics[width=0.8\linewidth]{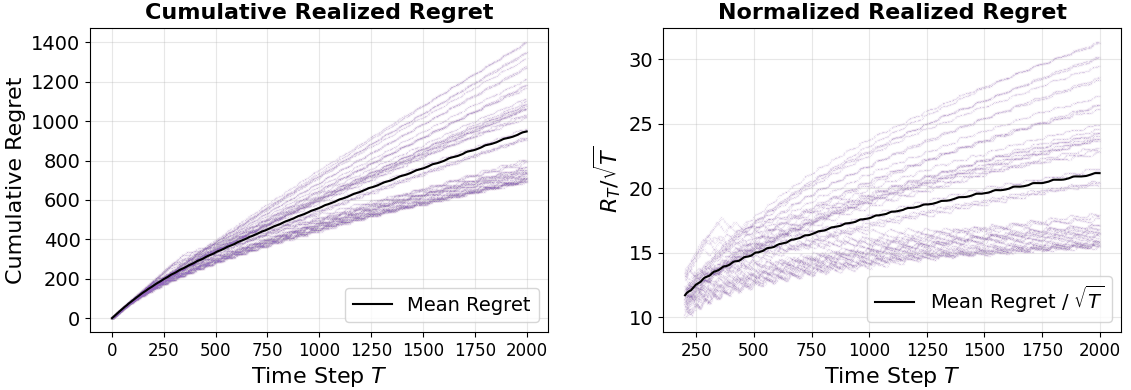}
    \caption{Numerical results for the APSEE-G algorithm on the ellipsoid action set.}
    \label{fig:case4}
\end{figure}

\textbf{Impact of Basis Actions on Exploration.} Next, we examine the sensitivity of the APSEE-G algorithm (Algorithm~\ref{alg:adaptive_general_PEE}) to the choice of basis actions utilized during the exploration phase. 
We conduct experiments using the following three distinct sets of basis actions under the same ellipsoid action set: (i) the standard basis vectors in $\mathbb{R}^d$ (consistent with the main evaluation), (ii) a random orthogonal basis generated from a standard Gaussian distribution, and (iii) a random orthogonal basis sampled from a uniform distribution. All these sets of basis actions satisfy the invertibility condition in Assumption~\ref{ass:arms during exploration}.

As illustrated in Figure~\ref{fig:app_basis_actions}, the cumulative $\alpha$-regret trajectories of APSEE-G under these different basis sets are nearly indistinguishable, all exhibiting matching sub-linear growth rates. This empirical evidence corroborates the theoretical claims derived from Lemmas~\ref{lemma:stopping_condition_l2ball} and \ref{lemma:general_consistency}: provided that the chosen basis actions span the necessary space to satisfy Assumption~\ref{ass:arms during exploration}, the exact distribution of the basis actions in the exploration phase has a negligible impact on the overall regret of the algorithm. 

\begin{figure}[!htbp]
    \centering
    \includegraphics[width=0.5\linewidth]{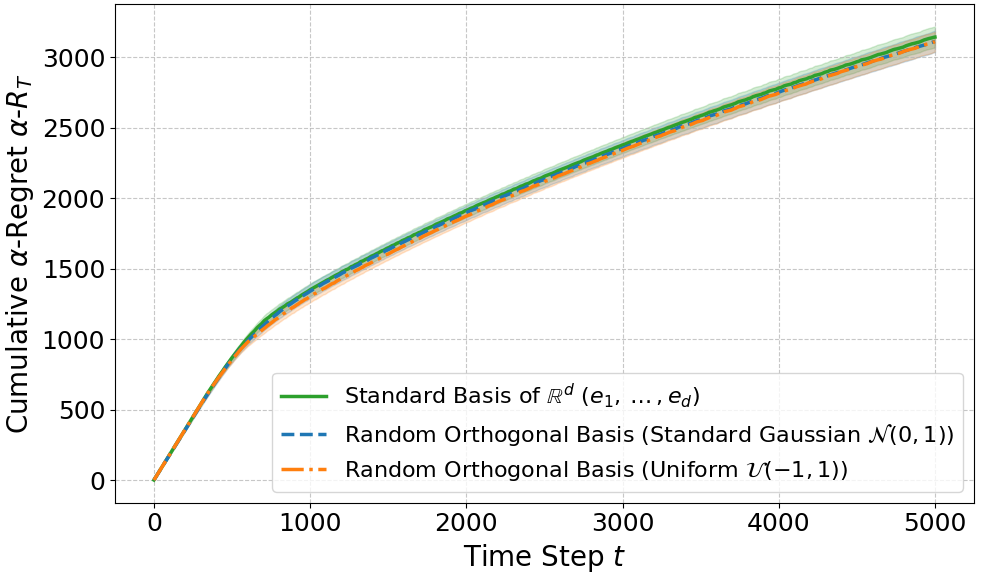} 
    \caption{Cumulative $\alpha$-regret of APSEE-G algorithm under different choices of basis actions in the exploration phase. }
    \label{fig:app_basis_actions}
\end{figure}
\begin{figure}[!htbp]
    \centering
    \includegraphics[width=0.5\linewidth]{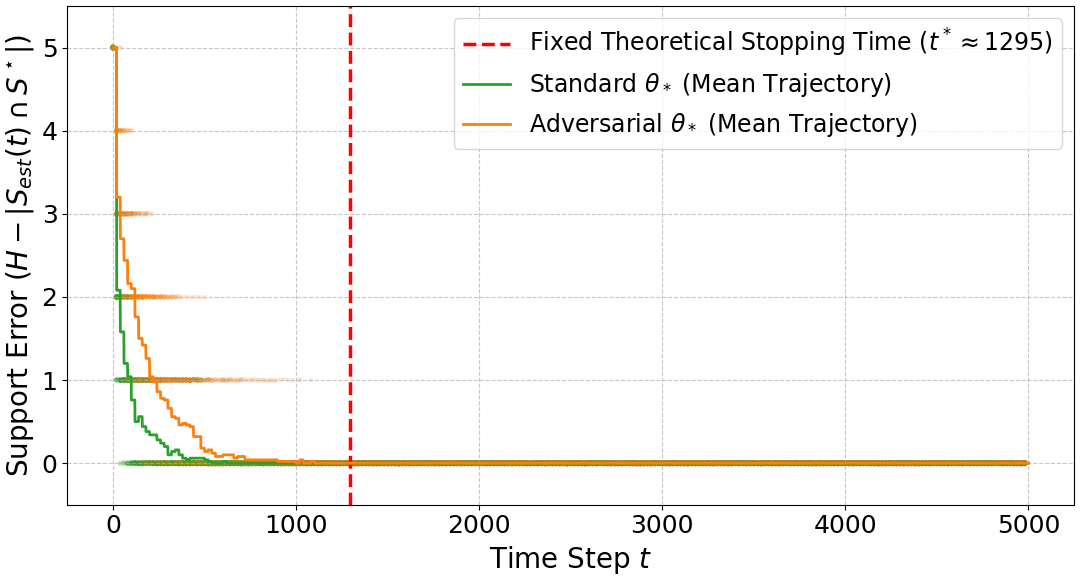} 
    \caption{Support recovery performance of the APSEE algorithm over time under standard and adversarial $\theta_{*}$ vectors with the same signal gap $\Delta_{\min}$. }
    \label{fig:app_support_recovery}
\end{figure}
\color{black}
\textbf{Empirical Verification of Support Recovery.} While the recovery of the true support $S^{\star}$ is theoretically justified in Lemmas~\ref{lemma:stopping_condition_l2ball} and \ref{lemma:general_consistency}, we show that this guarantee holds in practical evaluations. To this end, we conduct additional experiments to track the real-time support recovery process of the APSEE algorithm (Algorithm~\ref{alg:adaptive_PEE}) under two distinct configurations of the ground-truth parameter vector $\theta_{*}\in\mathbb{R}^{20}$: a ``standard'' $\theta_{*}$ (where the non-zero elements have more separated absolute values, making them easily identifiable) and an ``adversarial'' $\theta_{*}$ (where the absolute values of the non-zero elements are nearly identical, making them highly difficult to distinguish). Both parameter vectors are constructed to share the exact same signal gap $\Delta_{\min}$ (defined as the absolute difference between the $H$-th and $(H+1)$-th largest elements of $\theta_{*}$). According to our theoretical proofs, sharing an identical $\Delta_{\min}$ ensures that the theoretical worst-case stopping time for the exploration phase, $t^{*}=\mathcal{O}(1/\Delta_{\min}^2)$, is exactly the same for both cases.

Figure~\ref{fig:app_support_recovery} plots the evolution of the support recovery over time. For the above two distinct cases, we conduct 50 trials each and plot scatter plots as well as average curves. The results clearly demonstrate that despite the differing difficulty landscapes of the standard and adversarial parameter vectors, the APSEE algorithm successfully identifies the true support $S^{\star}$ within the theoretical stopping time $t^{*}$.

\section{Conclusion}
We proposed a new framework of stochastic linear bandits called learning to sparsify stochastic linear bandits, where the learner must jointly identify a sparse support and optimize the corresponding continuous action under an unknown linear reward model. We developed an adaptive algorithmic framework that separates sparse support identification from subsequent learning on the identified low-dimensional subspace. For Euclidean-ball action sets, the APSEE algorithm achieves a $\tilde{\mathcal O}(H\sqrt{T})$ horizon-dependent regret bound. For smooth convex action sets, the APSEE-G algorithm exploits the positive submodularity ratio guaranteed by the smoothness boundary condition and, under an additional Lipschitz regularity condition on the restricted optimal-action mapping, attains a $\tilde{\mathcal O}(H\sqrt{T})$ $\alpha$-regret bound via greedy support selection. For general compact action sets, the AG-FGTS algorithm leverages adaptive greedy support recovery and a restricted Feel-Good Thompson Sampling scheme, yielding a $\tilde{\mathcal O}(H\sqrt{T})$ regret bound relative to the greedy-support benchmark and, under the Smooth Convex Inner-Set Regularity condition, a corresponding $\alpha_{\mathrm{SC}}$-regret guarantee relative to the globally optimal sparse action. We further established an $\Omega(H\sqrt{T})$ minimax $1$-regret lower bound, identifying $H\sqrt{T}$ as a fundamental benchmark for the regret dependence on sparsity and horizon. Future work includes developing gap-independent regret guarantees, extending lower bounds to broader action-set geometries, further characterizing the approximation factor in $\alpha$-regret, and studying time-varying action sets or sparse supports.

\section*{Generative AI Use}
Generative AI tools were used to help provide ingredients for proving mathematical claims; assist with the writing of proofs; and coding pertaining to the experimental results. Generative AI tools were also used to search related literature; polish the language; and proofread the manuscript. 
All AI-assisted content was carefully reviewed and verified by the authors. We take responsibility for the final content of this work.

\bibliographystyle{unsrt}
\bibliography{main}

\appendix
\counterwithin{lemma}{section}
\counterwithin{theorem}{section}
\counterwithin{proposition}{section}
\counterwithin{corollary}{section}
\counterwithin{definition}{section}
\counterwithin{equation}{section}
\onecolumn

{\bf Organization of the Appendix}. More related work can be found in Appendix~\ref{app:related work}. Appendices~\ref{app:Omitted Proofs in sec:euclidean_case},~\ref{app:proofs_smooth_convex_case},~\ref{app:Omitted Proofs in Section ag_ts}, and~\ref{app:Lower Bound} supplement the omitted proofs of Sections~\ref{sec:euclidean_case},~\ref{sec:smooth_convex_case},~\ref{sec:ag_ts}, and~\ref{sec:Lower Bound}, respectively. Appendix~\ref{app:time_varying_action_set} provides discussions on time-varying action sets. More numerical simulation experiments are provided in Appendix~\ref{app:additional_experiments}. Technical lemmas are provided in Appendix~\ref{app:tech lemmas}.

\section{More Related Work}
\label{app:related work}

Building upon Section~\ref{sec:introduction}, we position our \textit{Learning to Sparsify Linear Bandits} problem within four related research directions. Our setting combines discrete support selection with continuous action optimization under stochastic bandit feedback. Similar hybrid optimization problems have been studied in offline \cite{adibi2022minimax,bunton2022joint} and adversarial online \cite{ye2025online} settings, but these results do not directly provide regret guarantees for our stochastic sparse-action formulation. We highlight the key distinctions below.

\subsection{Linear Bandits with Sparse Parameters}

Sparse linear bandits assume that the unknown parameter satisfies $\|\theta_*\|_0\le s\ll d$, aiming to exploit parameter sparsity to improve regret dependence on the ambient dimension. Existing approaches include confidence-set methods \cite{abbasi2012online}, compressed sensing and Lasso-based estimation \cite{carpentier2012bandit,bastani2021mostly}, minimax regret analyses \cite{liu2025minimax,hao2020high}, and adaptive identification of unknown sparsity structures \cite{kim2019doubly,oh2021sparsity}. Our formulation differs fundamentally in the \emph{location of sparsity}: rather than assuming a sparse environment parameter, we impose the constraint $\|\mathbf{x}_t\|_0\le H$ on the learner's actions. The resulting challenge is to jointly optimize discrete supports and continuous action values while controlling regret, even when $\theta_*$ is dense.

\subsection{Combinatorial Multi-Armed Bandits (CMAB)}

Combinatorial bandits study the selection of subsets of base arms, often under cardinality or matroid constraints \cite{chen2013combinatorial,chen2016combinatorial}. Approximation-oracle frameworks enable regret guarantees for computationally difficult objectives \cite{chen2016combinatorial}, with extensions to broader stochastic submodular settings \cite{fourati2024combinatorial,nie2023framework}. Unlike standard CMAB, our learner must not only select a support $S\subseteq[d]$ but also optimize continuous coefficients on that support. This hybrid decision structure couples combinatorial selection with the geometry of the feasible action set, requiring continuous optimization and approximation guarantees beyond conventional subset-based bandit formulations.

\subsection{Online Combinatorial Optimization with Continuous Variables}

Our work differs from \cite{ye2025online} in three key aspects: (i) We focus on the stochastic setting with a fixed $\theta_*$, allowing us to use OLS estimation rather than adversarial online gradient descent. This enables us to achieve a tighter $\tilde{\mathcal{O}}(\sqrt{T})$ regret compared to the $T^{5/6}$ or $T^{2/3}$ rates often seen in adversarial combinatorial semi-bandits with partial feedback. (ii) We adopt an Explore-Then-Commit (Phased) architecture. This separates the learning phase (estimating $\theta_*$) from the optimization phase, simplifying the combinatorial challenge to a pure offline optimization problem given the estimates. (iii) We provide specific analyses for $\ell_2$-balls (tractable) and smooth convex action sets (greedy approximation), establishing exact regret guarantees for Euclidean balls and approximation-based $\alpha$-regret guarantees for smooth convex sets under the required restricted-optimizer regularity condition.

\subsection{Subset Selection and Submodularity in Regression}

The offline problem $\max_{\mathbf{x}\in\X,\|\mathbf{x}\|_0\le H}\theta^\top\mathbf{x}$ is closely related to best subset selection and sparse constrained optimization, whose computational difficulty depends on the feasible-set geometry. Related sparse optimization problems are known to be NP-hard \cite{natarajan1995sparse}. To analyze efficient approximations, \cite{das2011submodular,das2018approximate} developed submodularity-ratio techniques for greedy selection under objectives that need not be exactly submodular. We leverage this connection by establishing geometry-dependent approximation guarantees for our greedy sparse optimization oracle through Proposition~\ref{prop:submodularity_general}. Integrating these guarantees into stochastic bandit learning enables computationally tractable exploitation with provable $\alpha$-regret bounds, linking offline approximation theory with online statistical learning.

\section{Omitted Proofs in Sections~\ref{sec:Problem Formulation} and~\ref{sec:euclidean_case}}\label{app:Omitted Proofs in sec:euclidean_case}

\subsection{Proof of Proposition~\ref{prop:lipschitz_h}}
\begin{proof}
Let $S\subseteq[d]$ be a fixed support set and define the restricted action set $\X(S) = \{\x\in\X:\supp(\x)\subseteq S\}$. The function $h(S;\theta)=\sup_{\x\in\X(S)}\theta^\top\x$ is the support function of $\X(S)$. Since $\X$ is compact, $\X(S)$ is also compact.

We first establish convexity. For any $\theta_1,\theta_2\in\R^d$ and $\lambda\in[0,1]$,
\begin{align*}
h(S;\lambda\theta_1+(1-\lambda)\theta_2) &=\sup_{\x\in\X(S)}\left[\lambda\theta_1^\top\x+(1-\lambda)\theta_2^\top\x\right]\\
&\le\lambda\sup_{\x\in\X(S)}\theta_1^\top\x+(1-\lambda)\sup_{\x\in\X(S)}\theta_2^\top\x\\
&=\lambda h(S;\theta_1)+(1-\lambda)h(S;\theta_2),
\end{align*}
which proves that $h(S;\theta)$ is convex with respect to $\theta$.

We next prove the Lipschitz continuity of the value function. Let $\x_1^\star\in\X(S)$ be a maximizer for $\theta_1$, such that $h(S;\theta_1)=\theta_1^\top\x_1^\star$. Then
\begin{align*}
h(S;\theta_1)-h(S;\theta_2)&=\theta_1^\top\x_1^\star-\sup_{\x\in\X(S)}\theta_2^\top\x\\
&\le\theta_1^\top\x_1^\star-\theta_2^\top\x_1^\star\\
&=(\theta_1-\theta_2)^\top\x_1^\star\\
&\le\norm{\theta_1-\theta_2}_2\norm{\x_1^\star}_2,
\end{align*}
where the first inequality follows because $\x_1^\star\in\X(S)$ is a feasible candidate for the optimization problem associated with $\theta_2$, and the last inequality follows from the Cauchy--Schwarz inequality. By the definition $L_{\X}=\sup_{\x\in\X}\norm{\x}_2$, we have $\norm{\x_1^\star}_2\le L_{\X}$. Therefore,
\[
h(S;\theta_1)-h(S;\theta_2)\le L_{\X}\norm{\theta_1-\theta_2}_2.
\]
By symmetry, interchanging $\theta_1$ and $\theta_2$ yields
\[
h(S;\theta_2)-h(S;\theta_1) \le L_{\X}\norm{\theta_1-\theta_2}_2.
\]
Combining the two inequalities gives
\begin{equation}
|h(S;\theta_1)-h(S;\theta_2)| \le L_{\X}\norm{\theta_1-\theta_2}_2.
\end{equation}
This completes the proof.
\end{proof}

\subsection{Proof of Proposition~\ref{prop:l2_exact_solution}}
\begin{proof}
\begin{sloppypar}
Let $S$ be any support set with $S\subseteq[d],|S| \le H$. The inner maximization problem restricted to $S$ is $\max_{\x: \norm{\x}_2 \le L_\text{max}, \supp(\x) \subseteq S} \theta^\top \x$. By the Cauchy-Schwarz inequality, $\theta^\top \x = \theta(S)^\top \x(S) \le \norm{\theta(S)}_2 \norm{\x(S)}_2 \le \norm{\theta(S)}_2 R$. The equality holds if and only if $\x(S)$ aligns with $\theta(S)$ and lies on the boundary, i.e., $\x = L_\text{max} \frac{\theta(S)}{\norm{\theta(S)}_2}$.
Thus, the problem reduces to finding the support $S$ that maximizes the value $L_\text{max}\norm{\theta(S)}_2$. Since $\norm{\theta(S)}_2^2 = \sum_{i \in S} \theta_i^2$, maximizing this norm is equivalent to selecting the indices $i$ with the largest squared values $\theta_i^2$, which corresponds to the indices with the largest absolute values $|\theta_i|$. Therefore, the optimal support $S^\star$ is the set of the top $H$ indices by magnitude, which completes the proof.
\end{sloppypar}
\end{proof}

\subsection{Proof of Lemma~\ref{lemma:ols-bound-hp}}
\begin{proof}
From the definition of the OLS estimate in \eqref{eqn:sparse-PEE-ols}, we have
$$ \hat{\theta}_c = (cB)^{-1} \sum_{s=1}^c \sum_{k=1}^d b_k(S_k) Y_k(s) = (cB)^{-1} \sum_{s=1}^c \sum_{k=1}^d b_k(S_k) (b_k(S_k)^\top \theta_* + \eta_k(s)) $$
$$\Rightarrow \hat{\theta}_c = (cB)^{-1} \left( \sum_{s=1}^c B \theta_* + \sum_{s=1}^c \sum_{k=1}^d b_k(S_k) \eta_k(s) \right) = \theta_* + (cB)^{-1} \sum_{s=1}^c \sum_{k=1}^d b_k(S_k) \eta_k(s) .$$
Let $W(c) = \sum_{s=1}^c \sum_{k=1}^d b_k(S_k) \eta_k(s)$. Then $\hat{\theta}_c - \theta_* = \frac{1}{c} B^{-1} W(c)$.
Using the spectral norm, we have 
\begin{align}
    \norm{\hat{\theta}_c - \theta_*}_2 \le \norm{\frac{1}{c} B^{-1}}_2 \norm{W(c)}_2 \le \frac{1}{c\lambda_0} \norm{W(c)}_2. \label{eqn:bound for theta difference}
\end{align}

Now we analyze the distribution of $W(c)$ to bound $\norm{W(c)}_2^2 = \sum_{i=1}^d W(c)_i^2$.
The $i$-th component is given by the weighted sum of noise terms:
$$ W(c)_i = \sum_{s=1}^c \sum_{k=1}^d (b_k(S_k))_i \eta_k(s). $$
Recall that if $X$ is $\sigma$-sub-Gaussian, then for any scalar $a \in \R$, $aX$ is $|a|\sigma$-sub-Gaussian (variance parameter $a^2\sigma^2$). Furthermore, the sum of independent sub-Gaussian variables is sub-Gaussian with a variance parameter equal to the sum of their individual variance parameters.
Specifically, using the MGF (Moment Generating Function) property for independent $\eta_k(s)$:
\begin{align*}
    \mathbb{E}\left[ \exp(\lambda W(c)_i) \right] &= \prod_{s=1}^c \prod_{k=1}^d \mathbb{E}\left[ \exp\left( \lambda (b_k(S_k))_i \eta_k(s) \right) \right] \\
    &\le \prod_{s=1}^c \prod_{k=1}^d \exp\left( \frac{\lambda^2 (b_k(S_k))_i^2 \sigma^2}{2} \right) \tag{by Assumption~\ref{ass:noise}} \\
    &= \exp\left( \frac{\lambda^2 \sigma^2}{2} \sum_{s=1}^c \sum_{k=1}^d (b_k(S_k))_i^2 \right).
\end{align*}
This shows that $W(c)_i$ is sub-Gaussian with variance parameter $\sigma_{W,i}^2 = \sigma^2 \sum_{s=1}^c \sum_{k=1}^d (b_k(S_k))_i^2$.
Since the exploration actions $b_k(S_k)$ are identical in each cycle $s$, the inner sum is constant over $s$. Thus,
$$ \sigma_{W,i}^2 = c \sigma^2 \sum_{k=1}^d (b_k(S_k))_i^2 = c \sigma^2 (B)_{ii}, $$
where we used the definition $B = \sum_{k=1}^d b_k(S_k) b_k(S_k)^\top$, implying $(B)_{ii} = \sum_{k=1}^d (b_k(S_k))_i^2$.

By Lemma \ref{lemma:sum-abs-subgaussian}, for a fixed $c$ and dimension $i$, and any failure probability $\delta_{c,i} > 0$:
$$ P\left(|W(c)_i| \ge \sqrt{2 \sigma_{W,i}^2 \ln(2/\delta_{c,i})}\right) \le \delta_{c,i}. $$
To ensure the bound holds for \textit{all} cycles $c \ge 1$ and all dimensions $i \in [d]$ simultaneously, we apply a union bound. We allocate $\delta_{c,i} = \frac{\delta}{d \cdot 2c^2}$. Note that $\sum_{c=1}^\infty \sum_{i=1}^d \frac{\delta}{2dc^2} = \frac{\delta}{2} \sum_{c=1}^\infty \frac{1}{c^2} = \frac{\delta}{2} \frac{\pi^2}{6} < \delta$.
Thus, with probability at least $1-\delta$, for all $c \ge 1$ and $i$:
$$ |W(c)_i|^2 \le 2 c \sigma^2 (B)_{ii} \ln\left(\frac{2d c^2}{\delta}\right). $$
Summing over $i$ for a fixed $c$:
\begin{align*}
    \norm{W(c)}_2^2 = \sum_{i=1}^d |W(c)_i|^2 &\le \sum_{i=1}^d \left( 2 c\sigma^2 (B)_{ii} \ln\left(\frac{2d c^2}{\delta}\right) \right) \\
    &= 2 c\sigma^2 \ln\left(\frac{2d c^2}{\delta}\right) \sum_{i=1}^d (B)_{ii} \\
    &= 2 c\sigma^2 \mathrm{Tr}(B) \ln\left(\frac{2d c^2}{\delta}\right).
\end{align*}
Substituting this back into \eqref{eqn:bound for theta difference}:
$$ \norm{\hat{\theta}_c - \theta_*}_2^2 \le \frac{1}{(c\lambda_0)^2} \norm{W(c)}_2^2 \le \frac{2 c \sigma^2 \mathrm{Tr}(B) \ln(2d c^2/\delta)}{c^2 \lambda_0^2} = \frac{h_1 \ln(2d c^2/\delta)}{c}, $$
where $h_1 = \frac{2 \sigma^2 \mathrm{Tr}(B)}{\lambda_0^2}$. This completes the proof for the full $d$-dimensional estimator. 
\end{proof}

\begin{corollary}[Restricted OLS Concentration]
\label{coro:restricted_ols}
Suppose Assumption~\ref{ass:arms during exploration} holds. Let $\hat S\subseteq[d]$ with $|\hat S|=H$ denote a support locked by the corresponding support-recovery procedure. Let $\{e_1(\hat S),\dots,e_H(\hat S)\}\subseteq \X(\hat S)$ be the restricted exploration actions guaranteed by Assumption~\ref{ass:arms during exploration}, and let $e_k^\prime(\hat S)\in\mathbb R^H$ denote the corresponding restricted coordinate vector. Define $B_H(\hat S)
\triangleq \sum_{k=1}^H e_k^\prime(\hat S) e_k^\prime(\hat S)^\top$. Then $\lambda_{\min}\big(B_H(\hat S)\big)\ge\lambda_H>0$. Suppose that during the restricted exploration phase, the algorithm plays each action $e_k(\hat S)$ once in every restricted exploration cycle. Let $\tilde Y_k(s)=e_k(\hat S)^\top\theta_*+\tilde\eta_k(s)$ denote the corresponding reward in the $s$-th restricted exploration cycle, and define the restricted OLS estimator
\begin{equation}
\label{eqn:restricted_ols_estimator}
\bar{\theta}_m^{(H)}=\left(mB_H(\hat S)\right)^{-1}\sum_{s=1}^m\sum_{k=1}^H e_k^\prime(\hat S)\tilde Y_k(s).
\end{equation}
Then, for any $\delta_2\in(0,1)$, with probability at least $1-\delta_2$, simultaneously for all $m\ge1$,
\begin{equation}
\label{eqn:restricted_ols_concentration}
\norm{\bar{\theta}_m^{(H)}-\theta_*^\prime(\hat S)}_2^2
\le\frac{h_H(\hat S)\ln(2Hm^2/\delta_2)}{m},
\end{equation}
where
\begin{equation}
\label{eqn:def_hH}
h_H(\hat S)\triangleq\frac{2\sigma^2\operatorname{Tr}\big(B_H(\hat S)\big)}{\lambda_{\min}\big(B_H(\hat S)\big)^2}.
\end{equation}
Moreover, under Assumption~\ref{ass:feasible set bound}, $h_H(\hat S)\le\frac{2\sigma^2 H L_{\max}^2}{\lambda_H^2}=\mathcal O(H)$, where $\sigma$, $L_{\max}$, and $\lambda_H$ are regarded as problem-dependent constants.
\end{corollary}

\begin{proof}
Since every restricted exploration action $e_k(\hat S)$ is supported on $\hat S$, we have $e_k(\hat S)^\top\theta_*=e_k^\prime(\hat S)^\top\theta_*^\prime(\hat S)$. Hence the restricted reward can be written as $\tilde Y_k(s)=e_k^\prime(\hat S)^\top\theta_*^\prime(\hat S)+\tilde\eta_k(s)$. Therefore, after the support $\hat S$ is locked, the restricted
exploration phase is exactly an $H$-dimensional linear regression problem with unknown parameter $\theta_*^\prime(\hat S)\in\mathbb R^H$. Substituting the above reward expression into
\eqref{eqn:restricted_ols_estimator}, we obtain
\begin{align*}
\bar{\theta}_m^{(H)}&=\left(mB_H(\hat S)\right)^{-1}\sum_{s=1}^m\sum_{k=1}^H e_k^\prime(\hat S)
\left(e_k^\prime(\hat S)^\top\theta_*^\prime(\hat S)+\tilde\eta_k(s)\right)\\
&=\theta_*^\prime(\hat S)+\left(mB_H(\hat S)\right)^{-1}\sum_{s=1}^m\sum_{k=1}^He_k^\prime(\hat S)\tilde\eta_k(s),
\end{align*}
where we used $\sum_{s=1}^m\sum_{k=1}^He_k^\prime(\hat S)e_k^\prime(\hat S)^\top=mB_H(\hat S)$. Consequently,
\begin{equation}
\label{eqn:restricted_ols_error}
\bar{\theta}_m^{(H)}-\theta_*^\prime(\hat S)=\left(mB_H(\hat S)\right)^{-1}\sum_{s=1}^m\sum_{k=1}^He_k^\prime(\hat S)\tilde\eta_k(s).
\end{equation}
Conditioned on the history up to the time at which $\hat S$ is locked, the support $\hat S$ and the corresponding restricted exploration actions are fixed. Moreover, the reward noise remains
conditionally $\sigma$-sub-Gaussian by Assumption~\ref{ass:noise}. Therefore, the same martingale concentration argument used in Lemma~\ref{lemma:ols-bound-hp} applies to \eqref{eqn:restricted_ols_error}, with the substitutions
\[
d\longrightarrow H, \quad B\longrightarrow B_H(\hat S), \quad
b_k\longrightarrow e_k^\prime(\hat S).
\]
Hence, for any $\delta_2\in(0,1)$, with probability at least $1-\delta_2$, simultaneously for all $m\ge1$,
\[
\norm{\bar{\theta}_m^{(H)}-\theta_*^\prime(\hat S)}_2^2\le\frac{2\sigma^2
\operatorname{Tr}\big(B_H(\hat S)\big)}{\lambda_{\min}\big(B_H(\hat S)\big)^2}\cdot\frac{
\ln(2Hm^2/\delta_2)}{m},
\]
which is exactly \eqref{eqn:restricted_ols_concentration} with $h_H(\hat S)$ defined in~\eqref{eqn:def_hH}.

It remains to bound $h_H(\hat S)$. By Assumption~\ref{ass:feasible set bound}, every feasible restricted exploration action satisfies $\norm{e_k(\hat S)}_2 \le L_{\max}$. Since $e_k(\hat S)$ is supported entirely on $\hat S$, restricting it to the active coordinates does not change its Euclidean norm:
\[
\norm{e_k^\prime(\hat S)}_2=\norm{e_k(\hat S)}_2\le L_{\max}.
\]
Therefore,
\begin{align*}
\operatorname{Tr}\big(B_H(\hat S)\big) &=\operatorname{Tr}\left(\sum_{k=1}^He_k^\prime(\hat S)
e_k^\prime(\hat S)^\top\right)\\
&=\sum_{k=1}^H\norm{e_k^\prime(\hat S)}_2^2\le H L_{\max}^2.
\end{align*}
Furthermore, Assumption~\ref{ass:arms during exploration} guarantees that $\lambda_{\min}\big(B_H(\hat S)\big)\ge\lambda_H$. Combining the two bounds gives
\[
h_H(\hat S)\le\frac{2\sigma^2 H L_{\max}^2}{\lambda_H^2}=\mathcal O(H).
\]

\end{proof}

\subsection{Proof of Lemma~\ref{lemma:stopping_condition_l2ball}}
\begin{proof}
We proceed by contradiction. Suppose the condition $\Delta_c^\prime > 2\varepsilon_c$ holds but $S_c^\star \neq S^\star$. Since $|S_c^\star| = |S^\star| = H$, there must exist at least one ``noise'' index $j \in S_c^\star \setminus S^\star$ and one ``signal'' index $k \in S^\star \setminus S_c^\star$.
By the definition of $S_c^\star$ (selecting the top-$H$ indices $i_1^\prime, \dots, i_H^\prime$), and since $j$ is selected but $k$ is not, we must have:
$$ |(\hat{\theta}_c)_j| \ge |(\hat{\theta}_c)_{i_H^\prime}| \quad \text{and} \quad |(\hat{\theta}_c)_k| \le |(\hat{\theta}_c)_{i_{H+1}^\prime}|. $$
Thus, the difference is lower bounded by the empirical gap:
\begin{equation} \label{eqn:gap_derived}
    |(\hat{\theta}_c)_j| - |(\hat{\theta}_c)_k| \ge |(\hat{\theta}_c)_{i_H^\prime}| - |(\hat{\theta}_c)_{i_{H+1}^\prime}| = \Delta_c^\prime > 2\varepsilon_c.
\end{equation}
Now, we relate this to the true parameters $\theta_*$. From the error bound $\norm{\hat{\theta}_c - \theta_*}_2 \le \varepsilon_c$, we have component-wise bounds $|(\hat{\theta}_c)_i - (\theta_*)_i| \le \varepsilon_c$ for all $i \in [d]$. Using the reverse triangle inequality ($|a| - |b| \le |a-b|$):
\begin{align*}
    |(\theta_*)_j| \ge |(\hat{\theta}_c)_j| - \varepsilon_c, \quad
    |(\theta_*)_k| \le |(\hat{\theta}_c)_k| + \varepsilon_c.
\end{align*}
Subtracting these inequalities and using \eqref{eqn:gap_derived}:
$$ |(\theta_*)_j| - |(\theta_*)_k| \ge (|(\hat{\theta}_c)_j| - \varepsilon_c) - (|(\hat{\theta}_c)_k| + \varepsilon_c) = (|(\hat{\theta}_c)_j| - |(\hat{\theta}_c)_k|) - 2\varepsilon_c > 2\varepsilon_c - 2\varepsilon_c = 0. $$
This implies $|(\theta_*)_j| > |(\theta_*)_k|$. However, let us check the definitions of $j$ and $k$:
\begin{itemize}
    \item $k \in S^\star = \{i_1^\star, \dots, i_H^\star\}$ implies $|(\theta_*)_k| \ge |(\theta_*)_{i_H^\star}|$.
    \item $j \notin S^\star$ implies $j \in \{i_{H+1}^\star, \dots, i_d^\star\}$, so $|(\theta_*)_j| \le |(\theta_*)_{i_{H+1}^\star}|$.
\end{itemize}
By the definition of the sorted indices of $\theta_*$, we have $|(\theta_*)_{i_H^\star}| \ge |(\theta_*)_{i_{H+1}^\star}|$. Therefore:
$$ |(\theta_*)_k| \ge |(\theta_*)_{i_H^\star}| \ge |(\theta_*)_{i_{H+1}^\star}| \ge |(\theta_*)_j|. $$
This contradicts the derived strict inequality $|(\theta_*)_j| > |(\theta_*)_k|$. Thus, the assumption $S_c^\star \neq S^\star$ is false, and we conclude $S_c^\star = S^\star$, which completes the proof.
\end{proof}

\subsection{Proof of Theorem~\ref{thm:regret_adaptive}}\label{app:Proof of Thm PEE}
We prove the high-probability regret bound for the APSEE algorithm. Because of the tractable precise solution, we analyze the $1$-regret, defined as $R_T=\left (\sum_{t=1}^{T}\theta_*^\top\x^\star(S^\star)\right )-\left(\sum_{t=1}^{T}\theta_*^\top\x_t(S_t) \right)$. The proof proceeds by bounding the regret in the two algorithmic phases: the Warm-up Phase (where the support is not yet found) and the Stable Phase (where the restricted $H$-dimensional exploration and exploitation occur).

Let $K$ be the random total number of cycles run by the algorithm until time $T$. During the initial warm-up, each cycle takes exactly $d$ steps. Once the support is locked at cycle $C_0$, the subsequent stable cycles take $H+c$ steps each. Thus, the total time satisfies $T = C_0 d + \sum_{c=C_0}^K (H+c)$, establishing the relationship $T = \Theta(K^2) \implies K = \Theta(\sqrt{T})$.

Hence, we have the following regret decomposition:
\begin{align}\nonumber
    R_T \le R_K &= \underbrace{\sum_{c=1}^{C_0-1} \sum_{k=1}^{d} \left(\theta_*^\top\x^\star(S^\star) - \theta_*^\top b_k(S_k) \right)}_{R_{PExplore}:\text{\quad Pure Exploration Regret (Warm-up)}} \\
    &\quad +\underbrace{\sum_{c=C_0}^{K} \sum_{k=1}^{H} \left(\theta_*^\top\x^\star(S^\star) - \theta_*^\top e_k \right)}_{R_{SExplore}:\text{\quad Restricted Exploration Regret (Stable)}}+\underbrace{\sum_{c=C_0}^{K} c \left(\theta_*^\top\x^\star(S^\star) - \theta_*^\top a_c \right)}_{R_{SExploit}:\text{\quad Exploitation Regret (Stable)}},\label{eqn:split overall reg}
\end{align}
where the inequality is due to the definition of $K$ (the total time $T$ falls at a certain step of the $K$-th cycle) and the non-negativity of regret.

To guarantee that our probability bounds hold simultaneously for both estimators (i.e., during the Pure Exploration phase and the Restricted Exploration phase) throughout the entire algorithmic horizon, we allocate the failure probability budget $\delta$ via a union bound. Let $\delta_1 = \delta/2$ and $\delta_2 = \delta/2$. We define the following high-probability concentration events for the pure OLS estimator and the restricted OLS estimator, respectively:
\begin{align}
    \mathcal{E}_1 &= \left\{ \norm{\hat{\theta}_c - \theta_*}_2^2 \le \frac{h_1 \ln(2dc^2/\delta_1)}{c},\ \forall c\ge 1 \right\}, \label{eqn:event E1} \\
    \mathcal{E}_2 &= \left\{ \norm{\bar{\theta}_m^{(H)} - \theta_*^G}_2^2 \le \frac{h_H \ln(2H m^2/\delta_2)}{m}, \ \forall m\ge 1 \right\},
\label{eqn:event E2_general}
\end{align} 
where $\theta_*^G = \theta_*^\prime(S^G) \in\mathbb R^H$ denotes the restriction of the true parameter to the greedy support. 
By Lemma~\ref{lemma:ols-bound-hp} and Corollary~\ref{coro:restricted_ols}, we have $\mathbb{P}(\mathcal{E}_1) \ge 1-\delta_1$ and $\mathbb{P}(\mathcal{E}_2) \ge 1-\delta_2$. Applying the union bound, the joint event $\mathcal{E} = \mathcal{E}_1 \cap \mathcal{E}_2$ holds with probability at least $1 - (\delta_1 + \delta_2) = 1 - \delta$. In the subsequent derivations, we strictly condition our analysis on the success of the joint event $\mathcal{E}$.

\textbf{Determine $C_0$.}
We first bound the duration $C_0$ of the warm-up phase (pure exploration). 
\begin{lemma}\label{lemma:C_0}
    Under the assumption $\Delta_{\min} = |(\theta_*)_{i_H^\star}| - |(\theta_*)_{i_{H+1}^\star}| > 0$ in Theorem~\ref{thm:regret_adaptive} and conditioned the event $\mathcal{E} = \mathcal{E}_1 \cap \mathcal{E}_2$ by \eqref{eqn:event E1} and \eqref{eqn:event E2_general}, we can define $C_0$ as $C_0 = \frac{32 h_1}{\Delta_{\min}^2} \ln\left(\frac{2d}{\delta}\right) + \frac{128 h_1}{\Delta_{\min}^2} \ln\left(\frac{64 h_1}{\Delta_{\min}^2}\right)= \tilde{\mathcal{O}}\left(\frac{h_1}{\Delta_{\min}^2} \ln\frac{d}{\delta}\right).$
\end{lemma}
\begin{proof}
Recall the algorithm stops when $\Delta_c^\prime > 2\varepsilon_c$ and the true gap $\Delta_{\min} = |(\theta_*)_{i_H^\star}| - |(\theta_*)_{i_{H+1}^\star}| > 0$. By Weyl's inequality \cite[Theorem 4.3.1]{horn2012matrix}, the stability of the sorted values is bounded by the $\ell_\infty$-norm of the perturbation.
Using Weyl's inequality for the stability of sorted values (or simple vector perturbation analysis), for any rank $k$, the $k$-th largest absolute value satisfies:
$$ \left| |\hat{\theta}_c|_{(k)} - |\theta_*|_{(k)} \right| \le \norm{\hat{\theta}_c - \theta_*}_\infty \le \norm{\hat{\theta}_c - \theta_*}_2 \le \varepsilon_c. $$
Therefore, we can lower bound the empirical gap $\Delta_c^\prime$:
\begin{align*}
    \Delta_c^\prime &= |\hat{\theta}_c|_{(H)} - |\hat{\theta}_c|_{(H+1)} \\
    &\ge (|\theta_*|_{(H)} - \varepsilon_c) - (|\theta_*|_{(H+1)} + \varepsilon_c) \\
    &= (|\theta_*|_{(H)} - |\theta_*|_{(H+1)}) - 2\varepsilon_c \\
    &= \Delta_{\min} - 2\varepsilon_c.
\end{align*}
The stopping condition $\Delta_c^\prime > 2\varepsilon_c$ is satisfied if our lower bound satisfies it:
$$ \Delta_{\min} - 2\varepsilon_c > 2\varepsilon_c \implies \Delta_{\min} > 4\varepsilon_c \implies \varepsilon_c < \frac{\Delta_{\min}}{4}. $$
Substituting the expression for $\varepsilon_c = \sqrt{h_1 \ln(2dc^2/\delta)/c}$:
\begin{equation}\label{eqn:warmup_inequality_raw}
    \frac{h_1 \ln(2d c^2/\delta)}{c} < \frac{\Delta_{\min}^2}{16}.
\end{equation}
Let $\Gamma = \frac{16 h_1}{\Delta_{\min}^2}$. Rearranging terms and using $\ln(2d c^2/\delta) = \ln(2d/\delta) + 2\ln c$, we seek a sufficient condition for $c$:
\begin{equation*}
    c > \Gamma \ln\left(\frac{2d}{\delta}\right) + 2\Gamma \ln c.
\end{equation*}
It suffices to find $c$ such that $c/2$ dominates the constant term and the remaining $c/2$ dominates the logarithmic term:
\begin{align}
    c/2 &\ge \Gamma \ln(2d/\delta) \implies c \ge 2\Gamma \ln(2d/\delta), \label{eqn:cond1}\\
    c/2 &\ge 2\Gamma \ln c \implies c \ge 4\Gamma \ln c. \label{eqn:cond2}
\end{align}
To solve \eqref{eqn:cond2}, we apply a standard auxiliary result for transcendental inequalities: for any constant $A > e$, a sufficient condition for $x \ge A \ln x$ is $x \ge 2A \ln A$ (since $2A \ln A \ge A \ln(2A \ln A)$ holds for large $A$). Setting $A = 4\Gamma$, we obtain:
\begin{equation*}
    c \ge 2(4\Gamma) \ln(4\Gamma) = 8\Gamma \ln(4\Gamma).
\end{equation*}
Summing the bounds from \eqref{eqn:cond1} and the solution to \eqref{eqn:cond2}, we define $C_0$ as:
\begin{equation}\label{eqn:C_0}
    C_0 = 2\Gamma \ln\left(\frac{2d}{\delta}\right) + 8\Gamma \ln(4\Gamma) = \frac{32 h_1}{\Delta_{\min}^2} \ln\left(\frac{2d}{\delta}\right) + \frac{128 h_1}{\Delta_{\min}^2} \ln\left(\frac{64 h_1}{\Delta_{\min}^2}\right)= \tilde{\mathcal{O}}\left(\frac{h_1}{\Delta_{\min}^2} \ln\frac{d}{\delta}\right).
\end{equation}
This implies that $C_0$ depends on $\Delta_{\min}^{-2}$, but not polynomially on $T$.
\end{proof}

{\bf Bounding $R_{PExplore}$ and $R_{SExplore}$.} 
During the Warm-up Phase, the algorithm executes exactly $d$ steps per cycle. The instantaneous regret from playing any action $b_k(S_k)$ in Assumption~\ref{ass:arms during exploration} is  bounded by $\max_{v\in\X}\theta_*^\top(v-b_k(S_k))\le 2dL_{\max}\norm{\theta_*}_2$. Thus, the pure exploration regret is strictly bounded by the stopping time:
\begin{equation} \label{eqn:pure exploration regret}
    R_{PExplore} \le (C_0 \cdot d) \cdot 2 L_{\max} \|\theta_*\|_2 = \tilde{\mathcal{O}}\left( \frac{d L_{\max} \|\theta_*\|_2 h_1}{\Delta_{\min}^2} \right) \triangleq C_{gap}.
\end{equation}
This term constitutes an instance-dependent constant $C_{gap}$ independent of the total horizon $T$. 

Conversely, during the Stable Phase ($c \ge C_0$), conditioned on the support-consistency event in Lemma~\ref{lemma:general_consistency}, the support locked by Algorithm~\ref{alg:adaptive_general_PEE} satisfies $S_{est}=S_c^G=S^G$. Therefore, all restricted exploration and exploitation actions after cycle $C_0$ are supported on the fixed greedy support $S^G$. Since the algorithm only plays $H$ base actions per cycle, the stable exploration regret is bounded by:
\begin{equation} \label{eqn:stable exploration regret bound}
    R_{SExplore} \le \sum_{c=C_0}^{K} H \cdot 2 L_{\max} \|\theta_*\|_2 \le K \cdot 2 H L_{\max} \|\theta_*\|_2 = \mathcal{O}(H \sqrt{T}\norm{\theta_*}_2).
\end{equation}

{\bf Bounding $R_{SExploit}$.} 
For cycles $c \ge C_0$, the condition $\Delta_c^\prime > 2\varepsilon_c$ has been met. By Lemma~\ref{lemma:stopping_condition_l2ball}, it is guaranteed that the estimated support is correct: $S_{est} = S^\star$.
We define 
\begin{equation}\label{eqn:best sparse action response}
    \x^\star(\Omega)[\theta] = \argmax_{\x \in \X, \supp(\x) \subseteq \Omega} \theta^\top \x
\end{equation}
as the best sparse action response on a fixed support $\Omega$. For $c=C_0$, the algorithm exploits using $\hat{\theta}_{C_0}$, incurring at most $C_0 \cdot 2L_{\max}\|\theta_*\|_2$ regret, which is seamlessly absorbed into the finite constant $C_{gap}$. For $c > C_0$, let $m = c - C_0 \ge 1$ denote the index of the Restricted OLS update. Note that line 24 of Algorithm~\ref{alg:adaptive_PEE} zero-pad $\bar{\theta}_m^{(H)}$ outside the support $S_{est}$ to construct the $d$-dimensional estimator $\bar{\theta}_m^{(d)} \in \R^d$. Hence, The chosen action $a_c = \x^\star(S^\star)[\bar{\theta}_m^{(d)}]$ is $d$-dim. The target optimal action is $\x^\star(S^\star) = u^\star(S^\star)[\theta_*]$.

Thus, The instantaneous regret in cycle $c > C_0$ is:
$$ r(c) = \theta_*^\top u^\star(S^\star)[\theta_*] - \theta_*^\top a_c = \theta_*(S^\star)^\top u^\star(S^\star)[\theta_*] - \theta_*(S^\star)^\top u^\star(S^\star)[\bar{\theta}_m^{(d)}]. $$
To bound this, we add and subtract the cross terms $\bar{\theta}_m^{(d)}(S^\star)^\top u^\star(S^\star)[\theta_*]$ and $\bar{\theta}_m^{(d)}(S^\star)^\top u^\star(S^\star)[\bar{\theta}_m^{(d)}]$ to algebraically decouple the parameter error from the action difference:
\begin{align*}
    r(c) &= (\theta_*(S^\star) - \bar{\theta}_m^{(d)}(S^\star))^\top u^\star(S^\star)[\theta_*] + \bar{\theta}_m^{(d)}(S^\star)^\top u^\star(S^\star)[\theta_*] - \theta_*(S^\star)^\top u^\star(S^\star)[\bar{\theta}_m^{(d)}] \\
    &= (\theta_*(S^\star) - \bar{\theta}_m^{(d)}(S^\star))^\top u^\star(S^\star)[\theta_*] \\
    &\quad + \bar{\theta}_m^{(d)}(S^\star)^\top ( u^\star(S^\star)[\theta_*] - u^\star(S^\star)[\bar{\theta}_m^{(d)}] )+ (\bar{\theta}_m^{(d)}(S^\star) - \theta_*(S^\star))^\top u^\star(S^\star)[\bar{\theta}_m^{(d)}].
\end{align*}
\begin{sloppypar}
By definition, the action $u^\star(S^\star)[\bar{\theta}_m^{(d)}]$ maximizes the inner product with the estimated parameter $\bar{\theta}_m^{(d)}$ over the restricted support $S^\star$. Thus, we have the optimality condition $\bar{\theta}_m^{(d)}(S^\star)^\top u^\star(S^\star)[\bar{\theta}_m^{(d)}] \ge \bar{\theta}_m^{(d)}(S^\star)^\top u^\star(S^\star)[\theta_*]$, which implies that the middle term in the decomposition is non-positive, i.e., 
$ \bar{\theta}_m^{(d)}(S^\star)^\top \big( u^\star(S^\star)[\theta_*] - u^\star(S^\star)[\bar{\theta}_m^{(d)}] \big) \le 0 $. 
By dropping this non-positive maximization difference term and combining the remaining terms, we obtain: \end{sloppypar}
\begin{align}\nonumber
    r(c) &\le (\theta_*(S^\star) - \bar{\theta}_m^{(d)}(S^\star))^\top u^\star(S^\star)[\theta_*] + (\bar{\theta}_m^{(d)}(S^\star) - \theta_*(S^\star))^\top u^\star(S^\star)[\bar{\theta}_m^{(d)}] \\
    &= (\bar{\theta}_m^{(d)}(S^\star) - \theta_*(S^\star))^\top ( u^\star(S^\star)[\bar{\theta}_m^{(d)}] - u^\star(S^\star)[\theta_*] ). \label{eqn:r(c)1}
\end{align}
Since the non-zero elements of both actions ($u^\star(S^\star)[\bar{\theta}_m^{(d)}]$ and $u^\star(S^\star)[\theta_*]$ in \eqref{eqn:r(c)1}) lie in the fixed $H$-dimensional subspace defined by $S^\star$, and the action set (Euclidean ball) is strongly convex on fixed supports, we invoke Proposition~\ref{prop:lipschitz_h}. The map from parameter to optimal action on the Euclidean ball is $L_\text{max}$-Lipschitz (where $L_\text{max}$ is the radius of $\X$). Specifically, $u^\star(S^\star)[\theta] = L_\text{max} \frac{\theta(S^\star)}{\norm{\theta(S^\star)}_2}$.
Using the standard inequality for normalized vectors $\norm{\frac{w}{\norm{w}_2} - \frac{z}{\norm{z}_2}}_2 \le \frac{2\norm{w-z}_2}{\norm{z}_2}$ \cite[Lemma 3.5]{rusmevichientong2010linearly}, we have:
\begin{align*}
    \norm{u^\star(S^\star)[\bar{\theta}_m^{(d)}] - u^\star(S^\star)[\theta_*]}_2 
    &= L_\text{max} \norm{\frac{\bar{\theta}_m^{(d)}(S^\star)}{\norm{\bar{\theta}_m^{(d)}(S^\star)}_2} - \frac{\theta_*(S^\star)}{\norm{\theta_*(S^\star)}_2}}_2 \le L_\text{max} \frac{2\norm{\bar{\theta}_m^{(d)}(S^\star) - \theta_*(S^\star)}_2}{\norm{\theta_*(S^\star)}_2}.
\end{align*}
Substituting this back into \eqref{eqn:r(c)1} yields:
$$ r(c) \le \frac{2L_\text{max}}{\norm{\theta_*(S^\star)}_2} \norm{\bar{\theta}_m^{(d)}(S^\star) - \theta_*(S^\star)}_2^2. $$
We now substitute the high-probability Restricted OLS error bound $\norm{\bar{\theta}_m^{(H)} - \theta_*^\prime(S^\star)}_2^2 \le \frac{h_H \ln(2H m^2/\delta)}{m}$ in Corollary~\ref{coro:restricted_ols} (conditioned on event $\mathcal{E}$) and notice that $\norm{\bar{\theta}_m^{(d)}(S^\star) - \theta_*(S^\star)}_2^2=\norm{\bar{\theta}_m^{(H)} - \theta_*^\prime(S^\star)}_2^2$:
$$ r(c) \le \frac{2L_\text{max} h_H \ln(2H m^2/\delta)}{m \norm{\theta_*(S^\star)}_2}. $$
Therefore, the stable exploitation regret $R_{SExploit}$ for $c > C_0$ is bounded by evaluating the sum over $m = c - C_0$:
\begin{align}\nonumber
    R_{SExploit} &\le \sum_{c=C_0+1}^{K} c \cdot r(c) \le \sum_{m=1}^{K-C_0} (m+C_0) \frac{2L_\text{max} h_H \ln(2H m^2/\delta)}{m \norm{\theta_*(S^\star)}_2} \\\nonumber
    &= \frac{2L_\text{max} h_H}{\norm{\theta_*(S^\star)}_2} \sum_{m=1}^{K-C_0} \left(1 + \frac{C_0}{m}\right) \ln\left(\frac{2H m^2}{\delta}\right) \\ \nonumber
    &\le \frac{2L_\text{max} h_H \ln(2H K^2/\delta)}{\norm{\theta_*(S^\star)}_2} \left[ (K-C_0) + C_0 \ln(K-C_0) \right] \\
    &\le \frac{2L_\text{max} h_H \ln(2H K^2/\delta)}{\zeta_\star} \left[ (K-C_0) + C_0 \ln(K-C_0) \right],
    \label{eqn:stable exploitation regret}
\end{align}
where the last inequality is due to the definition of the restricted signal strength
$\zeta_\star=\norm{\theta_*(S^\star)}_2$. Corollary~\ref{coro:restricted_ols} gives us $h_H=\O(H)$. Utilizing $K = \Theta(\sqrt{T})$ and separating the constants, the above sum expands to a dominant term of $\tilde{\mathcal{O}}(H \sqrt{T} / \zeta_\star)$ and a secondary term $\tilde{\mathcal{O}}(C_0 H / \zeta_\star)$ which is subsumed into $C_{gap}$.

\textbf{Final Total Regret.}
Substituting \eqref{eqn:pure exploration regret}, \eqref{eqn:stable exploration regret bound}, and \eqref{eqn:stable exploitation regret} back into the decomposition \eqref{eqn:split overall reg} yields the final high-probability regret, with probability $1-\delta$:
\begin{align*}
    R_T \le C_{gap} + \mathcal{O}(H\sqrt{T}\norm{\theta_*}_2) + \tilde{\mathcal{O}}\left(\frac{H\sqrt{T}}{\zeta_\star}\right) = \tilde{\CO}\left(\left(\norm{\theta_*}_2+\frac{1}{\zeta_\star}\right)H\sqrt{T}\right) + C_{gap}.
\end{align*}
This completes the proof of Theorem~\ref{thm:regret_adaptive}.
$\hfill\blacksquare$

\section{Omitted Proofs in Section~\ref{sec:smooth_convex_case}}
\label{app:proofs_smooth_convex_case}

\subsection{Proof of Lemma~\ref{lemma:general_consistency}}

We proceed by induction on the steps $k=1,\dots,H$. Assume that the selected sets match up to step $k-1$, and let $G_{k-1}$ denote their common support. Let $\hat g_k$ and $g_k^*$ be the elements selected at step $k$ by the Greedy algorithms using $\hat\theta_c$ and $\theta_*$, respectively. The empirical gap is
\[
\delta_k'=\rho_k(\hat g_k;\hat\theta_c)-\max_{v\neq\hat g_k}\rho_k(v;\hat\theta_c).
\]
By the stopping condition,
\[
\delta_k'\geq\Delta_c^{Greedy}>2L_{\X}\varepsilon_c.
\]
To prove that $\hat g_k=g_k^*$, we use contradiction. Suppose that $\hat g_k\neq g_k^*$. Then, in the empirical view based on $\hat\theta_c$, $\hat g_k$ is strictly better than $g_k^*$:
\[
\rho_k(\hat g_k;\hat\theta_c)-\rho_k(g_k^*;\hat\theta_c)\geq\delta_k'>2L_{\X}\varepsilon_c.
\]
By the definition of marginal gain,
\[
\rho_k(v;\theta)=h(G_{k-1}\cup\{v\};\theta)-h(G_{k-1};\theta),
\]
we have
\begin{align*}
\rho_k(\hat g_k;\hat\theta_c)-\rho_k(g_k^*;\hat\theta_c)
&=\left[h(G_{k-1}\cup\{\hat g_k\};\hat\theta_c)-h(G_{k-1};\hat\theta_c)\right]-
\left[h(G_{k-1}\cup\{g_k^*\};\hat\theta_c)-h(G_{k-1};\hat\theta_c)\right]\\
&=h(G_{k-1}\cup\{\hat g_k\};\hat\theta_c)-h(G_{k-1}\cup\{g_k^*\};\hat\theta_c).
\end{align*}
Thus,
\begin{align}
h(G_{k-1}\cup\{\hat g_k\};\hat\theta_c)-h(G_{k-1}\cup\{g_k^*\};\hat\theta_c)>2L_{\X}\varepsilon_c.
\label{eqn:h11}
\end{align}
We now relate \eqref{eqn:h11} to the true parameter $\theta_*$. By the support-value Lipschitz property \eqref{eqn:support_value_lipschitz}, we have
\[
|h(S;\hat\theta_c)-h(S;\theta_*)|\leq L_{\X}\varepsilon_c
\]
for any support set $S$. Applying this inequality twice gives
\begin{align*}
h(G_{k-1}\cup\{\hat g_k\};\theta_*)&\geq
h(G_{k-1}\cup\{\hat g_k\};\hat\theta_c)-L_{\X}\varepsilon_c,\\
h(G_{k-1}\cup\{g_k^*\};\theta_*)&\leq
h(G_{k-1}\cup\{g_k^*\};\hat\theta_c)+L_{\X}\varepsilon_c.
\end{align*}
Subtracting the above inequalities yields
\begin{align*}
\rho_k(\hat g_k;\theta_*)-\rho_k(g_k^*;\theta_*)&=h(G_{k-1}\cup\{\hat g_k\};\theta_*)
-h(G_{k-1}\cup\{g_k^*\};\theta_*)\\
&\geq
h(G_{k-1}\cup\{\hat g_k\};\hat\theta_c)-h(G_{k-1}\cup\{g_k^*\};\hat\theta_c)-2L_{\X}\varepsilon_c\\
&\overset{\eqref{eqn:h11}}{>}2L_{\X}\varepsilon_c-2L_{\X}\varepsilon_c\\
&=0,
\end{align*}
which implies $\rho_k(\hat g_k;\theta_*)>\rho_k(g_k^*;\theta_*)$. This contradicts the definition of $g_k^*$ as the element with the maximum marginal gain under $\theta_*$. Thus, $\hat g_k=g_k^*$. By induction, the entire support sets are identical, i.e., $S_c^G=S^G$. This completes the proof.

\subsection{Proof of Proposition~\ref{prop:submodularity_general}}

For any $S\subseteq[d]$, define the coordinate subspace $E_S\triangleq\{\x\in\R^d:\supp(\x)\subseteq S\}$, and let $P_S$ denote the orthogonal projection onto
$E_S$. For any $i\notin S$, define the marginal gain
\[
\rho(i\mid S)
\triangleq   
h(S\cup\{i\};\theta)-h(S;\theta).
\]
Since $\X$ contains the origin and the feasible sets are nested, the set function $h(\cdot;\theta)$ is nonnegative and monotone nondecreasing. We first establish the following essential property:
for any $S\subseteq U\subseteq[d]$ satisfying $h(U;\theta)>h(S;\theta)$, there exists at least one coordinate $i\in U\setminus S$ such that
\begin{equation}
\label{eqn:positive_marginal_exists}
\rho(i\mid S)>0.
\end{equation}
We prove this property by contradiction. Suppose that
\[
h(U;\theta)>h(S;\theta),
\qquad
\rho(i\mid S)=0,
\quad\forall i\in U\setminus S.
\]
We consider two cases.

\medskip
\noindent
\textbf{Case 1: $\theta_S=0$.}
In this case, $h(S;\theta)=0$. Since $h(U;\theta)>h(S;\theta)=0$, we must have $\theta_U\neq0$. Moreover, because $\theta_S=0$, there exists at least one coordinate $i\in U\setminus S$ such that $\theta_i\neq0$. By $0\in\operatorname{int}(\X)$, there exists $r>0$ such that $r\mathbb B_2^d\subseteq\X$, where $\mathbb B_2^d$ denotes the unit Euclidean ball. Consequently, the action $\widetilde{\x}=r\,\operatorname{sign}(\theta_i)\mathbf e_i$ belongs to $\X$ and has support contained in $\{i\}$.
It follows that
\begin{align*}
h(S\cup\{i\};\theta)
&\geq
\theta^\top\widetilde{\x}\\
&=r|\theta_i|\\
&>0=h(S;\theta),
\end{align*}
which contradicts $\rho(i\mid S)=0$. Hence, Case 1 cannot occur.

\medskip
\noindent
\textbf{Case 2: $\theta_S\neq0$.}
Let $\x_S\in\argmax_{\x\in\X\cap E_S}\theta^\top\x$. Since $\theta_S\neq0$, the optimizer $\x_S$ must lie on the boundary $\partial\X$. Indeed, if $\x_S$ were an interior point of $\X$, then for sufficiently small $\varepsilon>0$, the point $\x_S+\varepsilon P_S\theta$ would remain feasible and belong to $E_S$, while its objective value would be strictly larger than $\theta^\top\x_S$. This contradicts the optimality of $\x_S$.

By the smoothness of $\partial\X$, there exists a unique outward unit normal vector $\n$ to $\X$ at $\x_S$. In particular, the normal cone of $\X$ at $\x_S$ is $N_{\X}(\x_S)
=\{\lambda\n:\lambda\geq0\}$, where the normal cone is defined as $N_C(\x)\triangleq\{\v:\v^\top(\y-\x)\leq0,\ \forall\y\in C\}$.

Now, by the assumption $\rho(i\mid S)=0$, we have
\[
h(S\cup\{i\};\theta)=h(S;\theta),\qquad\forall i\in U\setminus S.
\]
Therefore, the same point $\x_S$ is also an optimizer of the enlarged optimization problem over $\X\cap E_{S\cup\{i\}}$, namely, $\x_S\in\argmax_{\x\in\X\cap E_{S\cup\{i\}}}\theta^\top\x$.

We next invoke the normal-cone intersection rule. Since $0\in\operatorname{int}(\X)$ and
$0\in E_{S\cup\{i\}}$, the relative-interior qualification holds for the intersection of $\X$ and $E_{S\cup\{i\}}$. Thus,
\[
\begin{aligned}
N_{\X\cap E_{S\cup\{i\}}}(\x_S)&=N_{\X}(\x_S)+N_{E_{S\cup\{i\}}}(\x_S)\\
&=N_{\X}(\x_S)+E_{S\cup\{i\}}^\perp.
\end{aligned}
\]
By the first-order optimality condition for maximizing a linear function over a convex set, $\theta\in N_{\X\cap E_{S\cup\{i\}}}(\x_S)$. Consequently, for each $i\in U\setminus S$, there exist $\lambda_i\geq0$ and $\z_i\in E_{S\cup\{i\}}^\perp$ such that $\theta=\lambda_i\n+\z_i$. Projecting both sides onto $E_{S\cup\{i\}}$ yields
\begin{equation}
\label{eqn:normal_projection_single_coordinate}
P_{S\cup\{i\}}\theta=\lambda_iP_{S\cup\{i\}}\n.
\end{equation}
In particular, restricting \eqref{eqn:normal_projection_single_coordinate} to the common coordinates in $S$, we obtain
\[
\theta_S=\lambda_i\n_S,\qquad\forall i\in U\setminus S.
\]
Since $\theta_S\neq0$, it follows that $\n_S\neq0$. Therefore, the scalar $\lambda_i$ is uniquely determined by the above equality and must be identical for all $i\in U\setminus S$. Denote this common value by $\lambda>0$.

Combining the equality on $S$ with \eqref{eqn:normal_projection_single_coordinate} on each additional coordinate $i\in U\setminus S$, we obtain $P_U\theta=\lambda P_U\n$. On the other hand, since $\n$ is the outward normal vector to the convex set $\X$ at $\x_S$, we have
\[
\n^\top(\y-\x_S)\leq0, \qquad\forall\y\in\X.
\]
For any $\y\in\X\cap E_U$, both $\y$ and $\x_S$ belong to $E_U$. Hence,
\begin{align*}
\theta^\top(\y-\x_S)&=(P_U\theta)^\top(\y-\x_S)\\
&=\lambda(P_U\n)^\top(\y-\x_S)\\
&=\lambda\n^\top(\y-\x_S)\\
&\leq0.
\end{align*}
This shows that $\x_S$ is also an optimizer over $\X\cap E_U$, implying $h(U;\theta)=h(S;\theta)$, which contradicts our assumption that $h(U;\theta)>h(S;\theta)$. Therefore, Case 2 is also impossible. We conclude that \eqref{eqn:positive_marginal_exists} holds.

\medskip
\noindent
\textbf{Completion of the proof.} Since $\theta\neq0$ and $0\in\operatorname{int}(\X)$, there exists at least one coordinate $j\in[d]$ such that $\theta_j\neq0$.
By the interior-point condition, $h(\{j\};\theta) \geq r|\theta_j|>0 =h(\emptyset;\theta)$. Thus, the collection of positive marginal gains in \eqref{eqn:minimum_positive_marginal_gain} is nonempty. Furthermore, since the ground set $[d]$ is finite, there are only finitely many pairs $(S,i)$ with $S\subseteq[d]$ and $i\notin S$. Consequently, $\delta_{\mathrm{marg}}>0$.

Now consider any $S\subseteq U\subseteq[d]$ satisfying $h(U;\theta)>h(S;\theta)$. By \eqref{eqn:positive_marginal_exists}, at least one coordinate $i\in U\setminus S$ has a strictly positive marginal gain. By the definition of $\delta_{\mathrm{marg}}$, we therefore have
\[
\sum_{i\in U\setminus S}\left[h(S\cup\{i\};\theta)-h(S;\theta)\right]\geq\delta_{\mathrm{marg}}.
\]
Moreover, since $h(\cdot;\theta)$ is nonnegative and monotone nondecreasing,
\[
0<h(U;\theta)-h(S;\theta) \leq h([d];\theta).
\]
Combining the above inequalities gives
\[
\begin{aligned}
\frac{\sum_{i\in U\setminus S}\left[h(S\cup\{i\};\theta)-h(S;\theta)\right]}{h(U;\theta)-h(S;\theta)}\geq\frac{\delta_{\mathrm{marg}}}{h([d];\theta)}>0.
\end{aligned}
\]
Since the above inequality holds for every pair $S\subseteq U\subseteq[d]$ with $h(U;\theta)>h(S;\theta)$, it follows from Definition~\ref{def:submodularity ratio} that
\[
\gamma\geq
\frac{\delta_{\mathrm{marg}}}{h([d];\theta)}
>0.
\]
This completes the proof.

\subsection{Comparison with the Conference-Version Strong-Convexity Claim}
\label{appsub:comparison_conference_version}
We further explain why the submodularity-ratio proposition in the conference version cannot be established for the entire class of strongly convex action sets, whereas the  smoothness condition adopted in this paper yields a valid positive lower bound. The key distinction lies in the geometry of the boundary and, in particular, the uniqueness of its outward normal vector.

\paragraph{Failure of the Strong-Convexity Condition}
Under the ball-intersection formulation of strong convexity used in the conference version, consider the following action set in $\R^3$: $\X=B(\c_2,2)\cap B(\c_3,2)$, where $\c_2=(1-\sqrt{2},-\sqrt{2},0)^\top,\ \c_3=(1-\sqrt{2},0,-\sqrt{2})^\top$, and $B(\c,R)
=\{\x\in\R^3:\|\x-\c\|_2\leq R\}$. This set is an intersection of two Euclidean balls with the same radius and is therefore strongly convex under the conference-version definition. Moreover, the origin belongs to its interior because 
\[
\|\c_2\|_2^2=\|\c_3\|_2^2=5-2\sqrt{2}<4.
\]
Consider the parameter $\theta=(1,1,1)^\top$ and the point $\x_0=(1,0,0)^\top$. We have $\|\x_0-\c_2\|_2^2=\|\x_0-\c_3\|_2^2=4$. Thus, both ball constraints are active at $\x_0$.

The corresponding outward normal directions are $\n_2=\frac{1}{\sqrt{2}}(1,1,0)^\top,
\ \n_3=\frac{1}{\sqrt{2}}(1,0,1)^\top$. Clearly, $\n_2$ and $\n_3$ are not parallel.
Consequently, the boundary of $\X$ is not continuously differentiable at $\x_0$, and its normal cone contains multiple distinct normal directions.

Now consider the support sets $S=\{1\},\ U=\{1,2,3\}$. We first show that $h(\{1\};\theta)=1$. Indeed, for any feasible action of the form $\x=(x_1,0,0)^\top$, the constraint $\x\in B(\c_2,2)$ implies $(x_1-1+\sqrt{2})^2+2\leq4$, which gives $x_1\leq1$. Since $\x_0\in\X$, the maximum is attained at $\x_0$.

Next, since $\X\subseteq B(\c_2,2)$, for any $\x$ supported on $\{1,2\}$, the Cauchy--Schwarz inequality gives
\[
\begin{aligned}
x_1+x_2&\leq(c_2)_1+(c_2)_2+2\|(1,1,0)^\top\|_2\\
&=1-2\sqrt{2}+2\sqrt{2}\\
&=1.
\end{aligned}
\]
Again, equality is attained at $\x_0$. Therefore, $h(\{1,2\};\theta)=1$. By the symmetric argument using $\X\subseteq B(\c_3,2)$, we also obtain $h(\{1,3\};\theta)=1$. However, consider the action $\x(t)=(1-t-2t^2,t,t)^\top$. A direct calculation gives
\[
\begin{aligned}
\|\x(t)-\c_2\|_2^2&=\|\x(t)-\c_3\|_2^2\\&=4+4t^4+4t^3+(3-4\sqrt{2})t^2.
\end{aligned}
\]
For $t=0.1$, both squared distances are strictly smaller than $4$. Hence, $\x(0.1)\in\X$.
Moreover, $\theta^\top\x(0.1)=1+0.1-2(0.1)^2=1.08>1$. It follows that $h(\{1,2,3\};\theta)\geq1.08>1$. Combining these results, we obtain
\[
\begin{aligned}
\rho(2\mid\{1\})&=h(\{1,2\};\theta)-h(\{1\};\theta)=0,\\
\rho(3\mid\{1\})&=h(\{1,3\};\theta)-h(\{1\};\theta)=0,
\end{aligned}
\]
while $h(\{1,2,3\};\theta)-h(\{1\};\theta)\geq0.08>0$. Therefore,
\[
\frac{\rho(2\mid\{1\})+\rho(3\mid\{1\})}{h(\{1,2,3\};\theta)-h(\{1\};\theta)}=0.
\]
This demonstrates that the submodularity ratio can be zero even for an action set that is strongly convex under the conference-version definition.

In particular, the singleton-based lower bound proposed in the conference version is also invalid: because $0\in\operatorname{int}(\X)$ and $\theta_2=\theta_3=1$, both singleton optimization values are strictly positive, although the actual marginal gains from adding coordinates $2$ and $3$ to $S=\{1\}$ are both zero.

\paragraph{Why Smoothness Resolves the Issue}
The failure above is caused by the non-uniqueness of the outward normal vector at the intersection point $\x_0$. Different single-coordinate extensions can satisfy their first-order optimality conditions through different active boundary normals. Consequently, the absence of individual marginal gains does not rule out a positive joint gain.

The smoothness condition eliminates this possibility. At every boundary point $\x_S$, the normal cone is a single ray: $N_{\X}(\x_S)=\{\lambda\n:\lambda\geq0\}$. If all single-coordinate marginal gains vanish, the normal-cone optimality conditions imply
\[
\theta_S=\lambda_i\n_S,\qquad \forall i\in U\setminus S.
\]
Since $\theta_S\neq0$, all the coefficients $\lambda_i$ must coincide. This forces $P_U\theta=\lambda P_U\n$, which implies that the original optimizer $\x_S$ also remains optimal after jointly adding all coordinates in $U\setminus S$.

Thus, a positive joint gain necessarily implies a positive marginal gain for at least one individual coordinate. The finite-dimensional argument then yields the strictly positive instance-dependent lower bound $\gamma\geq
\frac{\delta_{\mathrm{marg}}}{h([d];\theta)}>0$.

We emphasize that the new assumption is a smoothness boundary condition rather than a uniform positive-curvature condition. The resulting lower bound is instance-dependent and does not imply a uniform positive constant over all smooth convex action sets and parameter vectors.

Finally, the smoothness condition is used specifically to establish the positive submodularity ratio. It does not, by itself, guarantee Lipschitz continuity of the restricted optimal-action mapping. This is why the latter property is stated separately in Assumption~\ref{ass:smooth_restricted_optimizer} and is used only in the regret analysis. This distinction ensures that the approximation guarantee and the exploitation-regret analysis are supported by their respective assumptions.

\subsection{Ellipsoidal Refinement of the Submodularity-Ratio Bound}
\label{appsub:ellipsoid_submodularity}
For completeness, we derive the explicitsubmodularity-ratio lower bound for ellipsoidal
action sets stated in \eqref{eqn:ellipsoid_gamma_bound}. Consider $\X=\{\x\in\R^d:\x^\top A\x\leq1\},\ A\succ0$. 
For any nonempty support set $S\subseteq[d]$, let $A_{SS}$ denote the principal submatrix of $A$ obtained by retaining the rows and columns indexed by $S$. Similarly, $\x_S$ and $\theta_S$ denote the corresponding subvectors. Since $A\succ0$, we have $A_{SS}\succ0$. The corresponding optimization problem can therefore be written as
\[
h(S;\theta)=\max_{\x_S^\top A_{SS}\x_S\leq1}\theta_S^\top\x_S.
\]
By the Cauchy--Schwarz inequality in the $A_{SS}$-induced inner product, we obtain
\begin{equation}
\label{eqn:ellipsoid_support_value}
h(S;\theta)=\sqrt{\theta_S^\top A_{SS}^{-1}\theta_S}.
\end{equation}
Define $f(S)\triangleq h^2(S;\theta)=\theta_S^\top A_{SS}^{-1}\theta_S$. Consider any $S\subseteq U\subseteq[d]$ and let $V=U\setminus S$. For nonempty $S$, partition $A_{UU}$ and $\theta_U$ as
\[
A_{UU}=\begin{pmatrix}A_{SS}&A_{SV}\\A_{VS}&A_{VV}
\end{pmatrix},
\qquad
\theta_U=\begin{pmatrix}\theta_S\\\theta_V
\end{pmatrix}.
\]
Define the Schur complement
\begin{equation}
\label{eqn:ellipsoid_schur_complement}
C\triangleq A_{VV}-A_{VS}A_{SS}^{-1}A_{SV}
\end{equation}
and the residual vector
\begin{equation}
\label{eqn:ellipsoid_residual}
r\triangleq\theta_V-A_{VS}A_{SS}^{-1}\theta_S.
\end{equation}
Since $A\succ0$, we have $C\succ0$. By the block matrix inverse formula,
\begin{equation}
\label{eqn:ellipsoid_joint_gain}
f(U)-f(S)=r^\top C^{-1}r.
\end{equation}
For $S=\emptyset$, the same identity holds under the convention $C=A_{UU}$ and $r=\theta_U$.

Next, consider adding a single coordinate $i\in V$. Applying the same Schur-complement identity to the one-dimensional extension gives
\begin{equation}
\label{eqn:ellipsoid_single_gain}
f(S\cup\{i\})-f(S) = \frac{r_i^2}{C_{ii}}.
\end{equation}

Consequently,
\begin{align*}
\sum_{i\in V}\left[f(S\cup\{i\})-f(S)\right]=
\sum_{i\in V}\frac{r_i^2}{C_{ii}}\geq\frac{\|r\|_2^2}{\max_{i\in V}C_{ii}}.
\end{align*}
On the other hand, by \eqref{eqn:ellipsoid_joint_gain},
\[
\begin{aligned}
f(U)-f(S)=r^\top C^{-1}r\leq\frac{\|r\|_2^2}{\lambda_{\min}(C)}.
\end{aligned}
\]
Combining the above inequalities yields
\begin{equation}
\label{eqn:ellipsoid_f_ratio}
\frac{\sum_{i\in V}[f(S\cup\{i\})-f(S)]}{f(U)-f(S)}\geq\frac{\lambda_{\min}(C)}{\max_{i\in V}C_{ii}}.
\end{equation}
We now transfer this bound from $f(S)$ to the original support-value function $h(S;\theta)$. Since $h(S;\theta)=\sqrt{f(S)}$, we have
\[
\begin{aligned}
h(S\cup\{i\};\theta)-h(S;\theta)=\frac{f(S\cup\{i\})-f(S)}{h(S\cup\{i\};\theta)+h(S;\theta)}.
\end{aligned}
\]
By monotonicity, $h(S)\leq h(S\cup\{i\})\leq h(U)$. Therefore,
\[
\begin{aligned}
\sum_{i\in V}[h(S\cup\{i\};\theta)-h(S;\theta)]\geq
\frac{\sum_{i\in V}[f(S\cup\{i\})-f(S)]}{h(U;\theta)+h(S;\theta)
}.
\end{aligned}
\]
Meanwhile,
\[
h(U;\theta)-h(S;\theta)=\frac{f(U)-f(S)}{h(U;\theta)+h(S;\theta)
}.
\]
The common denominator cancels, and hence
\[
\begin{aligned}
\frac{\sum_{i\in V}[h(S\cup\{i\};\theta)-h(S;\theta)]}{h(U;\theta)-h(S;\theta)}\geq\frac{\lambda_{\min}(C)}{\max_{i\in V}C_{ii}}.
\end{aligned}
\]
It remains to relate the Schur complement to the eigenvalues of $A$. Since $C=A_{VV}-A_{VS}A_{SS}^{-1}A_{SV}$, we have $C\succeq\lambda_{\min}(A)I$, and thus $\lambda_{\min}(C)\geq\lambda_{\min}(A)$. Moreover,
\[
C_{ii}\leq A_{ii},
\qquad i\in V.
\]
It follows that $\max_{i\in V}C_{ii}\leq\max_{i\in[d]}A_{ii}$. Consequently, $\gamma
\geq\frac{\lambda_{\min}(A)}{\max_{i\in[d]}A_{ii}}$. Finally, since $\max_{i\in[d]}A_{ii}
\leq\lambda_{\max}(A)$, we obtain
\[
\gamma\geq\frac{\lambda_{\min}(A)}
{\max_{i\in[d]}A_{ii}}\geq\frac{\lambda_{\min}(A)}
{\lambda_{\max}(A)}
>0.
\]
This completes the proof.

\subsection{Proof of Theorem~\ref{thm:regret_general}}

\begin{proof}
The total $\alpha$-regret consists of regret from the warm-up phase (Full Exploration) and the stable phase (Restricted Exploration and Exploitation). Let $K$ be the final cycle index, and let $C_0^G$ be the first cycle in which the stopping condition $\Delta_c^{Greedy}>2L_{\X}\varepsilon_c$ is satisfied.

Since the support is identified during the Full Exploration phase of cycle $C_0^G$, the algorithm performs Full Exploration in cycles $c=1,\dots,C_0^G$, Restricted Exploration in cycles $c=C_0^G+1,\dots,K$, and Exploitation in cycles $c=C_0^G,\dots,K$. Up to a possible final partial cycle, whose contribution can be bounded in the same way, the cumulative $\alpha$-regret decomposes as
\begin{align}
\alpha\text{-}R_T&=\underbrace{\sum_{c=1}^{C_0^G}\sum_{k=1}^{d}\left(\alpha\theta_*^\top\x^\star(S^\star)-\theta_*^\top b_k(S_k)\right)}_{\alpha\text{-}R_{PExplore}:\text{\quad Pure Exploration $\alpha$-Regret}}+
\underbrace{\sum_{c=C_0^G+1}^{K}\sum_{k=1}^{H}\left(\alpha\theta_*^\top\x^\star(S^\star)-\theta_*^\top e_k\right)}_{\alpha\text{-}R_{SExplore}:
\text{\quad Restricted Exploration $\alpha$-Regret}}\nonumber\\
&\quad+\underbrace{\sum_{c=C_0^G}^{K}c\left(\alpha\theta_*^\top\x^\star(S^\star)-\theta_*^\top a_c\right)}_{\alpha\text{-}R_{SExploit}:\text{\quad Exploitation $\alpha$-Regret}}.
\label{eqn:split overall reg2}
\end{align}

To guarantee that the probability bounds hold simultaneously for both estimators throughout the algorithmic horizon, we allocate the failure probability budget $\delta$ via a union bound. Let $\delta_1=\delta/2,\ \delta_2=\delta/2$. We define the following high-probability concentration events for the full OLS estimator $\hat\theta_c$ and the restricted OLS estimator $\bar\theta_m^{(H)}$, respectively:
\begin{align}
\mathcal E_1&=\left\{\|\hat\theta_c-\theta_*\|_2^2\leq\frac{h_1\ln(2dc^2/\delta_1)}{c},\ \forall c\geq1\right\},\label{eqn:event E1_general2}\\
\mathcal E_2&=\left\{\|\bar\theta_m^{(H)}-\theta_*^G\|_2^2\leq\frac{h_H\ln(2Hm^2/\delta_2)}{m},\ \forall m\geq1\right\}. \label{eqn:event E2_general2}
\end{align}
Here, $\theta_*^G=\theta_*^\prime(S^G)$ denotes the $H$-dimensional restriction of the true parameter to the greedy support.

By Lemma~\ref{lemma:general_consistency}, on the support-consistency event, $S_{est}=S^G$, so Corollary~\ref{coro:restricted_ols} applies with $\hat S=S^G$. By Lemma~\ref{lemma:ols-bound-hp} and Corollary~\ref{coro:restricted_ols}, we have
\[
\mathbb P(\mathcal E_1)\geq1-\delta_1,\qquad
\mathbb P(\mathcal E_2)\geq1-\delta_2.
\]
Applying the union bound, the joint event $\mathcal E=\mathcal E_1\cap\mathcal E_2$ holds with probability at least $1-\delta$. In the following, we condition on the success of this joint event.

Utilizing the same logic as \eqref{eqn:C_0} in the proof of Lemma~\ref{lemma:C_0}, with $\delta_1=\delta/2$, the number of warm-up cycles is bounded by
\[
C_0^G=\widetilde{\mathcal O}\left(\frac{h_1}{(\Delta_{\min}^G)^2}\right).
\]

\textbf{Bounding $\alpha\text{-}R_{PExplore}$ and $\alpha\text{-}R_{SExplore}$.} During the Pure Exploration phase ($c\leq C_0^G$), we play $d$ basis arms. For any arm $b_k(S_k)$ in Assumption~\ref{ass:arms during exploration}, the regret is bounded by
\[
\begin{aligned}
\alpha\theta_*^\top\x^\star(S^\star)
-\theta_*^\top b_k(S_k)&\leq \theta_*^\top\x^\star(S^\star) -\theta_*^\top b_k(S_k)\\
&\leq
2L_{\max}\|\theta_*\|_2.
\end{aligned}
\]
Here, we use $\alpha\in(0,1]$ and the fact that the origin is feasible, so the optimal expected reward is nonnegative. Therefore, the pure exploration $\alpha$-regret is bounded by
\begin{equation}
\label{eqn:pure exploration regret_general}
\begin{aligned}
\alpha\text{-}R_{PExplore}&\leq C_0^G\cdot d\cdot 2L_{\max}\|\theta_*\|_2\\
&= \widetilde{\mathcal O} \left( \frac{ dL_{\max}\|\theta_*\|_2h_1 }{ (\Delta_{\min}^G)^2 }
\right) \triangleq C_{PExplore}^G.
\end{aligned}
\end{equation}

During the stable phase ($c\geq C_0^G+1$), exploration is restricted to the $H$-dimensional subspace. The algorithm plays only $H$ restricted basis actions $e_k$ per cycle. Since $\alpha\in(0,1]$, the restricted exploration $\alpha$-regret over $K$ cycles is bounded by
\begin{align}
\alpha\text{-}R_{SExplore}&\leq \sum_{c=C_0^G+1}^{K} H\cdot2L_{\max}\|\theta_*\|_2
\nonumber\\
&\leq 2KH L_{\max}\|\theta_*\|_2.
\label{eqn:total exploration alpha regret}
\end{align}

\textbf{Bounding $\alpha\text{-}R_{SExploit}$.} We first derive a general bound for the exploitation regret on the identified greedy support $S^G$. Let $\x^G(S^G)\in\argmax_{\x\in\X,\ \supp(\x)\subseteq S^G} \theta_*^\top\x$ be the optimal continuous action on the greedy support under the true parameter $\theta_*$. For a parameter estimate $\widetilde\theta$, let $\x_{\widetilde\theta}^G(S^G) \in\argmax_{\x\in\X,\ \supp(\x)\subseteq S^G} \widetilde\theta^\top\x$ be the corresponding exploitation action. By the Greedy approximation guarantee in Lemma~\ref{lemma:approximation ratio expression}, $\theta_*^\top\x^G(S^G)\geq\alpha\theta_*^\top\x^\star(S^\star)$, where $\alpha=1-e^{-\gamma}$.

Define the instantaneous $\alpha$-regret $\alpha\text{-}r(\widetilde\theta)=\alpha\theta_*^\top\x^\star(S^\star)-\theta_*^\top\x_{\widetilde\theta}^G(S^G)$. Adding and subtracting the cross terms
$\widetilde\theta^\top\x^G(S^G)$ and $\widetilde\theta^\top\x_{\widetilde\theta}^G(S^G)$,
we obtain
\begin{align*}
\alpha\text{-}r(\widetilde\theta)&\leq\theta_*^\top\x^G(S^G)
-\theta_*^\top\x_{\widetilde\theta}^G(S^G)\\
&=(\theta_*-\widetilde\theta)^\top\x^G(S^G)+\widetilde\theta^\top
\left(\x^G(S^G)-\x_{\widetilde\theta}^G(S^G)\right)+(\widetilde\theta-\theta_*)^\top\x_{\widetilde\theta}^G(S^G).
\end{align*}

By definition, $\x_{\widetilde\theta}^G(S^G)$ maximizes the inner product with $\widetilde\theta$ over the restricted action set $\X(S^G)$. Therefore,
\[
\widetilde\theta^\top\left(\x^G(S^G)-\x_{\widetilde\theta}^G(S^G)\right)
\leq0.
\]
Consequently,
\begin{align*}
\alpha\text{-}r(\widetilde\theta)&\leq(\theta_*-\widetilde\theta)^\top\x^G(S^G)+
(\widetilde\theta-\theta_*)^\top
\x_{\widetilde\theta}^G(S^G)\\
&=(\widetilde\theta-\theta_*)^\top\left(\x_{\widetilde\theta}^G(S^G)-\x^G(S^G)
\right).
\end{align*}

Since both actions are supported on $S^G$, their difference is also supported on $S^G$. Hence,
\begin{align*}
\alpha\text{-}r(\widetilde\theta)\leq
\|\widetilde\theta(S^G)-\theta_*(S^G)\|_2\cdot
\|\x_{\widetilde\theta}^G(S^G)-\x^G(S^G)\|_2.
\end{align*}

By Assumption~\ref{ass:smooth_restricted_optimizer}, the restricted optimal-action mapping satisfies
\[
\begin{aligned}
&\|\x_{\widetilde\theta}^G(S^G)-\x^G(S^G)\|_2\leq
L_{\mathrm{grad}}\left\|\frac{\widetilde\theta(S^G)}{\|\widetilde\theta(S^G)\|_2}-\frac{\theta_*^G}{\|\theta_*^G\|_2}\right\|_2.
\end{aligned}
\]

Using the standard inequality for normalized vectors
\[
\left\|\frac{w}{\|w\|_2}-\frac{z}{\|z\|_2}\right\|_2\leq\frac{2\|w-z\|_2}{\|z\|_2},
\]
for nonzero $w,z$ \cite[Lemma 3.5]{rusmevichientong2010linearly}, and the fixed-direction tie-breaking convention when $w=0$, we obtain
\[
\begin{aligned}
\|\x_{\widetilde\theta}^G(S^G)-\x^G(S^G)\|_2\leq
\frac{2L_{\mathrm{grad}}}{\|\theta_*^G\|_2}\|\widetilde\theta(S^G)-\theta_*(S^G)\|_2.
\end{aligned}
\]

Combining the above inequalities and recalling $\zeta_G=\|\theta_*^G\|_2$, we arrive at
\begin{equation}
\label{eqn:arc}
\alpha\text{-}r(\widetilde\theta)\leq\frac{2L_{\mathrm{grad}}}{\zeta_G} \|\widetilde\theta(S^G)-\theta_*(S^G)\|_2^2.
\end{equation}

\textbf{Exploitation in the Support-Identification Cycle.}
In cycle $C_0^G$, the algorithm uses the full OLS estimator $\hat\theta_{C_0^G}$ for exploitation. On the support-consistency event, $S_{est}=S^G$. By \eqref{eqn:arc} and the event $\mathcal E_1$, the instantaneous $\alpha$-regret in this cycle satisfies
\[
\begin{aligned}
\alpha\text{-}r(C_0^G)&\leq\frac{2L_{\mathrm{grad}}}{\zeta_G}\|\hat\theta_{C_0^G}(S^G)-\theta_*^G\|_2^2\\
&\leq\frac{2L_{\mathrm{grad}}h_1}{C_0^G\zeta_G}\ln\left(\frac{2d(C_0^G)^2}{\delta_1}
\right).
\end{aligned}
\]
Since the exploitation action is played $C_0^G$ times during this cycle, its total contribution is
\begin{equation}
\begin{aligned}
C_0^G\alpha\text{-}r(C_0^G)\leq \frac{2L_{\mathrm{grad}}h_1}{\zeta_G}\ln\left(\frac{2d(C_0^G)^2}{\delta_1}
\right).
\end{aligned}
\label{eqn:initial_greedy_exploitation_regret}
\end{equation}

\textbf{Exploitation after Restricted Estimation.} 
For $c\geq C_0^G+1$, Lemma~\ref{lemma:general_consistency} ensures that $S_c^G=S^G$ with high probability. Let $m=c-C_0^G\geq1$ denote the index of the Restricted OLS update. By construction of the zero-padded estimator, the restriction of $\bar\theta_m^{(d)}$ to $S^G$
is exactly $\bar\theta_m^{(H)}$. Hence, by \eqref{eqn:arc} and the event $\mathcal E_2$, the instantaneous $\alpha$-regret satisfies
\[
\begin{aligned}
\alpha\text{-}r(c)&\leq\frac{2L_{\mathrm{grad}}}{\zeta_G}\|\bar\theta_m^{(H)}-\theta_*^G\|_2^2\\
&\leq\frac{2L_{\mathrm{grad}}h_H\ln(2Hm^2/\delta_2)}{m\zeta_G}\\
&=\frac{2L_{\mathrm{grad}}h_H\ln(4Hm^2/\delta)}{m\zeta_G}.
\end{aligned}
\]

Therefore, the stable exploitation $\alpha$-regret over cycles $c=C_0^G+1,\dots,K$ is bounded by
\begin{align}
\alpha\text{-}R_{SExploit}^{\mathrm{stable}}&=\sum_{c=C_0^G+1}^{K}c\cdot\alpha\text{-}r(c)
\nonumber\\
&\leq\sum_{m=1}^{K-C_0^G}(m+C_0^G)\frac{2L_{\mathrm{grad}}h_H\ln(4Hm^2/\delta)}{m\zeta_G}
\nonumber\\
&=\frac{2L_{\mathrm{grad}}h_H}{\zeta_G}\sum_{m=1}^{K-C_0^G}\left(1+\frac{C_0^G}{m}\right)\ln(4Hm^2/\delta)\nonumber\\
&\leq\frac{2L_{\mathrm{grad}}h_H\ln(4HK^2/\delta)}{\zeta_G}\left[K-C_0^G+C_0^G\left(1+\ln(K-C_0^G)\right)\right].
\label{eqn:stable exploitation alpha regret}
\end{align}

The last inequality follows from $\sum_{m=1}^{n}\frac{1}{m}
\leq1+\ln n$. By Corollary~\ref{coro:restricted_ols}, we have $h_H=\mathcal O(H)$. Utilizing $K=\Theta(\sqrt T)$, the dominant term in \eqref{eqn:stable exploitation alpha regret} is $\widetilde{\mathcal O}\left(\frac{L_{\mathrm{grad}}H\sqrt T}{\zeta_G}
\right)$. The remaining term involving $C_0^G$ is instance-dependent and is absorbed into the lower-order contribution $C_{\mathrm{gap}}^G$.

\textbf{Final Total Regret.}
Combining \eqref{eqn:pure exploration regret_general}, \eqref{eqn:total exploration alpha regret}, \eqref{eqn:initial_greedy_exploitation_regret}, and \eqref{eqn:stable exploitation alpha regret}, we obtain
\[
\begin{aligned}
\alpha\text{-}R_T&\leq C_{\mathrm{gap}}^G + \mathcal O(H\sqrt T L_{\max}\|\theta_*\|_2)+
\widetilde{\mathcal O}
\left(\frac{L_{\mathrm{grad}}H\sqrt T}{\zeta_G}\right)\\
&=\widetilde{\mathcal O}\left(H\sqrt T\left(L_{\max}\|\theta_*\|_2+\frac{L_{\mathrm{grad}}}{\zeta_G}\right)\right)+C_{\mathrm{gap}}^G,
\end{aligned}
\]
where $C_{\mathrm{gap}}^G$ collects the finite warm-up and lower-order contributions. This establishes the claimed $\alpha$-regret bound and completes the proof.
\end{proof}

\section{Omitted Proofs in Section~\ref{sec:ag_ts}}
\label{app:Omitted Proofs in Section ag_ts}

\subsection{Detailed Analysis of Phase 2 (Restricted FGTS)}

This section details the analysis of the Restricted Feel-Good TS regret referenced in Section~\ref{sec:ag_ts}. To start with, let $\mathcal{E}_{consist}$ be the event that Phase 1
successfully identifies the true greedy support ($S^G$) within a finite time. From Corollary~\ref{coro:agts_consistency}, by setting the confidence level in Phase 1 to $1-\delta_1$, we have $\mathbb{P}(\mathcal{E}_{consist})\ge1-\delta_1$, where $\delta_1$ will be explicit later. Conditioned on $\mathcal{E}_{consist}$, for $t>\tau$, Algorithm~\ref{alg:ag_ts} performs FGTS on the fixed support $S^G$. Conditioned on $\mathcal E_{consist}$, recall the definition of $\vartheta_*$ in Section~\ref{sec:ag_ts}. Since every action used in Phase 2 is supported on $S^G$, for every $\x\in\X(S^G)$, $\x^\top\theta_*=\x^\top\vartheta_*$ and $J(\theta_*;S^G)=J(\vartheta_*;S^G)$. Thus, $\vartheta_*$, rather than the ambient parameter $\theta_*$, is the effective true parameter of the Restricted FGTS subproblem. Throughout the remainder of the Phase-2 analysis, all expectations are understood conditionally on $\mathcal E_{consist}$ unless explicitly stated otherwise.

\textbf{Regret Decomposition.}
We first analyze Phase 2 relative to the greedy support benchmark. Define the instantaneous restricted regret
\begin{equation} \label{eqn:def_restricted_fgts_regret}
g_t\triangleq J(\vartheta_*;S^G) - \x_t^\top\vartheta_*.
\end{equation}
Since $\x_t\in\X(S^G)$ and $J(\vartheta_*;S^G)=\sup_{\x\in\X(S^G)}\x^\top\vartheta_*$, we have $g_t\ge0$ for every $t>\tau$. Furthermore, using $J(\vartheta_*;S^G)=J(\theta_*;S^G)$ and $\x_t^\top\vartheta_*=\x_t^\top\theta_*$, we have $g_t=J(\theta_*;S^G)-\x_t^\top\theta_*$.  Therefore, if we define $G_{\mathrm{FGTS}}\triangleq\sum_{t=\tau+1}^{T}g_t$, then $G_{\mathrm{FGTS}}=R_{FGTS}^G$. Note that the above definition and equality do not require
Assumption~\ref{ass:smooth_convex_inner_set}. Thus, the entire Restricted FGTS statistical analysis below applies to general compact action sets and is carried out relative to the
greedy support benchmark.

In standard TS analysis~\cite{abeille2017linear}, the regret is typically bounded via confidence ellipsoids and anti-concentration, which introduces suboptimal $\sqrt{H}$ variance-inflation factors. Instead, we employ the Decoupling Coefficient framework introduced by \cite{zhang2022feel}. We analyze the expected restricted regret conditioned on $\mathcal{E}_{consist}$. Recall that $\mathcal{F}_t=\sigma(\x_\tau,Y_\tau,\dots,\x_t,Y_t)$ is the filtration generated by the Phase-2 history. Let $\mathbb{E}_t$ denote the conditional expectation $\mathbb{E}_t[\cdot]=\mathbb{E}[\cdot\mid\mathcal{F}_{t-1}]$.

\textbf{The Decoupling Coefficient.}
The cornerstone of the Feel-Good Thompson Sampling analysis is bounding the instantaneous regret using the ``Decoupling Coefficient'' (DC) \cite[Definition~1]{zhang2022feel}. The DC serves as a mathematical bridge that decouples the action selected by the sampled parameter from the parameter itself, effectively reducing the regret analysis to a manageable least-squares estimation problem. According to \cite[Eq.~(10)]{zhang2022feel}, a contextual bandit
problem is ``linearly embeddable'' if its reward function admits an inner-product representation in a finite-dimensional feature space. Conditioned on $\mathcal E_{consist}$, the candidate parameters sampled by the Restricted Feel-Good posterior belong to $\Omega(S^G)$, and every Phase-2 action belongs to $\X(S^G)$. Moreover, the actual reward mean satisfies $\x^\top\theta_*=\x^\top\vartheta_*,\ \vartheta_*\in\Omega(S^G)$. Hence both the sampled models and the effective true parameter admit an $H$-dimensional representation on the coordinates indexed by $S^G$. Therefore, the Phase-2 subproblem is linearly embeddable with embedding
dimension at most $H$.

By \cite[Lemma~2]{zhang2022feel}, the decoupling coefficient $dc$ for a linearly embeddable class is bounded by its embedding dimension, yielding $dc\le H$ for our restricted subspace.
Let $\x(\theta)=\argmax_{\x\in\X(S^G)} \x^\top\theta$ denote the greedy action with respect to a parameter $\theta$. Notice that the optimistic evaluation function is exactly $J(\theta;S^G)=\x(\theta)^\top\theta$.

Applying the definition of the decoupling coefficient \cite[Definition~1]{zhang2022feel} to our restricted posterior distribution $p_t=p(\theta\mid\mathcal F_{t-1})$, we obtain the following fundamental inequality for any tuning scalar $\mu>0$:
\begin{equation}
\label{eqn:fgts_decoupling}
\mathbb{E}_{\tilde{\theta}_t\sim p_t}\left[J(\tilde{\theta}_t;S^G)-\x(\tilde{\theta}_t)^\top\vartheta_*\right]\le\frac{H}{4\mu}+\mu\mathbb{E}_{\x_t\sim\pi_{p_t}}\mathbb{E}_{\tilde{\theta}\sim p_t}\left[\left(\x_t^\top\tilde{\theta}-\x_t^\top\vartheta_*\right)^2\right],
\end{equation}
where $\pi_{p_t}$ is the marginal distribution of actions induced by first sampling $\tilde{\theta}_t\sim p_t$ and then executing $\x(\tilde{\theta}_t)$.

\textbf{Restricted Regret and Least Squares Separation.}
By definition, $\x_t=\x(\tilde{\theta}_t)$ maximizes the inner product with $\tilde{\theta}_t$. Using \eqref{eqn:def_restricted_fgts_regret}, we can rewrite the expected instantaneous restricted regret by injecting the sampled parameter $\tilde{\theta}_t$:
\begin{align}
\mathbb{E}_t[g_t]&=\mathbb{E}_{\tilde{\theta}_t}\left[J(\vartheta_*;S^G)-\x_t^\top\vartheta_*\right]\nonumber\\
&=\mathbb{E}_{\tilde{\theta}_t}\left[J(\vartheta_*;S^G)-J(\tilde{\theta}_t;S^G)\right]
+\mathbb{E}_{\tilde{\theta}_t}\left[J(\tilde{\theta}_t;S^G)-\x_t^\top\vartheta_*\right].
\label{eqn:q1}
\end{align}
Since the executed action $\x_t=\x(\tilde{\theta}_t)$ is a deterministic mapping of the sampled parameter $\tilde{\theta}_t$, the conditional expectation $\mathbb{E}_t[\cdot]$ reduces exactly to expectation over the sampling distribution of $\tilde{\theta}_t$. Applying the decoupling inequality \eqref{eqn:fgts_decoupling} to~\eqref{eqn:q1}, we obtain
\begin{equation}
\label{eqn:fgts_reg_decomposed}
\mathbb{E}_t[g_t]\le\frac{H}{4\mu}+\mu\mathbb{E}_t[LS_t]-\mathbb{E}_t[FG_t],
\end{equation}
where we denote the Least Squares error as $LS_t=\left(\x_t^\top\tilde{\theta}_t\x_t^\top\vartheta_*\right)^2$ and the Feel-Good optimism bonus as $FG_t=J(\tilde{\theta}_t;S^G)-J(\vartheta_*;S^G)$. Moreover, taking the expectation over the sampled parameter $\tilde{\theta}_t$ and its induced action $\x_t$ exactly constitutes
the conditional expectation given the history $\mathcal F_{t-1}$, which unifies the notation for the subsequent cumulative summation.

\textbf{Online Aggregation via Log-Partition Function.}
To bound the accumulated prediction error and optimism penalty, we rely on the specific construction of the Feel-Good posterior defined in~\eqref{eqn:fgts_posterior}. Following the online aggregation techniques for sub-Gaussian noise established in \cite[Theorem~1]{zhang2022feel}, we set the tuning parameter $\mu=0.25\eta/\lambda$. The analysis in \cite[Theorem~1]{zhang2022feel} bounds the accumulated terms $\mu\mathbb E_t[LS_t]-\mathbb E_t[FG_t]$ through a log-partition function.

To formalize this, define the likelihood difference between a candidate restricted parameter $\theta\in\Omega(S^G)$ and the effective true parameter $\vartheta_*$ at step $s$:
\[
\Delta L(\theta,\x_s,Y_s)\triangleq L(\theta,\x_s,Y_s) - L(\vartheta_*,\x_s,Y_s).
\]
Substituting our restricted likelihood from \eqref{eqn:fgts_loss}, this difference explicitly expands to
\begin{align}
\Delta L(\theta,\x_s,Y_s) = \eta\left[(\x_s^\top\theta-Y_s)^2-(\x_s^\top\vartheta_*-Y_s)^2\right] -\lambda\left[J(\theta;S^G)-J(\vartheta_*;S^G)\right].
\label{eqn:fgts_delta_L_expand}
\end{align}

By summing the decoupled regret bounds from \eqref{eqn:fgts_reg_decomposed} over the Restricted FGTS phase ($t=\tau+1,\dots,T$) and taking expectation over the Phase-2 randomness, the online
aggregation framework guarantees \cite[Theorem~1]{zhang2022feel}:
\begin{equation}
\label{eqn:fgts_summed}
\mathbb{E}\left[G_{\mathrm{FGTS}}\mid\mathcal E_{consist}\right]\le\frac{\lambda H(T-\tau)}{\eta}+\mathcal{O}(\lambda T)-\frac{1}{\lambda}\mathbb{E}[Z_T],
\end{equation}
where the $\mathcal{O}(\lambda T)$ term is inherited from the online aggregation bound, and $Z_T$ is the random log-partition function defined by
\begin{equation}
\label{eqn:fgts_ZT}
Z_T\triangleq\ln\mathbb{E}_{\theta\sim p_{0,S^G}}\Big[\exp\Big(-\sum_{s=\tau+1}^{T}\Delta L(\theta,\x_s,Y_s)\Big)\Big].
\end{equation}

\textbf{Bounding the Log-Partition Function $Z_T$.}
We next lower-bound the expected log-partition function $\mathbb{E}[Z_T]$. Conditioned on $\mathcal E_{consist}$, the posterior sampling space is restricted to $\Omega(S^G)$. The effective parameter $\vartheta_*=P_{S^G}\theta_*$ belongs to $\Omega(S^G)$ and induces exactly the same expected rewards as $\theta_*$ for every Phase-2 action. Therefore, the Restricted FGTS subproblem can be analyzed entirely within the $H$-dimensional parameter space $\Omega(S^G)$ using $\vartheta_*$ as the true comparator. Under the natural identification
$\Omega(S^G)\simeq\mathbb R^H$, suppose that $p_{0,S^G}$ is the isotropic Gaussian prior $\mathcal N(0,\rho^{-1}I_H)$. For notational simplicity, we write $p_0\equiv p_{0,S^G}$ throughout the remainder of this proof.

To establish the lower bound, we restrict the domain of integration to a tightly concentrated neighborhood around the effective parameter $\vartheta_*$ \cite{zhang2022feel}. Specifically, let $B_T=\left\{\theta\in\Omega(S^G):\norm{\theta-\vartheta_*}_2\le\frac{1}{T}\right\}$. By restricting the inner expectation in \eqref{eqn:fgts_ZT} to $B_T$ and leveraging the monotonicity of the logarithm, we obtain
\begin{align}
\mathbb{E}[Z_T]&\ge\mathbb{E}_{\mathcal H_T}\left[\ln\left(\int_{B_T}p_0(\theta)\exp\left(-\sum_{s=\tau+1}^{T}\Delta L(\theta,\x_s,Y_s)\right)d\theta\right)\right]\nonumber\\
&\ge\mathbb{E}_{\mathcal H_T}\left[\ln\left(p_0(B_T)\inf_{\theta\in B_T}\exp\left(-\sum_{s=\tau+1}^{T}\Delta L(\theta,\x_s,Y_s)\right)\right)\right]\nonumber\\
&=\ln p_0(B_T)-\mathbb{E}_{\mathcal H_T}\left[\sup_{\theta\in B_T}\sum_{s=\tau+1}^{T}\Delta L(\theta,\x_s,Y_s)\right]. \label{eqn:fgts_zt_decomposition}
\end{align}

\textit{Bounding the cumulative likelihood deviation.}
We now bound the two terms on the right-hand side of \eqref{eqn:fgts_zt_decomposition} separately. Consider first the second term. Recall the definition of $B_T$. For any $\theta\in B_T$, define $v\triangleq\theta-\vartheta_*,\ \norm{v}_2\le\frac1T$. For clarity, in this part of the proof we denote the reward noise by $\xi_s$, so that $Y_s=\x_s^\top\vartheta_*+\xi_s$, where $\{\xi_s\}$ is a conditionally $\sigma$-sub-Gaussian martingale difference sequence by Assumption~\ref{ass:noise}. Since $\theta=\vartheta_*+v$ and $Y_s=\x_s^\top\vartheta_*+\xi_s$,
for any $\theta\in B_T$, we have from \eqref{eqn:fgts_delta_L_expand} that
\begin{align}
\Delta L(\theta,\x_s,Y_s) &=\eta\left[(\x_s^\top\theta-Y_s)^2-(\x_s^\top\vartheta_*-Y_s)^2\right]-\lambda\left[J(\theta;S^G)-J(\vartheta_*;S^G)\right]\nonumber\\
&=\eta\left[(\x_s^\top v-\xi_s)^2-\xi_s^2\right]-\lambda\left[J(\theta;S^G)-J(\vartheta_*;S^G)\right]\nonumber\\
&=\eta\left[(\x_s^\top v)^2-2\xi_s\x_s^\top v\right]-\lambda\left[J(\theta;S^G)-J(\vartheta_*;S^G)\right].
\label{eqn:deltaL_expand_local}
\end{align}
By Assumption~\ref{ass:feasible set bound}, $\norm{\x_s}_2\le L_{\X}$, and hence $|\x_s^\top v|\le\norm{\x_s}_2\norm{v}_2\le\frac{L_{\X}}{T}$. Moreover, by the Lipschitz continuity of the restricted value function established in Proposition~\ref{prop:lipschitz_h},
\[
\left|J(\theta;S^G)-J(\vartheta_*;S^G)\right|\le L_{\X}\norm{\theta-\vartheta_*}_2\le\frac{L_{\X}}{T}.
\]
Therefore, summing \eqref{eqn:deltaL_expand_local} over $s=\tau+1,\dots,T$, we obtain
\begin{align*}
\sum_{s=\tau+1}^{T}\Delta L(\theta,\x_s,Y_s)&\le\eta\sum_{s=\tau+1}^{T}(\x_s^\top v)^2+
2\eta\big|v^\top\sum_{s=\tau+1}^{T}\xi_s\x_s\big| +\lambda\sum_{s=\tau+1}^{T}\left|
J(\theta;S^G)-J(\vartheta_*;S^G)\right|\\
&\le\frac{\eta L_{\X}^2}{T}+\frac{2\eta}{T}\Big\|\sum_{s=\tau+1}^{T}\xi_s\x_s\Big\|_2+\lambda L_{\X},
\end{align*}
where we used $T-\tau\le T$. Taking the supremum over $\theta\in B_T$ and then expectation gives
\begin{align}
\mathbb E_{\mathcal H_T}\Big[\sup_{\theta\in B_T}\sum_{s=\tau+1}^{T}\Delta L(\theta,\x_s,Y_s)
\Big]
\le\frac{\eta L_{\X}^2}{T}+\frac{2\eta}{T}\mathbb E\Big[\Big\|\sum_{s=\tau+1}^{T}\xi_s\x_s\Big\|_2\Big]+\lambda L_{\X}.
\label{eqn:deltaL_sup_bound_1}
\end{align}
Since $\{\xi_s\}$ is a conditionally zero-mean $\sigma$-sub-Gaussian martingale difference sequence and $\mathbf{1}\{s>\tau\}\x_s$ is predictable with respect to the history available before observing $\xi_s$, the cross terms vanish in expectation. Moreover, conditional $\sigma$-sub-Gaussianity implies $\mathbb E[\xi_s^2\mid\mathcal F_{s-1}]\le\sigma^2$. Consequently,
\begin{align*}
\mathbb E\left[\norm{\sum_{s=\tau+1}^{T}\xi_s\x_s}_2^2\right]=\mathbb E\left[\sum_{s=\tau+1}^{T}
\xi_s^2\norm{\x_s}_2^2\right] \le\sigma^2 L_{\X}^2(T-\tau)
\le\sigma^2L_{\X}^2 T.
\end{align*}
By Jensen's inequality,
\[
\mathbb E\Big[\Big\|\sum_{s=\tau+1}^{T}\xi_s\x_s\Big\|_2\big]\le\sigma L_{\X}\sqrt T.
\]
Substituting this bound into \eqref{eqn:deltaL_sup_bound_1} yields
\begin{align}
\mathbb E_{\mathcal H_T} \left[ \sup_{\theta\in B_T} \sum_{s=\tau+1}^{T} \Delta L(\theta,\x_s,Y_s) \right] \le \frac{\eta L_{\X}^2}{T} + \frac{2\eta\sigma L_{\X}}{\sqrt T}+ \lambda L_{\X}. \label{eqn:deltaL_sup_bound}
\end{align}
Therefore, by \eqref{eqn:fgts_zt_decomposition} and~\eqref{eqn:deltaL_sup_bound},
\begin{equation}
\label{eqn:fgts_zt_bound_inter}
\mathbb E[Z_T]\ge\ln p_0(B_T)-\frac{\eta L_{\X}^2}{T}-\frac{2\eta\sigma L_{\X}}{\sqrt T}-\lambda L_{\X}.
\end{equation}

\textit{Bounding the prior mass.}
We next lower-bound the prior probability mass $p_0(B_T)$. Under the natural identification
$\Omega(S^G)\simeq\mathbb R^H$, the restricted Gaussian prior has density
\[
p_0(\theta)=\left(\frac{\rho}{2\pi}\right)^{H/2}\exp\left(-\frac{\rho}{2}\norm{\theta}_2^2\right).
\]
For every $\theta\in B_T$, the triangle inequality gives
\[
\norm{\theta}_2\le\norm{\vartheta_*}_2+\norm{\theta-\vartheta_*}_2\le\norm{\vartheta_*}_2+\frac1T.
\]
Therefore,
\begin{equation}
\label{eqn:prior_density_lower}
p_0(\theta)\ge\left(\frac{\rho}{2\pi}\right)^{H/2}\exp\left[-\frac{\rho}{2}\left(\norm{\vartheta_*}_2+\frac1T\right)^2\right],\quad \forall\theta\in B_T.
\end{equation}
Since $B_T$ is an $H$-dimensional Euclidean ball of radius $1/T$, its volume is
\[
\operatorname{Vol}_H(B_T)=\frac{\pi^{H/2}}{\Gamma(H/2+1)}T^{-H}.
\]
Integrating the lower bound \eqref{eqn:prior_density_lower} over $B_T$ gives
\begin{align}
p_0(B_T)&\ge\left(\frac{\rho}{2\pi}\right)^{H/2}\exp\left[-\frac{\rho}{2}\left(\norm{\vartheta_*}_2+\frac1T\right)^2\right]\frac{\pi^{H/2}}{\Gamma(H/2+1)}T^{-H}\nonumber\\
&=\frac{(\rho/2)^{H/2}}{\Gamma(H/2+1)}T^{-H}\exp\left[-\frac{\rho}{2}\left(\norm{\vartheta_*}_2+\frac1T\right)^2\right]. \label{eqn:prior_mass_lower}
\end{align}
Taking logarithms in \eqref{eqn:prior_mass_lower} yields
\begin{align}
\ln p_0(B_T)&\ge-H\ln T-\ln\Gamma\left(\frac H2+1\right)+\frac H2\ln\left(\frac\rho2\right)
-\frac\rho2\left(\norm{\vartheta_*}_2+\frac1T\right)^2. \label{eqn:fgts_log_prior}
\end{align}
By Stirling's approximation, $\ln\Gamma\left(\frac H2+1\right)=\mathcal O(H\log H)$. For a fixed prior precision $\rho>0$, the term $\frac H2\ln\left(\frac\rho2\right)$ has magnitude $\mathcal O(H)$ and is therefore absorbed by the $\mathcal O(H\log H)$ term. Moreover,
\begin{align*}
\left(\norm{\vartheta_*}_2+\frac1T\right)^2=\norm{\vartheta_*}_2^2+\frac{2\norm{\vartheta_*}_2}{T}+\frac1{T^2}.
\end{align*}
Since $\vartheta_*=P_{S^G}\theta_*$ and Assumption~\ref{asm:param.set} gives $\norm{\theta_*}_2\le S$, we have $\norm{\vartheta_*}_2\le\norm{\theta_*}_2\le S$. Hence,
\[
\frac{\rho}{2}\left(\norm{\vartheta_*}_2+\frac1T\right)^2=\mathcal O\left(\rho\norm{\vartheta_*}_2^2+\rho\right).
\]
Combining the above estimates with \eqref{eqn:fgts_log_prior}, we obtain
\begin{equation}
\label{eqn:fgts_log_prior_simplified}
\ln p_0(B_T)\ge-\mathcal O\left(H\log(HT)+\rho\norm{\vartheta_*}_2^2+\rho\right).
\end{equation}
Finally, substituting \eqref{eqn:fgts_log_prior_simplified} into \eqref{eqn:fgts_zt_bound_inter} gives the explicit lower bound
\begin{align}
\mathbb E[Z_T]\ge-\mathcal O\left(H\log(HT)+\rho\norm{\vartheta_*}_2^2+\rho\right)-
\frac{\eta L_{\X}^2}{T}-\frac{2\eta\sigma L_{\X}}{\sqrt T}-\lambda L_{\X}.
\label{eqn:fgts_zt_bound_full}
\end{align}
When $\eta$, $\rho$, $\sigma$, $L_{\X}$, and $S$ are treated as problem-dependent constants, the terms independent of the dominant $H$- and $T$-dependence can be absorbed into the asymptotic notation. Thus,
\begin{equation}
\label{eqn:fgts_zt_bound}
\mathbb E[Z_T]\ge-\mathcal O\left(H\log(HT)+\rho\norm{\vartheta_*}_2^2\right).
\end{equation}
Using $\norm{\vartheta_*}_2\le S$ in \eqref{eqn:fgts_zt_bound}, we further obtain the
parameter-independent form
\begin{equation}
\label{eqn:fgts_zt_bound_S}
\mathbb E[Z_T]\ge-\mathcal O\left(H\log(HT)+\rho S^2\right).
\end{equation}

\textbf{Parameter Optimization and High-Probability Bound.}
We now substitute the log-partition lower bound \eqref{eqn:fgts_zt_bound_S} into the aggregated restricted-regret inequality \eqref{eqn:fgts_summed}. Using $T-\tau\le T$ and absorbing the lower-order terms in \eqref{eqn:fgts_zt_bound_full} into the same asymptotic expression, we obtain
\begin{equation}
\label{eqn:expected_G_general}
\mathbb E\left[G_{\mathrm{FGTS}}\mid\mathcal E_{consist}\right]\le\mathcal O\left(\lambda T\left(1+\frac H\eta\right)\right)+\frac{\mathcal O\left(H\log(HT)+\rho S^2\right)}{\lambda}.
\end{equation}
The right-hand side of \eqref{eqn:expected_G_general} is minimized by balancing the two terms, which gives 
\begin{equation}
\label{eqn:optimal_lambda_fgts}
\lambda=\Theta\left(\sqrt{\frac{H\log(HT)+\rho S^2}{T\left(1+H/\eta\right)}}\right).
\end{equation}
Substituting \eqref{eqn:optimal_lambda_fgts} back into \eqref{eqn:expected_G_general} yields
\begin{align}
\mathbb E\left[G_{\mathrm{FGTS}}\mid\mathcal E_{consist}\right]\le\mathcal O\left(\sqrt{T\left(1+\frac H\eta\right)\left(H\log(HT)+\rho S^2\right)}\right)\triangleq B_{\mathrm{FGTS}}(T). \label{eqn:expected_G_bound}
\end{align}
When $\eta$, $\rho$, and $S$ are treated as problem-dependent constants, we have $1+\frac H\eta=\mathcal O(H)$ and $H\log(HT)+\rho S^2=\tilde{\mathcal O}(H)$. Consequently,
\begin{equation}
\label{eqn:expected_G_simplified}
B_{\mathrm{FGTS}}(T)=\tilde{\mathcal O}\left(H\sqrt T\right).
\end{equation}
We next convert the above conditional expected bound into a high-probability guarantee. Importantly, by \eqref{eqn:def_restricted_fgts_regret}, each $g_t$ is nonnegative, and therefore $G_{\mathrm{FGTS}}\ge0$. Let $\delta_2\in(0,1)$. By Markov's inequality, conditioned on $\mathcal E_{consist}$,
\begin{align}
\mathbb P\left(G_{\mathrm{FGTS}}>\frac{B_{\mathrm{FGTS}}(T)}{\delta_2}\ \middle|\ \mathcal E_{consist}\right)\le\frac{\delta_2\mathbb E[G_{\mathrm{FGTS}}\mid\mathcal E_{consist}]}{B_{\mathrm{FGTS}}(T)}
\le\delta_2. \label{eqn:fgts_markov}
\end{align}
Hence, conditioned on $\mathcal E_{consist}$, with probability at least $1-\delta_2$, $G_{\mathrm{FGTS}}\le B_{\mathrm{FGTS}}(T)/\delta_2$.
Since $R_{FGTS}^G=G_{\mathrm{FGTS}}$, we obtain
\begin{equation}
\label{eqn:R_FGTS}
R_{FGTS}^G\le\frac{B_{\mathrm{FGTS}}(T)}{\delta_2
}
\end{equation}
with conditional probability at least $1-\delta_2$.

To ensure that the entire algorithmic guarantee holds with probability at least $1-\delta$, we divide the failure probability budget between the two phases as $\delta_1=\frac{\delta}{2},\ \delta_2=\frac{\delta}{2}$. Invoking Corollary~\ref{coro:agts_consistency} with confidence parameter $\delta_1$ gives $\mathbb P(\mathcal E_{consist})\ge1-\frac{\delta}{2}$. Conditioned on $\mathcal E_{consist}$, \eqref{eqn:fgts_markov} guarantees that the Phase-2 greedy-benchmark regret event \eqref{eqn:R_FGTS} holds with conditional probability at least $1-\delta/2$. Therefore, by the union bound, both events hold simultaneously with probability at least $1-\frac{\delta}{2}-\frac{\delta}{2}=1-\delta$. Using $\delta_2=\delta/2$ in \eqref{eqn:R_FGTS}, together with \eqref{eqn:expected_G_simplified}, gives
\[
R_{FGTS}^G \le \frac{2}{\delta} B_{\mathrm{FGTS}}(T).
\]
Hence, with probability at least $1-\delta$ jointly with the support-consistency event, we have
\[
R_{FGTS}^G\le\tilde{\mathcal O}(H\sqrt T).
\]

\textbf{Final Result.}
We now aggregate the regret relative to the greedy support benchmark across both Phase 1 and Phase 2 of Algorithm~\ref{alg:ag_ts}. During Phase 1, the adaptive exploration terminates at a random stopping time $\tau$. By Corollary~\ref{coro:agts_consistency}, when it is invoked with confidence parameter $\delta_1=\delta/2$, the stopping time satisfies $\tau=\tilde{\mathcal O}\left(d(\Delta_{\min}^G)^{-2}\right)$ with probability at least $1-\delta/2$. For every Phase-1 action $\x_t\in\X$, we have $J(\theta_*;S^G)-\x_t^\top\theta_*\le2\sup_{\x\in\X}|\x^\top\theta_*|$. Therefore, the cumulative greedy-benchmark regret incurred during the initial pure exploration phase constitutes a finite, horizon-independent term $R_{Explore}^G\le C_{gap}^G$, where $C_{gap}^G=\tilde{\mathcal O}\left(\frac{dL_{\max}\norm{\theta_*}_2h_1}{(\Delta_{\min}^G)^2}\right)$.

Conditioned on the same support-consistency event, the Phase-2 high-probability analysis above yields
\[
R_{FGTS}^G\le\frac{2}{\delta}B_{\mathrm{FGTS}}(T)=\tilde{\mathcal O}(H\sqrt T)
\]
with conditional probability at least $1-\delta/2$. Combining the two success events through the above union bound, with probability at least $1-\delta$, the total realized greedy-benchmark regret defined in \eqref{eqn:def_greedy_benchmark_regret} satisfies
\begin{align*}
R_T^G =R_{Explore}^G+R_{FGTS}^G\le\underbrace{\tilde{\mathcal O}\left(\frac{dL_{\max}\norm{\theta_*}_2h_1}{(\Delta_{\min}^G)^2}\right)}_{C_{gap}^G}+\tilde{\mathcal O}(H\sqrt T).
\end{align*}
This proves the general compact-set guarantee \eqref{eqn:AGFGTS_greedy_regret}.

Finally, suppose that Assumption~\ref{ass:smooth_convex_inner_set} also holds. Then \eqref{eqn:J_vartheta_ge} gives $J(\theta_*;S^G)\ge\alpha_{\mathrm{CH}}J(\theta_*)$, where
$\alpha_{\mathrm{CH}}$ is defined in~\eqref{eqn:alpha_CH}. Equivalently, by \eqref{eqn:alpha_regret_leq_greedy_regret}, $\alpha_{\mathrm{CH}}\text{-}R_T\le R_T^G$. Consequently, on the same event of probability at least $1-\delta$,
\[
\alpha_{\mathrm{CH}}\text{-}R_T\le C_{gap}^G+\tilde{\mathcal O}(H\sqrt T),
\]
which proves \eqref{eqn:AGFGTS_alpha_regret} and completes the proof of Theorem~\ref{thm:ag_ts_regret}.
\hfill$\blacksquare$

\section{Discussion on Time-varying Action Sets}
\label{app:time_varying_action_set}

As noted in Remark~\ref{remark:fix action set}, many practical applications involve time-varying feasibility constraints, where the action set $\X_t$ changes with time. Several ingredients of our framework remain meaningful in this setting. For example, if the action sets are uniformly bounded, $\sup_{t\in[T]}\sup_{\x\in\X_t}\norm{\x}_2\le L_{\max}$, then the restricted value functions $J_t(\theta;S)=\sup_{\x\in\X_t(S)}\theta^\top\x$ remain uniformly Lipschitz in $\theta$. Similarly, provided suitable exploration actions remain available, parameter estimation can still be carried out using observations collected under the changing
feasible sets. However, a fully arbitrary time-varying action-set model introduces an additional difficulty that is absent from the fixed-set analysis. In particular, the support returned by the Greedy procedure may become time-dependent: $S_t^G\in\operatorname{Greedy}\bigl(S\mapsto J_t(\theta_*;S)\bigr)$. Thus, even though the unknown parameter $\theta_*$ remains fixed, changes in the geometry of $\X_t$ can change the sparse coordinates that are relevant for optimal decision making.
This issue is fundamental for the algorithms developed in this paper. Both APSEE-G and AG-FGTS use an adaptive support-recovery phase and subsequently lock the identified support. Their fixed-support regret analyses therefore do not directly apply when $S_t^G$ changes arbitrarily over time. Even if each individual set $\X_t$ satisfies the smooth-boundary and restricted-optimizer regularity conditions required by APSEE-G, these conditions only control the regularity of the optimal action on a fixed support and do not guarantee that the optimal or greedy support itself remains unchanged.

Likewise, for AG-FGTS, the restricted posterior after support recovery is supported on a fixed subspace $\Omega(S_{est})$. If the appropriate greedy support subsequently changes, restricting all future actions to $S_{est}$ may incur an additional support-mismatch cost that is not captured by Theorem~\ref{thm:ag_ts_regret}. Consequently, the $\tilde{\mathcal O}(H\sqrt T)$ guarantee established for fixed compact action sets should not be interpreted as automatically extending to arbitrarily varying action sets.

A direct extension is possible under stronger stability conditions. For instance, if all $\X_t$ share a common greedy support $S^G$, the corresponding Greedy gaps are uniformly bounded away from zero, sufficiently informative exploration actions remain uniformly available, and the geometric regularity conditions required by the corresponding algorithm hold uniformly over time, then the fixed-support arguments can be adapted to the time-varying setting. For AG-FGTS, a uniform version of Assumption~\ref{ass:smooth_convex_inner_set}, with time-dependent smooth convex inner sets $\mathcal C_t$ and the associated submodularity ratios uniformly bounded away from zero, i.e., $\inf_{t\in[T]}\gamma_{\mathrm{CH},t}>0$, would convert the resulting greedy-benchmark regret guarantee into an $\alpha_{\mathrm{CH}}$-regret guarantee relative to the time-varying global sparse optimum. Under such conditions, the leading $\tilde{\mathcal O}(H\sqrt T)$ dependence can be preserved.
Without these stability assumptions, however, handling genuinely time-varying sparse supports requires a different algorithmic mechanism capable of detecting and adapting to support changes, rather than permanently committing to a single support. Developing such algorithms, together with regret bounds that depend explicitly on an appropriate measure of support variation or switching complexity, is an important direction for future work.

\section{Omitted Proofs in Section~\ref{sec:Lower Bound}}\label{app:Lower Bound}
\subsection{Proof of Theorem~\ref{thm:lower_bound_randomized}}
\begin{proof}
The proof utilizes \textbf{Yao's Minimax Principle} \cite{yao1977probabilistic}. The principle states that the minimax risk is lower-bounded by the Bayesian risk of the optimal deterministic algorithm against any prior distribution over the environment parameter $\theta$. In particular, for any prior distribution $Q_H$ supported on a subset of the admissible parameter class,
\begin{equation}
\label{eqn:yao_lower_bound}
    \inf_{\pi}\sup_{\theta\in\mathbb R^d}\mathbb E_{\theta}^{\pi}\left[ R_T\right]
    \ge \inf_{\pi_{\mathrm{det}}}\mathbb E_{\theta\sim Q_H}\left[ R_T(\pi_{\mathrm{det}},\theta) \right],
\end{equation}
where $\pi$ ranges over all randomized policies and $\pi_{\mathrm{det}}$ ranges over all deterministic policies. Therefore, it suffices to lower-bound the Bayes regret of an arbitrary deterministic policy under a suitably chosen prior $Q_H$.

\textit{Step 1: Construction of the Prior $Q_H$.}
We construct the prior distribution $Q_H$ over $\theta$ for a given sparsity level $H$.

\begin{sloppypar}
    \textbf{Sampling Parameter $\theta$.} Let $\Sigma_H=\{-1,+1\}^H$. For every $\sigma=(\sigma_1,\dots,\sigma_H)\in\Sigma_H$, define $\theta^\sigma=\Delta_H(\sigma_1,\dots,\sigma_H,0,\dots,0)^\top$, where $\Delta_H\in(0,1)$ is a scalar dependent on $H$ and will be specified explicitly later. Equivalently,
for indices $i\in\{1,\dots,H\}$, let $\theta_i^\sigma=\sigma_i\Delta_H$, where the coordinates $\sigma_i$ are independent Rademacher random variables taking values in $\{-1,+1\}$ with probability $1/2$ each; for indices $i\in\{H+1,\dots,d\}$, let $\theta_i^\sigma=0$.
\end{sloppypar}
Let $Q_H$ be the uniform distribution over the $2^H$ parameter vectors $\left\{\theta^\sigma:\sigma\in\Sigma_H\right\}$. This ensures that for every realization, $\theta^\sigma$ is supported on the first $H$ coordinates and $\|\theta^\sigma\|_0=H$. Notice that restricting the prior to this $H$-sparse subclass does not impose an $H$-sparsity assumption on the unknown parameter in the original problem. For a minimax lower bound, it is sufficient to construct a hard prior supported on any admissible subclass of problem instances.

\textit{Step 2: Optimal Action and Geometric Regret Lower Bound.}
Given the fixed sparsity level $H$, the optimal action must align with $\theta^\sigma$. Since $\norm{\theta^\sigma}_2=\sqrt{H}\Delta_H$, the vector
\begin{equation*}
    a_\sigma^\star=\frac{\theta^\sigma}{\|\theta^\sigma\|_2}=\frac{1}{\sqrt{H}}(\sigma_1,\dots,\sigma_H,0,\dots,0)^\top
\end{equation*}
is a valid action in $\X_H$. The optimal expected reward is $\mu^\star=\left\langle a_\sigma^\star,\theta^\sigma\right\rangle=\sqrt{H}\Delta_H$. Let $A_t\in\X_H$ be the action selected by an arbitrary deterministic policy $\pi$ at round $t$. For convenience, define $u_\sigma\triangleq a_\sigma^\star$. The regret of the policy against a fixed parameter $\theta^\sigma$ can be written as
\begin{align*}
R_T(\pi,\theta^\sigma,H)=\sum_{t=1}^T\left(\left\langle u_\sigma,\theta^\sigma\right\rangle-\left\langle A_t,\theta^\sigma\right\rangle\right)=\Delta_H\sum_{t=1}^T\sum_{i=1}^H\left(\frac{1}{\sqrt H}-A_{t,i}\sigma_i\right).
\end{align*}
Taking the expectation over the prior $Q_H$, we use the linearity of expectation:
\begin{equation}
\label{eq:bayes_regret_total}
    \mathbb E_{Q_H}[R_T]=\frac{1}{2^H}\sum_{\sigma\in\Sigma_H}\mathbb E_{\sigma}\left[ R_T(\pi,\theta^\sigma,H)\right],
\end{equation}
where $\mathbb E_\sigma$ denotes expectation under the environment parameter $\theta^\sigma$.

We next derive a geometric lower bound on the instantaneous regret. For each round $t$,
\begin{align}
r_t(\sigma)&=\left\langle u_\sigma-A_t,\theta^\sigma\right\rangle=\Delta_H\sqrt H\left(1-u_\sigma^\top A_t\right).\label{eqn:instant_lower_regret}
\end{align}
Since $\norm{u_\sigma}_2=1$ and $\norm{A_t}_2\le1$, we have
\begin{align*}
1-u_\sigma^\top A_t=\frac12\norm{u_\sigma-A_t}_2^2+\frac12\left(1-\norm{A_t}_2^2\right)\ge\frac12\norm{u_\sigma-A_t}_2^2.
\end{align*}
Substituting this inequality into \eqref{eqn:instant_lower_regret} gives
\begin{equation}
\label{eqn:geometric_regret_lower}
r_t(\sigma)\ge\frac{\Delta_H\sqrt H}{2}\norm{u_\sigma-A_t}_2^2.
\end{equation}

\textit{Step 3: Reduction to Hypothesis Testing.}
We now relate the geometric regret lower bound to a collection of binary hypothesis testing problems. For each coordinate $i\in\{1,\dots,H\}$, the learner must infer the unknown sign $\sigma_i\in\{-1,+1\}$. Let $\mathcal H_{t-1}=(A_1,Y_1,\dots,A_{t-1},Y_{t-1})$ denote the interaction history available before choosing the action
at time $t$. Since the policy is deterministic, $A_t$ is a measurable function of $\mathcal H_{t-1}$. For each round $t$ and coordinate $i\in[H]$, define the test
\[
\widehat{\sigma}_{t,i}=\begin{cases}+1, & A_{t,i}\ge0,\\
-1, & A_{t,i}<0.
\end{cases}
\]
Thus, $\widehat{\sigma}_{t,i}$ is a valid decision rule based only on the information available before observing the reward at round $t$. Whenever
$\widehat{\sigma}_{t,i}\neq\sigma_i$, the $i$-th coordinate of $A_t$ has the wrong sign relative to the corresponding coordinate of the optimal action $u_\sigma$. Therefore,
\begin{equation}
\label{eqn:coordinate_distance_testing}
\left(A_{t,i}-\frac{\sigma_i}{\sqrt H}\right)^2\ge\frac1H\mathbb I\left\{\widehat{\sigma}_{t,i}\neq\sigma_i\right\}.
\end{equation}
Since $\norm{u_\sigma-A_t}_2^2\ge\sum_{i=1}^H\left(A_{t,i}-\frac{\sigma_i}{\sqrt H}\right)^2$, combining \eqref{eqn:geometric_regret_lower} and \eqref{eqn:coordinate_distance_testing} yields
\begin{align}
r_t(\sigma)\ge\frac{\Delta_H\sqrt H}{2}\sum_{i=1}^H\left(A_{t,i}-\frac{\sigma_i}{\sqrt H}\right)^2\ge\frac{\Delta_H}{2\sqrt H}\sum_{i=1}^H\mathbb I\left\{\widehat{\sigma}_{t,i}\neq\sigma_i\right\}.\label{eqn:regret_testing_relation}
\end{align}
Let $P_\sigma^{t-1}$ denote the probability distribution of the history $\mathcal H_{t-1}$ under parameter $\theta^\sigma$. For every $t\in[T]$ and $i\in[H]$, define the average testing error under the prior by
\begin{equation}
\label{eqn:def_testing_error}
e_{t,i}\triangleq\frac1{2^H}\sum_{\sigma\in\Sigma_H}P_\sigma^{t-1}\left(\widehat{\sigma}_{t,i}\neq\sigma_i\right).
\end{equation}

Taking expectations in \eqref{eqn:regret_testing_relation}, averaging over the uniform prior $Q_H$, and summing over $t=1,\dots,T$, the Bayes regret \eqref{eq:bayes_regret_total} satisfies
\begin{equation}
\label{eqn:bayes_regret_testing}
\mathbb E_{Q_H}[R_T]\ge\frac{\Delta_H}{2\sqrt H}\sum_{t=1}^T\sum_{i=1}^H e_{t,i}.
\end{equation}
For each $i\in[H]$ and $\sigma\in\Sigma_H$, let $\sigma^{(i)}$ denote the sign vector
obtained from $\sigma$ by flipping only the $i$-th coordinate. Pairing each $\sigma$ satisfying $\sigma_i=+1$ with $\sigma^{(i)}$, the average testing error can be written as
\begin{align*}
e_{t,i}&=\frac1{2^H}\sum_{\substack{\sigma\in\Sigma_H\\\sigma_i=+1}}\Big[P_\sigma^{t-1}\left(\widehat{\sigma}_{t,i}=-1\right)+P_{\sigma^{(i)}}^{t-1}\left(\widehat{\sigma}_{t,i}=+1\right)\Big].
\end{align*}
By the Bretagnolle--Huber Inequality \cite[Theorem~14.2]{lattimore2020bandit}, for every such pair,
\begin{align*}
P_\sigma^{t-1}\left(\widehat{\sigma}_{t,i}=-1\right)+P_{\sigma^{(i)}}^{t-1}\left(\widehat{\sigma}_{t,i}=+1\right)
\ge\frac12\exp\left(-KL\left(P_\sigma^{t-1}\|P_{\sigma^{(i)}}^{t-1}\right)\right).
\end{align*}
Consequently,
\begin{equation}
\label{eqn:testing_error_KL}
e_{t,i}\ge\frac1{2^{H+1}}\sum_{\substack{\sigma\in\Sigma_H\\\sigma_i=+1}}\exp\left(-KL\left(P_\sigma^{t-1}\|P_{\sigma^{(i)}}^{t-1}\right)\right).
\end{equation}

\textit{Step 4: Divergence Decomposition and Optimization.}
We now explicitly bound the KL divergences appearing in \eqref{eqn:testing_error_KL}. According to the Divergence Decomposition Lemma \cite[Lemma~15.1]{lattimore2020bandit}, for a bandit problem with Gaussian noise of unit variance, the KL divergence between the history distributions induced by two parameter vectors $\theta$ and $\theta'$ satisfies
\[
KL(P_\theta^{t-1}\|P_{\theta'}^{t-1})=\frac12\mathbb E_\theta\left[\sum_{s=1}^{t-1}
\left\langle A_s,\theta-\theta'\right\rangle^2\right].
\]
For a fixed coordinate $i$ and sign vector $\sigma$, the parameters $\theta^\sigma$ and $\theta^{\sigma^{(i)}}$ differ only in coordinate $i$:
\[
\theta^\sigma-\theta^{\sigma^{(i)}}=2\sigma_i\Delta_H\,\mathbf e_i.
\]
Therefore,
\begin{equation}
\label{eqn:pairwise_KL}
KL\left(P_\sigma^{t-1}\|P_{\sigma^{(i)}}^{t-1}\right)=2\Delta_H^2\mathbb E_\sigma\left[\sum_{s=1}^{t-1}A_{s,i}^2\right].
\end{equation}

Instead of assigning a separate exploration budget to each coordinate under different environment measures, we average all pairwise divergences simultaneously. Define
\[
\overline K\triangleq\frac{1}{TH2^{H-1}}\sum_{t=1}^T\sum_{i=1}^H\sum_{\substack{\sigma\in\Sigma_H\\\sigma_i=+1}}KL\left(
P_\sigma^{t-1}\|P_{\sigma^{(i)}}^{t-1}\right).
\]
Using \eqref{eqn:pairwise_KL} and changing the order of summation, we obtain
\begin{align*}
\sum_{t=1}^T\sum_{i=1}^H\sum_{\substack{\sigma\in\Sigma_H\\\sigma_i=+1}}KL\left(P_\sigma^{t-1}\|P_{\sigma^{(i)}}^{t-1}\right)=2\Delta_H^2\sum_{\sigma\in\Sigma_H}\mathbb E_\sigma\left[\sum_{s=1}^{T-1}(T-s)\sum_{i:\sigma_i=+1}A_{s,i}^2\right].
\end{align*}
For every $\sigma\in\Sigma_H$ and every $s$, $\sum_{i:\sigma_i=+1}A_{s,i}^2\le\sum_{j=1}^dA_{s,j}^2=\norm{A_s}_2^2\le1$, because $A_s\in\X_H$. Hence,
\begin{align*}
\sum_{t=1}^T\sum_{i=1}^H\sum_{\substack{\sigma\in\Sigma_H,\sigma_i=+1}}KL\left(P_\sigma^{t-1}\|P_{\sigma^{(i)}}^{t-1}\right)
&\le 2\Delta_H^22^H\sum_{s=1}^{T-1}(T-s)\\
&=\Delta_H^22^HT(T-1)\le\Delta_H^22^HT^2.
\end{align*}
Dividing both sides by $TH2^{H-1}$ yields
\begin{equation}
\label{eqn:average_KL_bound}
\overline K \le\frac{2T\Delta_H^2}{H}.
\end{equation}
We now sum \eqref{eqn:testing_error_KL} over all $t\in[T]$ and $i\in[H]$. Since the function $x\mapsto e^{-x}$ is convex, Jensen's inequality gives
\begin{align}
\sum_{t=1}^T\sum_{i=1}^He_{t,i}&\ge\frac1{2^{H+1}}\sum_{t=1}^T\sum_{i=1}^H\sum_{\substack{\sigma\in\Sigma_H\\\sigma_i=+1}}\exp\left(-KL\left(P_\sigma^{t-1}\|P_{\sigma^{(i)}}^{t-1}\right)\right)\nonumber\\
&\ge\frac{TH}{4}\exp(-\overline K)\ge\frac{TH}{4}\exp\left(-\frac{2T\Delta_H^2}{H}\right), \label{eqn:sum_testing_errors}
\end{align}
where the last inequality follows from \eqref{eqn:average_KL_bound}. Substituting
\eqref{eqn:sum_testing_errors} into \eqref{eqn:bayes_regret_testing} gives the following lower bound on the expected Bayes regret:
\begin{equation}
\label{eqn:bayes_regret_gap}
\mathbb E_{Q_H}[R_T]\ge\frac{\sqrt H\,T\Delta_H
}{8}\exp\left(-\frac{2T\Delta_H^2}{H}\right).
\end{equation}

\textit{Step 5: Optimization of $\Delta_H$.}
We now proceed to the final stage of the proof: selecting the optimal gap parameter $\Delta_H$ for the given sparsity level $H$ to maximize the lower bound derived in Step 4. From \eqref{eqn:bayes_regret_gap}, we have
\[
\mathbb E_{Q_H}[R_T]\ge\frac{\sqrt H\,T\Delta_H}{8}\exp\left(-\frac{
2T\Delta_H^2}{H}\right).
\]
To make this bound as tight as possible, we choose $\Delta_H$ to maximize
\[
g(x)=x\exp\left(-\frac{2T}{H}x^2\right),\quad x>0.
\]
Its derivative is $g'(x)=\exp\left(-\frac{2T}{H}x^2\right)\left(1-\frac{4T}{H}x^2\right)$. Therefore, the maximum is achieved at
\begin{equation}
\label{eqn:optimal_lower_gap}
\Delta_H^\star=\sqrt{\frac{H}{4T}}.
\end{equation}
Since $T\ge d\ge H$, we have $\Delta_H^\star\le\frac12<1$, so this choice is admissible under the construction in Step 1. Substituting \eqref{eqn:optimal_lower_gap} back into the regret bound \eqref{eqn:bayes_regret_gap}, we obtain
\begin{align*}
\mathbb E_{Q_H}[R_T]&\ge\frac{\sqrt H\,T}{8}\sqrt{\frac{H}{4T}}\exp\left(-\frac{2T}{H}\cdot\frac{H}{4T}\right)\\
&=\frac{\sqrt H\,T}{8}\cdot\frac{\sqrt H}{2\sqrt T}\cdot\exp\left(-\frac12\right)=\frac{H\sqrt T}{16\sqrt e}.
\end{align*}

Since this lower bound holds for every deterministic policy, Yao's Minimax Principle in \eqref{eqn:yao_lower_bound} implies
\[
\inf_{\pi}\sup_{\theta\in\mathbb R^d}\mathbb E_{\theta}^{\pi}\left[R_T\right]
\ge\frac{H\sqrt T}{16\sqrt e}=\Omega(H\sqrt T).
\]
This confirms that for any fixed sparsity level $H$, the minimax regret for the sparsifying stochastic linear bandit problem is lower bounded by $\Omega(H\sqrt T)$, which completes the proof of Theorem~\ref{thm:lower_bound_randomized}.
\end{proof}

\section{Additional Numerical Experiments}
\label{app:additional_experiments}

Firstly, to further verify Theorem~\ref{thm:regret_general}, we present the curve of the average $\alpha$-regret ($\alpha\text{-}R_T/T$) versus the time horizon $T$ over a long timescale under the ellipsoid case in Section~\ref{sec:numerical experiments}. To better display as many variation trends as possible, we set the time horizon to $T=2\times 10^5$ and make all other settings remain unchanged. As shown in Figure~\ref{fig:ellipsoid pro}, the average $\alpha$-regret curves exhibit a slow descent over time, which aligns with $\lim_{T \to \infty} R_T/T = 0$ implied by our sub-linear regret bounds in Theorem~\ref{thm:regret_general}.
\begin{figure}[!htbp]
    \centering
    \includegraphics[width=0.45\linewidth]{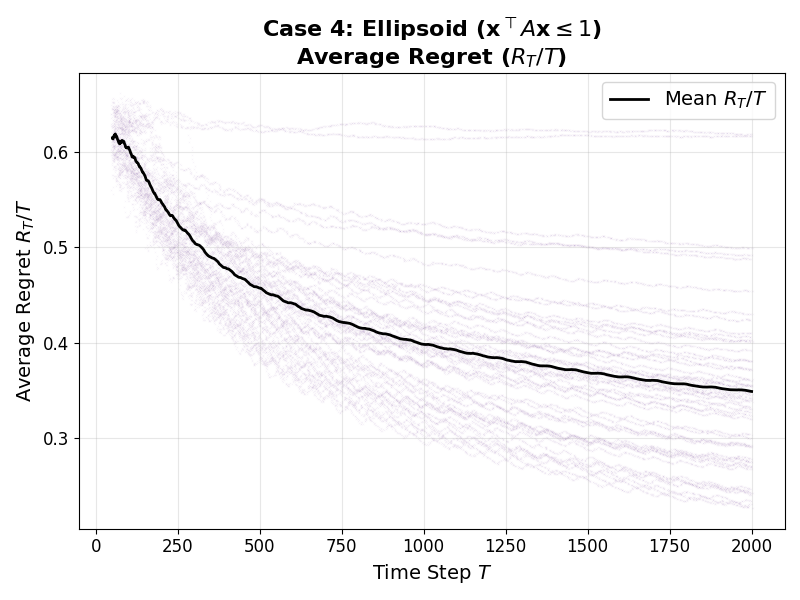}
    \caption{$\alpha\text{-}R_T/T$ vs. $T$ under the ellipsoid case. }
    \label{fig:ellipsoid pro}
\end{figure}
To proceed, we now present supplementary experimental results validating our framework on three other representative geometries: (1) Euclidean action sets where exact sparse solutions are tractable; (2) smooth convex action sets ($\ell_{1.5}$-ball) using the greedy approximation; and (3) general compact action sets ($\ell_1$-ball) with the adapted exploitation schedule. These experiments share the same simulation setup described in Section~\ref{sec:numerical experiments}. We also conduct comparative experiments between the conference and journal versions of APSEE to evaluate the performance improvement brought by restricted exploration under different ambient dimensions $d$ and sparsity levels $H$. Additionally, to further demonstrate the practical utility and versatility of our algorithms under concrete real-world constraints, we evaluate our framework on a personalized fashion recommendation task utilizing the Open Bandit Dataset (OBD) \cite{saito2020open} collected from the ZOZOTOWN platform. For all these additional experiments, we have included the average ($\alpha$)-regret ($(\alpha\text{-})R_T/T$) plots with a long horizon of $T=2\times 10^6$.

\begin{figure}[!htbp]
    \centering
    \begin{subfigure}{\textwidth}
        \centering
        \includegraphics[width=0.99\linewidth]{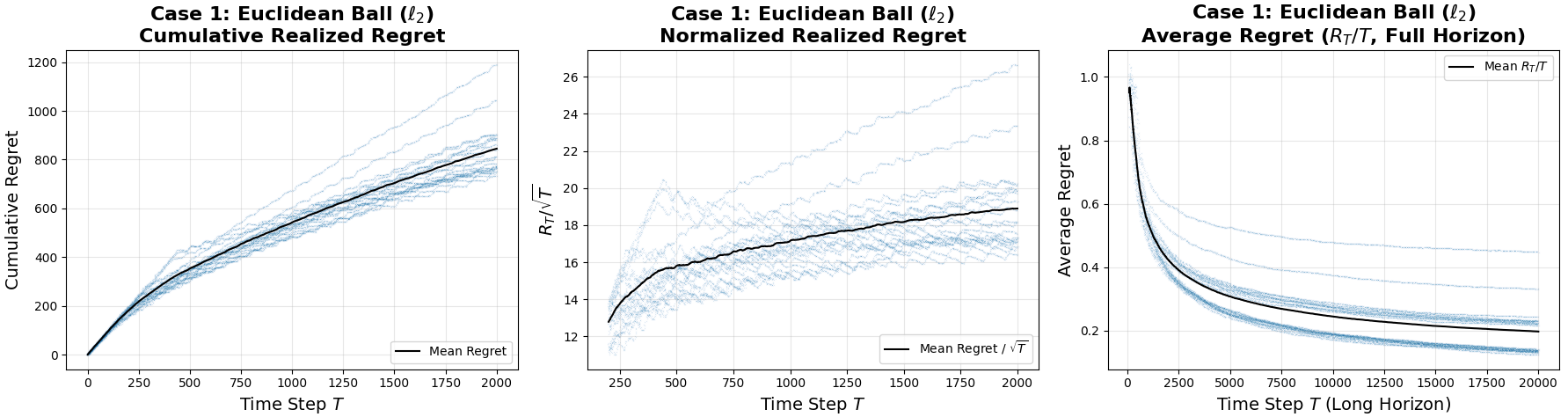}
        \caption{Case 1: $\ell_2$-Ball ($\tilde{\CO}(\sqrt{T})$ Regret)}
        \label{fig:case1_app}
    \end{subfigure}
    \vspace{0.5em}
    
    \begin{subfigure}{\textwidth}
        \centering
        \includegraphics[width=0.99\linewidth]{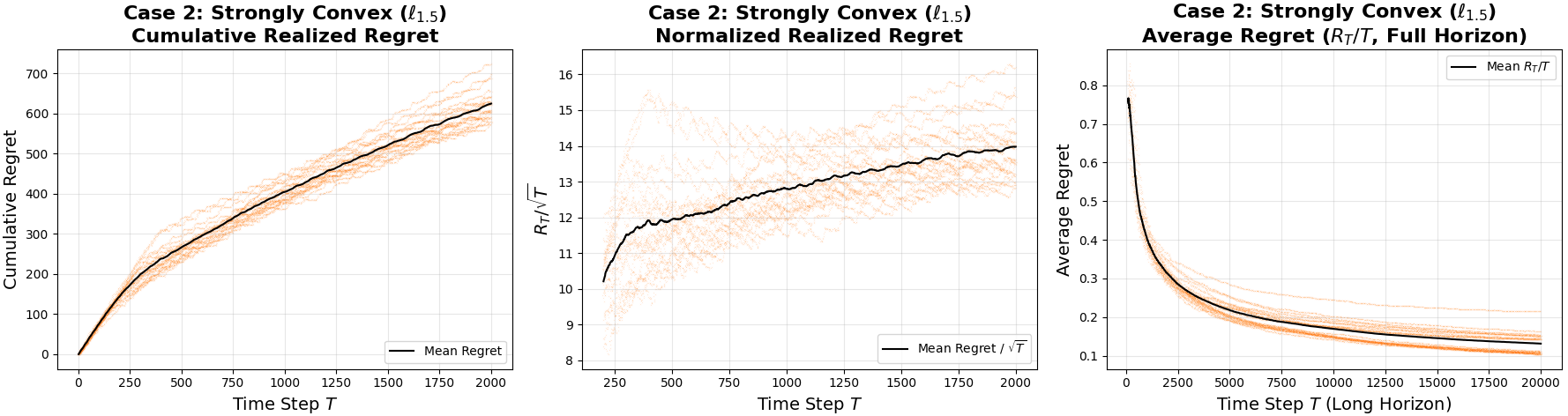}
        \caption{Case 2: $\ell_{1.5}$-Ball ($\tilde{\CO}(\sqrt{T})$ $\alpha$-Regret)}
        \label{fig:case2_app}
    \end{subfigure}
    \vspace{0.5em}
    
    \begin{subfigure}{\textwidth}
        \centering
        \includegraphics[width=0.99\linewidth]{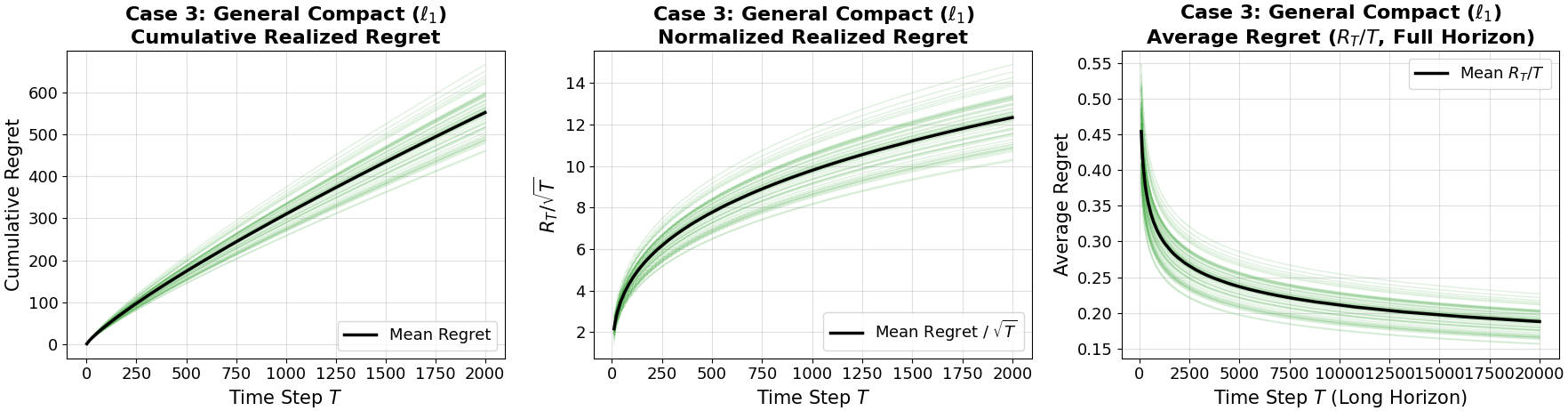}
        \caption{Case 3: $\ell_1$-Ball ($\tilde{\CO}(\sqrt{T})$ $\alpha$-Regret)}
        \label{fig:case3_app}
    \end{subfigure}
    
    \caption{Additional numerical simulation results.}
    \label{fig:appendix_experiments}
\end{figure}

\subsection{Case 1: Euclidean Action Sets ($\ell_2$-Ball)}
Consider the action set $\X$ is the unit Euclidean ball $\{\x \in \R^d : \norm{\x}_2 \le 1\}$. In this setting, we implement the \textbf{APSEE} algorithm (Algorithm~\ref{alg:adaptive_PEE}). The exact oracle efficiently selects the top-$H$ indices of the estimated parameter $\hat{\theta}_c$ based on absolute magnitude.
Figure~\ref{fig:appendix_experiments} (a) presents the results. The \textit{Left} plot illustrates the cumulative regret, where we observe that the regret grows sub-linearly. The \textit{Middle} plot shows the regret normalized by $\sqrt{T}$, where the curve approaches a logarithmic growth trend, corroborating the $\tilde{\mathcal{O}}(\sqrt{T})$ bound established in Theorem~\ref{thm:regret_adaptive}. The \textit{Right} plot displays the average regret $R_T/T$ over $T=2\times 10^5$. The curve monotonically decreases towards zero, confirming the algorithm's convergence and the validity of the sub-linear regret bound in Theorem~\ref{thm:regret_adaptive}.
\subsection{Case 2: Smooth Convex Action Sets ($\ell_{1.5}$-Ball)}
Consider an $\ell_p$-ball with $p=1.5$, i.e., $\X = \{\x \in \R^d : \norm{\x}_{1.5} \le 1\}$. This set has a continuously differentiable boundary and satisfies the restricted-optimizer regularity condition required by APSEE-G, but does not admit a trivial sorting-based solution for the optimal sparse support. Thus, we run \textbf{APSEE-G} (Algorithm~\ref{alg:adaptive_general_PEE}). Note that Algorithm~\ref{alg:adaptive_general_PEE} employs the Greedy algorithm to select the support set $S_{est}$ by maximizing the marginal gain of the function $h(S; \hat{\theta}_c) = \norm{\hat{\theta}_c(S)}_{q}$ (where $q=3$ is the dual norm of $p=1.5$).
Figure~\ref{fig:appendix_experiments} (b) shows the cumulative $\alpha$-regret relative to the greedy solution. The \textit{Left} plot demonstrates sub-linear growth. The \textit{Middle} plot confirms the $\tilde{\mathcal{O}}(\sqrt{T})$ $\alpha$-regret bound, as the normalized $\alpha$-regret ($\alpha\text{-}R_T/\sqrt{T}$) shows a logarithmic growth trend. The \textit{Right} plot of average $\alpha$-regret ($\alpha\text{-}R_T/T$) over $T=2\times 10^5$ shows a downward trend, confirming Theorem~\ref{thm:regret_general}.

\subsection{Case 3: General Compact Action Sets ($\ell_1$-Ball)}
Now, we consider the cross-polytope (unit $\ell_1$-ball), $\X = \{\x \in \R^d : \norm{\x}_1 \le 1\}$. This set is compact and convex but has a nonsmooth boundary with flat faces. Consequently, it does not satisfy the smoothness condition or the restricted-optimizer regularity condition required by APSEE-G. Hence, we employ the \textbf{AG-FGTS} algorithm (Algorithm~\ref{alg:ag_ts}), which is specifically designed to handle general compact action sets.
The results are presented in Figure~\ref{fig:appendix_experiments} (c), where the faint green lines represent individual experiment trajectories and the solid black line denotes the mean regret. The \textit{Left} plot shows the cumulative $1$-regret over 2,000 steps, demonstrating a clear and stable sub-linear growth trend. To validate the enhanced theoretical bound of $\tilde{\mathcal{O}}(\sqrt{T})$ established in Theorem~\ref{thm:ag_ts_regret}, the \textit{Middle} plot shows the regret normalized by $\sqrt{T}$ (i.e., $R_T / \sqrt{T}$). The normalized curve rapidly stabilizes and presents a slowly increasing logarithmic-like trend, perfectly consistent with the theoretical derivations. Finally, the \textit{Right} plot confirms that the average regret $R_T/T$ steadily decays to zero over $T=2\times 10^5$. These empirical observations validate the necessity and effectiveness of the restricted Feel-Good Thompson Sampling mechanism for gracefully balancing the exploration-exploitation trade-off in general compact sets, thereby experimentally confirming Theorem~\ref{thm:ag_ts_regret}.

\subsection{Empirical Verification of the Restricted-Exploration Mechanism}
\label{app:restricted_exploration_comparison}

A key improvement of the journal version over the conference version is the \textit{Restricted Exploration} mechanism in APSEE. Conference APSEE continues to explore all $d$ dimensions after support recovery, whereas the revised algorithm locks the identified support and explores only its $H$ active coordinates. This change directly targets the stable phase: although both versions must initially identify the optimal support using ambient-dimensional observations, only the journal version avoids repeatedly paying for all $d$ coordinates thereafter. The revised regret analysis accordingly replaces the leading horizon-dependent exploration factor $d$ by $H$.

To isolate this effect, we compare two APSEE variants under the Euclidean-ball setting:
\begin{itemize}
    \item \textbf{Conference APSEE:} After support recovery, it continues to play $d$ exploration actions per cycle and updates a full-dimensional OLS estimator.
    \item \textbf{Restricted-Exploration APSEE:} It uses the same $d$-dimensional warm-up, but after the verification condition locks $S_{est}=S^\star$, it plays only $H$ restricted exploration actions and updates an $H$-dimensional OLS estimator.
\end{itemize}
Both variants share the initial exploration basis, support-verification rule, and exploitation schedule. Their only structural difference after support recovery is the dimension of exploration and estimation. Within each independent trial, they are evaluated on the same underlying parameter $\theta_*$ and coupled using an identical pre-generated reward-noise sequence. This common-random-number design reduces extraneous simulation variation, so differences after support locking can be attributed more directly to the exploration mechanism rather than to changing problem instances.

We use the unit Euclidean ball $\X=\{\x\in\mathbb R^d:\norm{\x}_2\le1\}$, for which the optimal $H$-sparse action selects the $H$ largest-magnitude coordinates of $\theta_*$ and normalizes the restricted vector. This geometry avoids additional approximation error from offline sparse optimization, so the comparison isolates the exploration mechanism rather than differences in action-selection quality. The same action geometry is used in both experimental groups. Rewards have Gaussian noise with standard deviation $0.5$. Each curve averages $20$ independent trials over $T=5\times10^4$ rounds. We report cumulative pseudo-regret $R_t=\sum_{s=1}^{t}[\theta_*^\top\x^\star-\theta_*^\top\x_s]$, which excludes realized observation noise, and normalized pseudo-regret $R_t/\sqrt t$ to examine growth rates. Thus, observation noise affects the estimated parameters and chosen actions but does not enter the regret calculation directly. This distinction ensures consistency with the theoretical pseudo-regret benchmark used throughout the paper.

\paragraph{Varying the ambient dimension $d$ with fixed sparsity $H$}
We fix $H=5$ and consider $d\in\{20,40,80,160\}$. To maintain comparable problem instances, we construct one master parameter in $\mathbb R^{160}$ and truncate it to the first $d$ coordinates. Its five largest-magnitude coordinates lie within the first $20$ dimensions, while the others have substantially smaller magnitudes. The construction preserves the five strongest coordinates across all four dimensions, while the remaining coefficients contribute weaker signals. Thus, increasing $d$ mainly adds nuisance coordinates and raises the full-dimensional exploration cost, without altering the identity of the target sparse support. This makes the dependence on $d$ easier to interpret.

\begin{figure}[!htbp]
    \centering
    \includegraphics[width=1.0\linewidth]{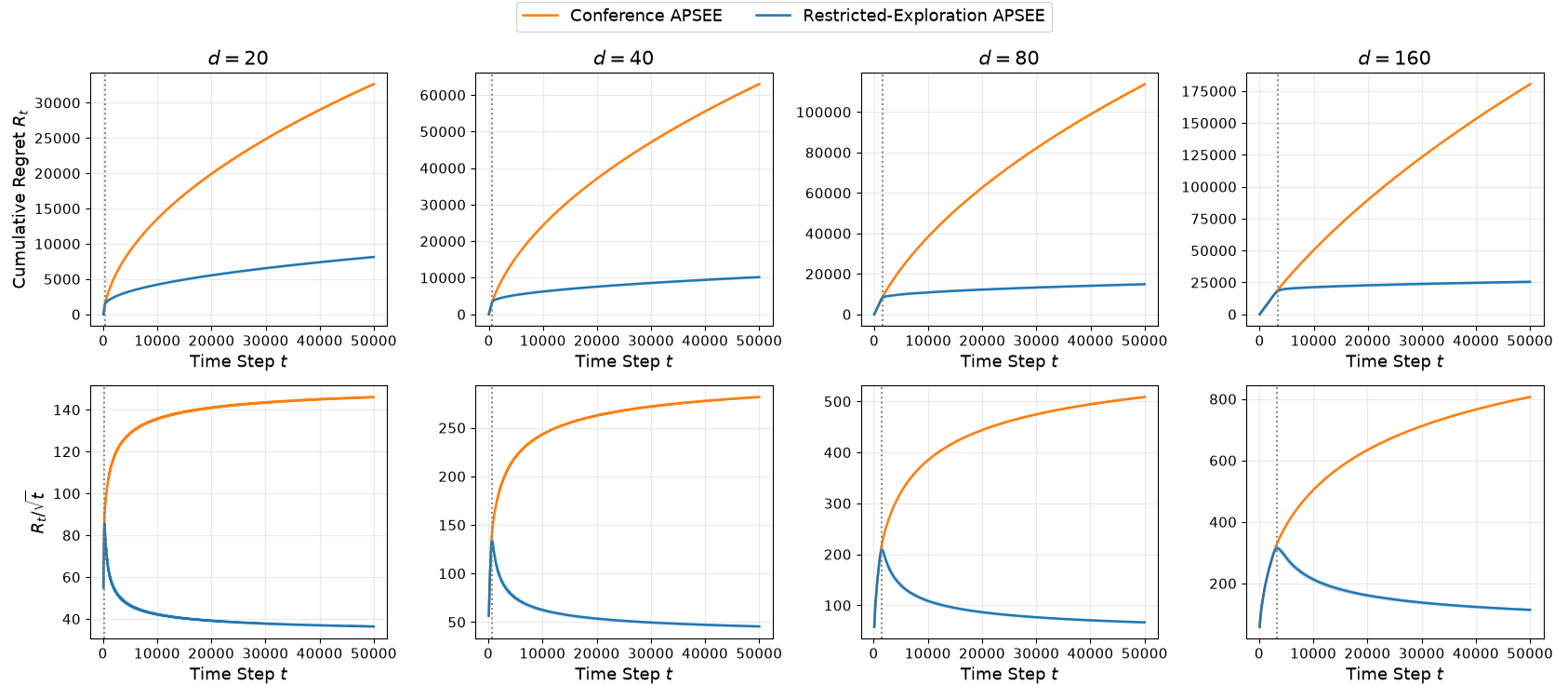}
    \caption{Comparison of Conference and Restricted-Exploration APSEE for fixed $H=5$ and $d\in\{20,40,80,160\}$. Upper panels show cumulative pseudo-regret $R_t$, and lower panels show $R_t/\sqrt t$. Vertical dotted lines mark the average support-locking time across trials.}
    \label{fig:restricted_exploration_vary_d}
\end{figure}
Conference APSEE incurs $d$ exploration rounds per stable-phase cycle, whereas Restricted-Exploration APSEE incurs only $H=5$ after support locking. We therefore expect their post-recovery regret growth to diverge increasingly as $d$ grows. This does not imply that the revised algorithm is independent of $d$: its warm-up stage still explores all ambient coordinates, and the finite support-recovery cost $C_{\mathrm{gap}}$ depends on $d$. The predicted improvement concerns the repeated, horizon-dependent cost after recovery, rather than identical regret trajectories from the beginning. In particular, a meaningful signature of the theoretical improvement is a weaker dependence on $d$ in the post-locking slope, not the disappearance of the initial dimension-dependent offset.

\begin{figure}[!htbp]
    \centering
    \includegraphics[width=1.0\linewidth]{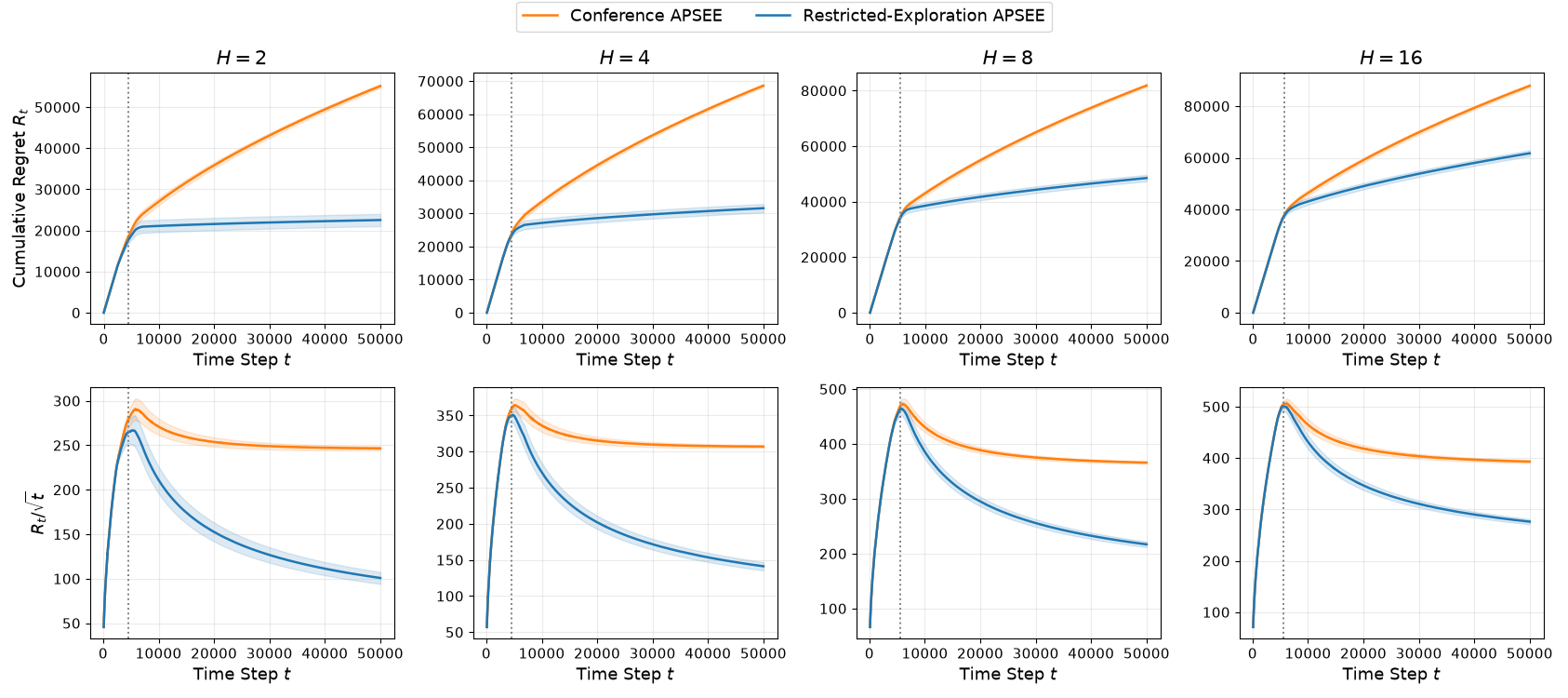}
    \caption{Comparison of both APSEE variants with fixed $d=40$ and $H\in\{2,4,8,16\}$. Upper and lower panels show $R_t$ and $R_t/\sqrt t$, respectively. Vertical dotted lines indicate the average support-locking time across trials.}
    \label{fig:restricted_exploration_vary_H}
\end{figure}

\paragraph{Varying the sparsity level $H$ with fixed ambient dimension $d$}
We next fix $d=40$ and vary $H\in\{2,4,8,16\}$ using the same underlying $\theta_*\in\mathbb R^{40}$. The coordinate magnitudes are chosen so that the gaps at ranks $2$, $4$, $8$, and $16$ are approximately equal. Using a single fixed parameter, rather than independently drawing new parameters for each $H$, also avoids confounding sparsity with changes in the underlying reward environment. The approximately matched signal gaps keep support-identification difficulty comparable and help isolate the effect of $H$ on the stable-phase cost.

\begin{figure}[!htbp]
    \centering
    \includegraphics[width=0.85\linewidth]{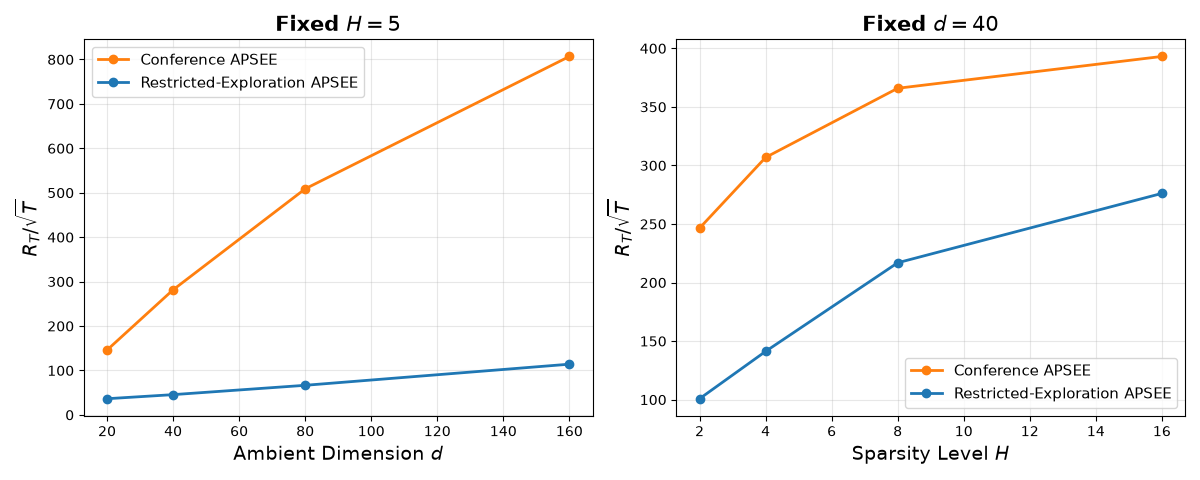}
    \caption{Terminal normalized pseudo-regret $R_T/\sqrt T$ for Conference and Restricted-Exploration APSEE, varying $d$ at $H=5$ (left) and varying $H$ at $d=40$ (right).}
    \label{fig:restricted_exploration_scaling}
\end{figure}
As $H$ increases, Restricted-Exploration APSEE uses more exploration actions per stable-phase cycle, while Conference APSEE always uses $d=40$. Thus, the advantage of restricted exploration should be most pronounced when $H\ll d$ and diminish as $H$ approaches $d$. The two experiments therefore provide complementary tests: varying $d$ at fixed $H$ probes sensitivity to unnecessary ambient coordinates, whereas varying $H$ at fixed $d$ probes sensitivity to the size of the identified subspace. In both cases, the main comparison concerns performance after support recovery. For a complementary comparison at the same horizon, we also plot the terminal normalized pseudo-regret $R_T/\sqrt T$ against $d$ and $H$. These summary plots highlight the contrasting scaling patterns of the two exploration strategies without relying only on the time-series trajectories.

Overall, these experiments directly examine the dimension reduction introduced by the revised APSEE. Its regret need not be entirely independent of $d$, because support identification retains a finite ambient-dimensional cost. Rather, after locking the support, the repeated exploration and exploitation costs are governed by the recovered $H$-dimensional subspace. The comparison with Conference APSEE is intended to reveal this transition from $d$-dependent to $H$-dependent stable-phase behavior.

\subsection{Real-world Application: Personalized Fashion Recommendation on ZOZOTOWN}
\label{subsec:video_recommendation}
To validate the practical efficacy of our proposed algorithms in realistic high-dimensional environments, we conduct experiments using the \textit{Open Bandit Dataset (OBD)} \cite{saito2020open}. Provided by ZOZO, Inc., Japan's largest fashion e-commerce company, OBD contains real-world logged bandit feedback and highly dimensional contextual features collected from the ZOZOTOWN platform. This renders it an exceptionally representative testbed for evaluating adaptive recommendations in fashion e-commerce.

{\bf Data Structure and Problem Mapping.} 
The OBD fundamentally consists of millions of user interaction logs, encompassing rich item features and user click behaviors. 
We extract $d=20$ distinct latent fashion attributes (e.g., brand affinity, color, price sensitivity) from the raw numerical features of the dataset via dimensionality reduction and normalization. 
For a specific user segment arriving at the platform, their underlying, global preferences toward these $d$ attributes are mathematically represented by the fixed but unknown parameter $\theta_* \in \R^d$. A positive component $(\theta_*)_i > 0$ indicates an affinity for attribute $i$. 
The action $\mathbf{x}_t \in \R^d$ designed by the platform represents the \textit{display intensity} (e.g., banner space allocation or promotional focus) of these $d$ attributes in a customized promotional layout, normalized to the interval $[0, 1]$. Thus, the action set is a $d$-dimensional hypercube (i.e., box constraint): $\X = \{ \mathbf{x} \in \R^d : 0 \le (\x)_i \le 1, \forall i \in [d] \}$. 
Furthermore, due to user interface (UI) limitations and the necessity to avoid cognitive overload on mobile screens, the promotional campaign can highlight at most $H=5$ distinct attributes simultaneously. This physical UI limitation perfectly aligns with our strict sparsity constraint $|\supp(\mathbf{x}_t)| \le H$.

{\bf Regret Calculation Methodology.} 
While the OBD contains the true collected rewards, they are logged passively under a distinct production policy. To construct a rigorous evaluation environment that reflects true user preferences, we process the binary click indicators alongside the $d$-dimensional feature vectors using Ridge regression. The learned weight vector from this regression explicitly serves as our ground-truth parameter $\theta_*$, effectively capturing the real-world correlation between item attributes and user engagement. 
The user's engagement (reward) is modeled as the inner product $Y_t = \theta_*^\top \mathbf{x}_t + \eta_t$. Highlighting a favored attribute ($(\theta_*)_i > 0$) with high intensity ($(\x)_i \to 1$) yields a high reward. The additive noise term $\eta_t$ encapsulates the inherent stochasticity in user behavior (e.g., accidental clicks). Consequently, the optimal baseline is evaluated as $\max_{\x \in \X, \|\x\|_0 \le H} \theta_*^\top \x$, grounding our empirical evaluation in realistic human feedback without requiring exhaustive oracle search over passive logs.

{\bf Simulation Setup and Algorithm Selection.}
Since the action set $\X = [0,1]^d$ is a convex compact polytope with nonsmooth boundaries and the origin lying on its boundary rather than in its interior, it does not satisfy the smooth convex action-set conditions required by APSEE-G. Consequently, we employ the \textbf{AG-FGTS} algorithm (Algorithm~\ref{alg:ag_ts}), which is specifically designed to handle general compact action sets.

Note that in this specific personalized recommendation setting, the Greedy strategy recovers the \textit{exact} optimal support ($\alpha=1$) due to the geometry of the box constraint $\X = [0, 1]^d$. We justify this exact optimality by explicitly evaluating the restricted value function. Consider the offline continuous subset selection problem:
\begin{equation}
    \max_{\mathbf{x} \in \mathbb{R}^d} \sum_{i=1}^d (\theta_*)_i (\x)_i \quad \text{s.t.} \quad \mathbf{x} \in [0, 1]^d, \ |\supp(\mathbf{x})| \le H.
\end{equation}
For any fixed candidate support set $S \subseteq [d]$ with $|S| \le H$, the constrained optimization problem completely decouples across coordinates. Since the objective is purely linear and the box constraints are separable, the optimal continuous action restricted to $S$, denoted as $\x^\star(S)$, is achieved by maximizing each coordinate independently:
\begin{equation}
    (\x^\star(S))_i = 
    \begin{cases} 
        1, & \text{if } i \in S \text{ and } (\theta_*)_i > 0, \\
        0, & \text{otherwise}.
    \end{cases}
\end{equation}
Substituting this closed-form optimal action back into the objective yields the maximal reward for a given support $S$, which effectively reduces the mixed continuous-discrete problem into a pure discrete set function $F(S)$:
\begin{equation}
    F(S) \triangleq \theta_*^\top \x^\star(S) = \sum_{i \in S} \max\big(0, (\theta_*)_i\big).
\end{equation}
Let $w_i = \max\big(0, (\theta_*)_i\big) \ge 0$. The set function $F(S) = \sum_{i \in S} w_i$ is strictly \textit{modular} \footnote{A fundamental case of submodular functions where the marginal gain of adding element $v$ to $S$ is exactly $w_v$, entirely independent of the existing elements in $S$, e.g., see \cite{nemhauser1978analysis} for more details.}. For maximizing a non-negative modular set function subject to a cardinality constraint $|S| \le H$, the Greedy algorithm---which iteratively selects the item $v$ with the largest marginal gain $w_v$---is algebraically equivalent to sorting the items by their non-negative weights $w_i$ in descending order and selecting the top $H$ elements. Consequently, the Greedy algorithm achieves the global optimum without any approximation loss, yielding an exact approximation ratio of $\alpha=1$ \cite{nemhauser1978analysis}. Thus, our benchmark is the true optimal strategy.

\begin{figure}[!ht]
    \centering
    \includegraphics[width=0.99\linewidth]{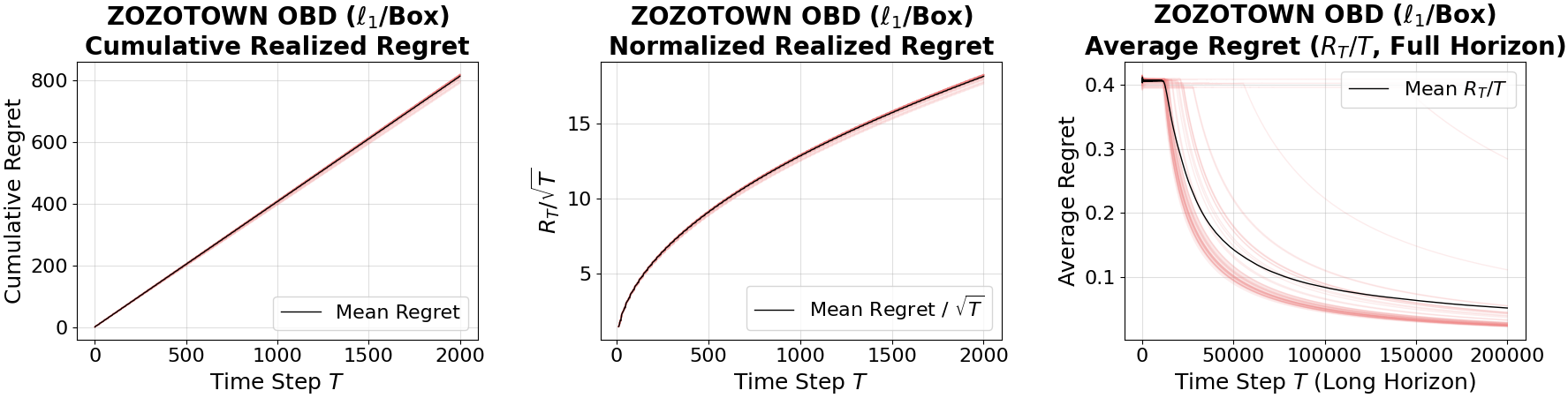}
    \caption{Results of the Personalized Fashion Recommendation experiment on ZOZOTOWN. }
    \label{fig:video_rec}
\end{figure}

{\bf Results.}
Figure~\ref{fig:video_rec} presents the performance of the recommendation system driven by OBD, where the faint red lines represent individual experiment trajectories and the solid black line denotes the mean regret. 
The \textit{Left} plot illustrates the cumulative regret over $T=2,000$ steps. We observe a clear and stable sub-linear growth trend, indicating that the algorithm effectively learns the user's latent preferences from the stochastic feedback and adapts the promotional attribute mix to maximize engagement. By employing the restricted Feel-Good posterior sampling of AG-FGTS, the system falls under the regime of Theorem~\ref{thm:ag_ts_regret}, unlocking the optimal $\tilde{\mathcal{O}}(\sqrt{T})$ rate.
To verify this, we normalize the regret by $\sqrt{T}$ in the \textit{Middle} plot. The resulting curve stabilizes and presents a slowly increasing logarithmic-like trend, empirically corroborating the $\tilde{\mathcal{O}}( \sqrt{T})$ regret bound.
Finally, the \textit{Right} plot displays the average regret $R_T/T$ over a longer horizon of $T=2\times 10^6$. The curve smoothly descends towards zero, confirming that the system's average performance consistently improves over time, highlighting the practical viability of the AG-FGTS algorithm in real-world applications with rigid UI sparsity constraints.

\section{Technical Lemma}\label{app:tech lemmas}

\begin{lemma}[Concentration for Sum of Conditionally Sub-Gaussian Variables]
\label{lemma:sum-abs-subgaussian}
Let $\{\eta_t\}_{t=1}^n$ be a sequence of random variables adapted to a filtration $\{\mathcal{F}_t\}_{t=0}^n$ (i.e., $\eta_t$ is $\mathcal{F}_t$-measurable). Assume that $\{\eta_t\}$ is a conditionally $\sigma$-sub-Gaussian martingale difference sequence, meaning that for all $t \in [n]$ and any $\lambda \in \mathbb{R}$:
$$
\mathbb{E}[\eta_t \mid \mathcal{F}_{t-1}] = 0 \quad \text{and} \quad \mathbb{E}[e^{\lambda \eta_t} \mid \mathcal{F}_{t-1}] \le e^{\frac{\lambda^2 \sigma^2}{2}}.
$$
Let $S_n = \sum_{t=1}^n \eta_t$. Then, for any $\delta \in (0, 1)$, with probability at least $1 - \delta$: $|S_n| \le \sqrt{2n\sigma^2\ln{(2/\delta)}}$.

\end{lemma}

\begin{proof}
This result follows the standard concentration bounds for sub-exponential martingales (see, e.g., \cite[Theorem 2.19]{wainwright2019high}). The proof follows the standard Chernoff bound argument extended to martingale difference sequences via the tower property of conditional expectation.

We begin by applying the Chernoff bound method. For any $\lambda > 0$ and $u \ge 0$, Markov's inequality yields:
\begin{equation}
    P(S_n \ge u) = P(e^{\lambda S_n} \ge e^{\lambda u}) \le e^{-\lambda u} \mathbb{E}[e^{\lambda S_n}]. \label{eqn:Chernoff bound}
\end{equation}

Next, we bound the moment generating function (MGF) $\mathbb{E}[e^{\lambda S_n}]$ by iteratively peeling off terms using the tower property of conditional expectation. Defining the partial sum $S_k = \sum_{t=1}^k \eta_t$, we observe that $S_n = S_{n-1} + \eta_n$, where $S_{n-1}$ is $\mathcal{F}_{n-1}$-measurable. Thus,
\begin{align}
    \mathbb{E}[e^{\lambda S_n}] = \mathbb{E}\left[ \mathbb{E}[e^{\lambda (S_{n-1} + \eta_n)} \mid \mathcal{F}_{n-1}] \right] = \mathbb{E}\left[ e^{\lambda S_{n-1}} \mathbb{E}[e^{\lambda \eta_n} \mid \mathcal{F}_{n-1}] \right].\label{eqn:e111}
\end{align}
By the assumption that $\eta_n$ is conditionally $\sigma$-sub-Gaussian, we have $\mathbb{E}[e^{\lambda \eta_n} \mid \mathcal{F}_{n-1}] \le e^{\frac{\lambda^2 \sigma^2}{2}}$. Substituting this back into \eqref{eqn:e111} gives
\begin{align*}
    \mathbb{E}[e^{\lambda S_n}]\le \mathbb{E}\left[ e^{\lambda S_{n-1}} e^{\frac{\lambda^2 \sigma^2}{2}} \right]= e^{\frac{\lambda^2 \sigma^2}{2}} \mathbb{E}[e^{\lambda S_{n-1}}].
\end{align*}
Repeating this process recursively for $t = n-1, n-2, \dots, 1$, we find that at each step $t$:
$$
\mathbb{E}[e^{\lambda S_t}] \le e^{\frac{\lambda^2 \sigma^2}{2}} \mathbb{E}[e^{\lambda S_{t-1}}].
$$
Iterating this $n$ times (with $S_0 = 0$) results in the final MGF bound:
\begin{equation}
    \mathbb{E}[e^{\lambda S_n}] \le \left( e^{\frac{\lambda^2 \sigma^2}{2}} \right)^n \mathbb{E}[e^{\lambda S_0}] = e^{\frac{n\lambda^2 \sigma^2}{2}}. \label{eqn:MGF}
\end{equation}

Substituting the MGF bound \eqref{eqn:MGF} into the initial Chernoff bound \eqref{eqn:Chernoff bound}, we obtain
$$
P(S_n \ge u) \le \exp\left( -\lambda u + \frac{n\lambda^2 \sigma^2}{2} \right).
$$
We minimize the exponent with respect to $\lambda$. Choosing $\lambda = \frac{u}{n\sigma^2}$ yields the one-sided tail bound:
\begin{equation}
    P(S_n \ge u) \le \exp\left( -\frac{u^2}{2n\sigma^2} \right). \label{eqn:one-sided tail}
\end{equation}

Finally, we extend \eqref{eqn:one-sided tail} to a two-sided bound for the absolute value $|S_n|$. Consider the sequence $\{-\eta_t\}$. Since $\eta_t$ is conditionally $\sigma$-sub-Gaussian, $-\eta_t$ is also conditionally $\sigma$-sub-Gaussian because $\mathbb{E}[e^{\lambda(-\eta_t)} \mid \mathcal{F}_{t-1}] = \mathbb{E}[e^{(-\lambda)\eta_t} \mid \mathcal{F}_{t-1}] \le e^{\frac{(-\lambda)^2 \sigma^2}{2}} = e^{\frac{\lambda^2 \sigma^2}{2}}$. Applying the same logic to $-S_n$, we get $P(S_n \le -u) \le \exp\left( -\frac{u^2}{2n\sigma^2} \right)$. By the union bound, we conclude:
$$
P(|S_n| \ge u) \le P(S_n \ge u) + P(S_n \le -u) \le 2\exp\left( -\frac{u^2}{2n\sigma^2} \right).
$$
Setting $\delta = 2\exp\big( -u^2/(2n\sigma^2) \big)$ and solving for $u$ gives $u = \sqrt{2n\sigma^2\ln(2/\delta)}$, which completes the proof.
\end{proof}

\end{document}